\documentclass{article} % For LaTeX2e
\usepackage{iclr2027_conference,times}
\usepackage{graphicx}
\usepackage{amsthm} 
\usepackage{amsmath,amssymb,thmtools}

\newtheorem{corollary}{Corollary}

\newtheorem{proposition}{Proposition}

\usepackage{amsmath,amsfonts,bm}

\def\eqref#1{equation~\ref{#1}}
\def\1{\bm{1}}

\DeclareMathAlphabet{\mathsfit}{\encodingdefault}{\sfdefault}{m}{sl}
\SetMathAlphabet{\mathsfit}{bold}{\encodingdefault}{\sfdefault}{bx}{n}

\usepackage{hyperref}
\usepackage{url}

\title{Isotropic Gaussian Processes Improve Vanilla Bayesian Optimization in High Dimensions}

\author{
Wei-Ting Tang, Madhav Muthyala \& Joel A. Paulson \\
Department of Chemical and Biological Engineering\\
University of Wisconsin--Madison\\
Madison, WI 53706, USA\\
\texttt{\{wtang98,mmuthyala,joel.paulson\}@wisc.edu}
}

\let\bs\boldsymbol
\iclrfinalcopy % Uncomment for camera-ready version, but NOT for submission.
\begin{document}

\maketitle
\lhead{Preprint}

\begin{abstract}
High-dimensional Bayesian optimization (BO) often fits Gaussian process (GP) surrogates from far fewer observations than input dimensions. Modern Vanilla BO can perform well in this regime with dimension-aware priors, initialization, and acquisition optimization, but it typically retains automatic relevance determination (ARD), fitting one lengthscale per input coordinate. We study this modeling choice and propose Iso-BO, a controlled modification that replaces the ARD GP with an isotropic GP using one shared lengthscale while keeping the surrounding BO pipeline matched. For radial kernels, we show that the marginal log likelihood (MLL) depends on the inverse-squared ARD lengthscales only through weighted pairwise distances among the observed inputs. The current design can therefore leave some ARD directions exactly invisible or only weakly constrained by the MLL. Iso-BO removes coordinatewise reweighting and fits a single shared scale instead. Lengthscale-fitting and predictive-density diagnostics show that this finite-data effect appears in practice, including when the data-generating process is anisotropic. Across GP-prior, synthetic, and real-world benchmarks, Iso-BO often improves over matched modern Vanilla BO and remains competitive with the included high-dimensional BO baselines under the tested budgets. Stress tests also show the expected boundary wherein sufficiently strong, learnable anisotropy can favor the more flexible ARD model.
\end{abstract}

\section{Introduction}
\label{sec:introduction}

High-dimensional Bayesian optimization (HDBO) is often framed as a search problem: how to select useful evaluations in a large (high-dimensional) input space. An equally important challenge is learning a useful surrogate from the limited observations available. In BO based on Gaussian process (GP) surrogates, this means estimating a covariance structure that is informative enough to guide the search without asking the data to support more flexibility than they contain. This tension becomes especially acute when the number of observations is small relative to the input dimension.

Many HDBO methods address this challenge by imposing structure through embeddings, sparsity, additive decompositions, trust regions, or other geometric restrictions \citep{binois2022survey,wang2023recent}. At the same time, recent work has shown that exact GP Vanilla BO can be a surprisingly strong high-dimensional baseline when its lengthscale priors, initialization, and acquisition optimization are adapted to dimension \citep{hvarfner2024vanilla,xu2024standard,papenmeier2025understanding}. These results prompted us to revisit a more basic part of the GP fitting procedure. If the surrounding modern Vanilla BO pipeline is already designed for high dimensions, should the GP still fit a separate lengthscale for every input coordinate?

Automatic relevance determination (ARD) does exactly this. With common radial kernels such as the squared exponential (SE) or Mat\'ern, a separate lengthscale for each coordinate allows the GP to scale distance differently along different input directions. This flexibility is useful when coordinatewise smoothness differs and the available data are sufficient to learn those differences. With limited data, however, the same model must estimate $D$ relative coordinate scales from one finite training design. We show that this difficulty has a specific structure: for radial kernels, the training covariance and marginal log likelihood (MLL) depend on the inverse-squared ARD lengthscales only through weighted pairwise distances among the observed inputs. The current design can therefore leave some ARD directions exactly invisible to the MLL and constrain others only weakly. A MAP prior can regularize such directions, but that preference comes from the prior rather than additional likelihood information in the observed data. Importantly for BO, lengthscale settings that are indistinguishable on the training points can still imply different covariances between a new candidate and the observed data, leading to different posterior predictions and acquisition values.

Motivated by this observation, we propose and study \textbf{Isotropic Bayesian Optimization (Iso-BO)}. Iso-BO makes one controlled change to modern Vanilla BO: instead of fitting $D$ ARD lengthscales, it fits a single shared lengthscale (without changing any other component). In geometric terms, ARD learns how strongly each coordinate contributes to kernel distance, whereas Iso-BO fixes those relative weights and learns only the overall radial scale. The point is not that high-dimensional objectives are inherently isotropic. Rather, Iso-BO tests whether it can be better, at finite BO budgets, to avoid estimating coordinatewise structure that current data may not reliably determine.

We test this idea first at the surrogate level and then inside the BO loop. Controlled fitting experiments examine how the estimated lengthscales behave as data are added, while held-out predictive-density experiments test whether the difference affects posterior quality. Our main predictive diagnostic and GP-prior BO experiments deliberately use anisotropic data-generating processes, so the isotropic GP is misspecified by construction. Even in this setting, Iso-BO often improves over its matched ARD counterpart under the tested budgets, showing that the benefit of a simpler fit can outweigh the loss of coordinatewise flexibility. We then extend the comparison to synthetic and real-world HDBO benchmarks. The tradeoff is also visible in the opposite direction: under the strongest anisotropy tested, the matched ARD model reaches lower final median values, providing a useful boundary for the isotropic approximation. Furthermore, Iso-BO substantially reduces repeated hyperparameter-fitting time in our implementation because only one lengthscale is optimized.

More broadly, the results suggest that surrogate flexibility may need to track how much structure the current data can support learning; Appendix~\ref{app:BO-FIM} explores this idea with BO-FIM, an exploratory extension that uses local Fisher information about the ARD parameters to adapt between isotropic and ARD models as data accumulate.

In summary, our main contributions are:
\begin{enumerate}
    \item We characterize ARD lengthscale fitting through the finite-data geometry of radial kernels. The observed pairwise-distance map can leave ARD directions exactly invisible or weakly informed; a Fisher information bound shows that weakly informed directions can persist even when this map has full column rank. We also show that unresolved ARD directions can change new-point predictions and, in the exact-null case, that isotropic prediction is locally as robust as the best predictor under this ambiguity.

    \item We introduce Iso-BO as a controlled intervention inside modern Vanilla BO: it replaces $D$ coordinatewise ARD lengthscales with one shared lengthscale while otherwise matching the primary BO pipeline.

    \item We evaluate this tradeoff through several lengthscale-fitting and predictive-density diagnostics and BO experiments on GP-prior, synthetic, and real-world benchmarks. Iso-BO often improves over matched Vanilla BO under the tested budgets, including on deliberately anisotropic objectives, while strong-anisotropy stress tests identify a regime where the added ARD flexibility is useful.
\end{enumerate}

\section{Related work}
\label{sec:related-work}

Most HDBO methods reduce the difficulty of the original problem by imposing structure on either the objective or the search. Random-embedding methods optimize in a lower-dimensional representation of the input space \citep{wang2016bayesian,letham2020re,nayebi2019framework,papenmeier2022increasing}, while sparse and variable-selection approaches seek a smaller set of active coordinates or subspaces \citep{eriksson2021high,hoang2025high,hellsten2025leveraging}. Additive methods instead model the objective through lower-dimensional components \citep{kandasamy2015high,gardner2017discovering,rolland2018high,ziomek2023random}. A complementary line of work uses locality, most notably trust-region methods such as TuRBO \citep{eriksson2019scalable} and related local BO approaches \citep{muller2021local,wu2023behavior,nguyen2022local,tang2026rethinking}. These methods can be effective when their structural assumptions are appropriate, but they alter the model structure, search geometry, or region over which the surrogate is used.

The most direct context for our work is the recent re-evaluation of Vanilla BO in high dimensions. \citet{hvarfner2024vanilla} showed that dimension-scaled lengthscale priors make exact-GP Vanilla BO substantially stronger than earlier baselines. \citet{xu2024standard} emphasized dimension-aware lengthscale initialization, while \citet{papenmeier2025understanding} highlighted the importance of acquisition optimization, including random axis-aligned subspace perturbations (RAASP). Together, these results show that much of the apparent failure of Vanilla BO in high dimensions can come from how the standard GP pipeline is configured. We start from this stronger baseline and revisit a remaining modeling choice: whether the GP should fit one lengthscale per input when the BO dataset is small.

There are several related ways to regularize or simplify the surrogate. ARD is standard in GP modeling \citep{williams2006gaussian}, and its lengthscales can be regularized through sparsity-inducing priors \citep{eriksson2021high}, dimension-scaled priors \citep{hvarfner2024vanilla}, or dimension-aware initialization \citep{xu2024standard}. Other methods make larger changes to the surrogate geometry: Linear BO \citep{doumont2025we} uses a linear kernel together with geometric transformations, while BOCK \citep{oh2018bock} and Elastic GP \citep{rana2017high} modify the kernel or input geometry. Iso-BO takes a narrower approach in that it retains the nonlinear radial GP family and asks whether the current finite dataset supports fitting coordinatewise ARD flexibility.

\section{Problem setup and Iso-BO}
\label{sec:setup-method}

\paragraph{BO with GP surrogates.} We consider minimization of an expensive black-box function $f:\mathcal X\rightarrow\mathbb R$ with evaluation budget $N$, using normalized inputs $\mathcal X=[0,1]^D$. After $n$ evaluations, the data are $\mathcal D_n=\{(\bs x_i,y_i)\}_{i=1}^n$, with (noisy) observation model $y_i=f(\bs x_i)+\epsilon_i$ and $\epsilon_i\sim\mathcal N(0,\sigma_\epsilon^2)$. BO alternates between fitting a surrogate to these observations and optimizing an acquisition function that scores candidate points using the surrogate posterior. Its decisions therefore depend directly on the posterior mean and uncertainty supplied by the GP.

Let $f\sim\mathcal{GP}(m,k_\theta)$ denote a GP model, with training inputs $X=(\bs x_1,\ldots,\bs x_n)^\top$, observations $y=(y_1,\ldots,y_n)^\top$, and $z=y-m_X$, where $m_X=(m(\bs x_1),\ldots,m(\bs x_n))^\top$. Writing $C_\theta=K_\theta(X,X)+\sigma_\epsilon^2I_n$, standard GP conditioning gives \citep{williams2006gaussian}
\begin{align*}
    \mu_n(\bs x)
    =m(\bs x)+k_\theta(\bs x,X)C_\theta^{-1}z, \qquad s_n^2(\bs x)
    =k_\theta(\bs x,\bs x)
      -k_\theta(\bs x,X)C_\theta^{-1}k_\theta(X,\bs x).
\end{align*}
A poor estimate of the covariance geometry can therefore affect both which locations appear promising and how uncertain the GP is away from the observed data.

The covariance hyperparameters must be learned from the same finite dataset. For a fixed mean function, the MLL is
\[
    \mathcal L(\theta)
    =
    -\frac{1}{2}z^\top C_\theta^{-1}z
    -\frac{1}{2}\log|C_\theta|
    -\frac{n}{2}\log(2\pi).
\]
Maximum likelihood estimation (MLE) maximizes $\mathcal L(\theta)$, while maximum a posteriori (MAP) estimation maximizes $\mathcal J(\theta)=\mathcal L(\theta)+\log\pi(\theta)$ for prior $\pi(\theta)$. The MLL is central to our analysis because it is the part of the fitting objective through which the observed data distinguish different lengthscale settings; the prior can regularize the fit, but its contribution is separate from the likelihood information in the observed data. In all GP fits, we standardize the observed responses using the mean and standard deviation of the current training set. On this scale, Vanilla-MAP and Iso-MAP fix the kernel signal variance at one; we retain $\sigma_f^2$ in the kernel definitions below for generality.

\paragraph{ARD and isotropic radial kernels.} We focus on radial kernels, which measure similarity through a scaled distance between two inputs. The SE and Mat\'ern~$5/2$ kernels used in this work are standard examples. With ARD, this distance has a separate scale in each coordinate:
\begin{align}
    k_{\mathrm{ARD}}(\bs x,\bs x')
    &=
    \sigma_f^2
    \psi\!\left(\rho_\eta(\bs x,\bs x')\right),
    \qquad
    \rho_\eta(\bs x,\bs x')
    =
    \textstyle\sum_{k=1}^D \eta_k(x_k-x_k')^2,
    \qquad
    \eta_k=\ell_k^{-2}.
    \label{eq:ard-radial-kernel}
\end{align}
Here, $\ell_k$ controls how quickly correlation decays along dimension/coordinate $k$. A short lengthscale makes even a modest change in that coordinate appear far apart to the GP, whereas a long lengthscale makes the GP relatively insensitive to movement in that direction. ARD therefore learns an axis-aligned metric on the input space: the contours of $\rho_\eta$ are ellipsoids whose shape is determined by the $D$ fitted lengthscales. This flexibility is useful when different coordinates genuinely vary on different scales, but it also requires the finite BO dataset to determine that $D$-dimensional geometry.

The isotropic model asks the data to learn less. It uses one shared lengthscale $\ell$, or $\lambda=\ell^{-2}$,
\begin{align}
    k_{\mathrm{iso}}(\bs x,\bs x')
    &=
    \sigma_f^2
    \psi\!\left(
        \lambda\|\bs x-\bs x'\|_2^2
    \right),
    \qquad
    \lambda\in\mathbb R_{++}.
    \label{eq:iso-radial-kernel}
\end{align}
Equivalently, Eq.~\ref{eq:iso-radial-kernel} restricts Eq.~\ref{eq:ard-radial-kernel} to $\eta=\lambda\mathbf 1_D$. The kernel family itself has not changed. What changes is the geometry that must be inferred from the data; ARD learns both the overall scale and the relative scaling of the coordinate axes, whereas the isotropic model fixes the latter and learns only one radial scale. This is the modeling choice we isolate in Iso-BO.

\paragraph{Iso-BO.} Iso-BO makes this single change inside an otherwise matched exact-GP BO pipeline. Our primary comparison is between \textbf{Vanilla-MAP}, which fits the $D$ ARD lengthscales, and \textbf{Iso-MAP}, which fits one shared lengthscale. Both methods otherwise use the same normalized inputs, response standardization, radial kernel family, MAP fitting procedure, initialization, acquisition function, and acquisition optimization procedure.

We retain the dimension-aware lengthscale-prior construction that has been important to the success of modern Vanilla BO \citep{hvarfner2024vanilla}. Vanilla-MAP applies the ambient-$D$-scaled univariate prior to each of its $D$ lengthscales, while Iso-MAP applies the same construction once to its shared lengthscale. This keeps the established prior scaling rather than introducing a separate regularizer for Iso-BO. It also means that the two MAP objectives are not exactly the same objective evaluated in different dimensions: at the likelihood level the isotropic model is the restriction $\eta=\lambda\mathbf 1_D$, but Vanilla-MAP contains $D$ lengthscale-prior contributions whereas Iso-MAP contains one. Section~\ref{sec:theory-package-b} and Appendix~\ref{app:isotropic-restriction-proof} make this distinction precise.

We use the name \emph{Iso-BO} for the shared-lengthscale modification and \emph{Iso-MAP} for its primary MAP implementation. We also report an MLE-based version, Iso-MLE, as a secondary variant. Full implementation details are given in Appendix~\ref{app:experimental-details}.

\section{Finite-data information about ARD lengthscales}
\label{sec:theory-package-b}

The issue is not merely that ARD fits $D$ lengthscales, but which coordinatewise changes the MLL can distinguish from the current training design. We analyze one GP fit with $X$ fixed and use the inverse-squared lengthscales $\eta$ from Eq.~\ref{eq:ard-radial-kernel}. To isolate the lengthscale geometry, we hold the mean, signal variance, noise variance, kernel family, and other non-lengthscale quantities fixed. This does not require those quantities to be fixed in the experiments: for any values they take, the exact lengthscale invariance below remains.

\paragraph{What the MLL sees.}
For each unordered pair of training inputs $i<j$, define $a_{ij}=(\bs x_i-\bs x_j)^{\odot 2}\in\mathbb R_+^D$. Let $P_n=n(n-1)/2$ and let $A_X\in\mathbb R^{P_n\times D}$ have rows $a_{ij}^\top$. Then $\rho_\eta(\bs x_i,\bs x_j)=a_{ij}^\top\eta$, so $A_X\eta$ collects the weighted squared distances among all observed input pairs. Because the diagonal kernel entries do not depend on the lengthscales, the entire training covariance, and hence the MLL, depends on $\eta$ only through $A_X\eta$.

\noindent\textbf{Proposition~\ref{prop:pairwise-design-main} (informal).}
\textit{For fixed $X$ and non-lengthscale quantities, the ARD MLL can be written as $\mathcal L_{\mathrm{ARD}}(\eta)=G(A_X\eta)$ for some function $G$. Any feasible direction $h$ with $A_Xh=0$ therefore leaves the training covariance and MLL unchanged. A nonzero exact MLL-null direction is guaranteed whenever $D>\operatorname{rank}(A_X)$, and in particular whenever $D>n(n-1)/2$. More generally, under the local derivative conditions, normalized directions with small $\|A_Xh\|_2$ have small local first- and second-order MLL sensitivity.}

The full statement and proof are given in Appendix~\ref{app:pairwise-design-proof}. With the 10-point initialization used in our BO experiments, $P_n=45$, so an exact MLL-null direction is guaranteed for every problem with $D>45$. Once $P_n\geq D$, this counting argument no longer guarantees such a direction. Appendix~\ref{app:weak-information} makes this weak-information regime quantitative: under mild local conditions, the Fisher information in direction $h$ is bounded by a constant times $\|A_Xh\|_2^2$. Thus, full rank does not imply that all ARD directions are well informed; directions that barely change the observed pairwise distances can still contain little local likelihood information.

\paragraph{Why MLL-null directions matter for BO.}
The main experiments use MAP fitting, so the prior must also be considered. If $h\in\ker(A_X)$, the MLL is constant along any feasible perturbation $\eta+th$; any change in the MAP objective along that direction therefore comes from the prior. A prior can regularize such a direction and select a stable fit, but it does not add likelihood information from the observed data (see Appendix~\ref{app:isotropic-restriction-proof}).

More importantly, two ARD settings that are indistinguishable on the training data need not make the same predictions. We call two settings \emph{training-equivalent} when they give the same regularized training covariance. An exact MLL-null perturbation preserves this covariance, but generally changes the covariance between a new candidate and the training points. For the SE and Mat\'ern~$5/2$ kernels considered here, Appendix~\ref{app:posterior-sensitivity} shows that, for any fixed nonzero exact-null perturbation, the train-test covariance changes for almost every candidate in a full-dimensional region. The posterior mean or variance can therefore change even though the training MLL is identical. Since the acquisition function is evaluated from these posterior quantities, such changes can alter the preferred candidate and, through sequential sampling, the subsequent BO trajectory. Whether this effect is practically important is an empirical question addressed in Sections~\ref{sec:gp-fitting-diagnostics} and~\ref{sec:experiments}.

\paragraph{What Iso-BO changes.}
Iso-BO restricts the lengthscale parameterization to $\eta=\lambda\mathbf 1_D$. At the MLL level,
$\mathcal L_{\mathrm{iso}}(\lambda)=\mathcal L_{\mathrm{ARD}}(\lambda\mathbf 1_D)=G(\lambda A_X\mathbf 1_D)$,
where $(A_X\mathbf 1_D)_{ij}=\|\bs x_i-\bs x_j\|_2^2$. The fit therefore retains one shared radial scale for the ordinary Euclidean pairwise distances while removing coordinatewise reweighting.

This restriction does not create additional information. Instead, it removes coordinatewise reweighting that the current data may not support. Appendix~\ref{app:weak-information} gives a complementary local result for the exact-null case. When two nearby ARD perturbations around a shared scale are indistinguishable from the training data, the isotropic posterior mean is locally as robust as any predictor that must choose without knowing which perturbation is correct. The tradeoff is that strong and learnable anisotropy can make the isotropic restriction too rigid.

\section{Diagnostics of finite-data GP fitting}
\label{sec:gp-fitting-diagnostics}

Section~\ref{sec:theory-package-b} characterizes what the current training design can tell the MLL about the ARD lengthscales. Before studying BO performance, we isolate the GP fit and ask two simpler questions. First, does the finite-data sensitivity appear in the fitted lengthscales? Second, does tying the lengthscales improve prediction when the underlying function is not isotropic? These diagnostics use training sizes beyond the regime where Proposition~\ref{prop:pairwise-design-main} guarantees an exact null direction, so they also test whether the issue remains visible when only weakly constrained ARD directions may be present.

\paragraph{Controlled lengthscale fitting.} We first use a setting where the target scale is known. Objective functions are sampled from an isotropic SE GP prior with shared lengthscale $\ell=1$, and the training inputs are Sobol points. We fit isotropic and ARD SE GPs by MAP using either gradient descent (GD) or L-BFGS-B~\citep{zhu1997algorithm}. For the ARD model, Figure~\ref{fig:Fitted lscale} reports the arithmetic mean of the fitted coordinatewise lengthscales as a one-dimensional summary.

The isotropic fits remain close to the known shared scale, while the ARD fit is substantially more sensitive to the hyperparameter optimizer, particularly in 50D and 100D. L-BFGS-B moves the ARD average well above the ground truth, whereas GD remains closer to the prior-mode region. Appendix~\ref{app:optimizer-sensitivity-ard-map} shows that L-BFGS-B generally reaches a higher MAP objective than the tested GD fits, so the optimizer dependence is not simply a comparison between equivalent fitted solutions. We use this result as a fitting diagnostic rather than as evidence that optimizer choice causes the underlying problem: the structural issue in Section~\ref{sec:theory-package-b} comes from what the training design constrains.

\begin{figure}[t]
\centering
\includegraphics[width=0.85\textwidth]{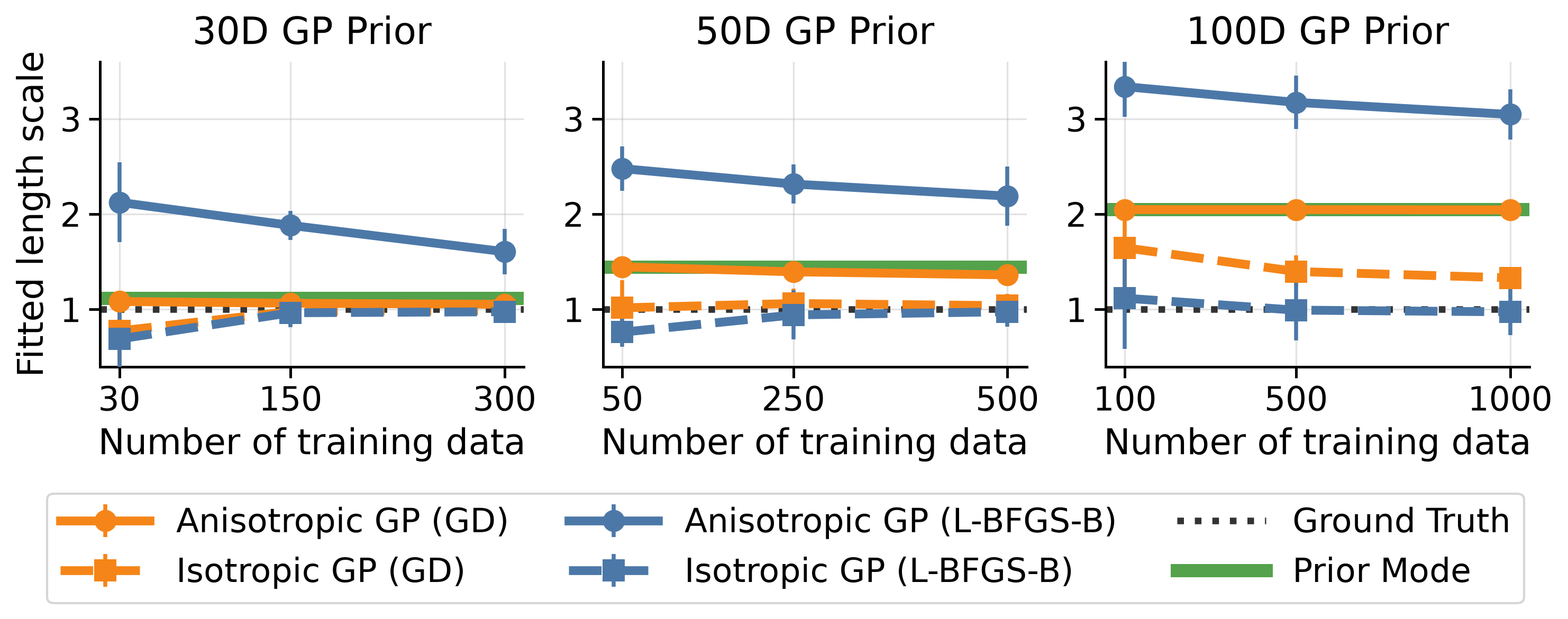}
\caption{Controlled lengthscale fitting on isotropic SE GP-prior realizations with ground-truth lengthscale $\ell=1$. Isotropic and ARD GPs are fit by MAP using L-BFGS-B or GD, with Sobol training inputs. Curves show the mean and $95\%$ confidence interval over 10 independent realizations; the dotted line is the ground truth and the green line is the prior mode. For ARD fits, the plotted value is the arithmetic mean across coordinatewise lengthscales and is used only as a scalar summary of the fitted ARD vector.}
\label{fig:Fitted lscale}
\end{figure}

\paragraph{Held-out predictive density.} The more important question is whether the parameterization changes the posterior predictions used by BO. We therefore remove the favorable assumption of isotropic ground truth. Objective functions are sampled from anisotropic SE GP priors, with coordinatewise lengthscales drawn as described in Appendix~\ref{app:gp-prior-objectives} using $\gamma=0.3$. For each realization, isotropic and ARD SE GPs are fit by MAP with L-BFGS-B to $n$ Sobol training points. We then evaluate average negative log predictive density (NLPD) on the remaining points from a 5000-point evaluation set, using only the training-set statistics for response standardization. Lower NLPD means that the fitted predictive distribution assigns higher predictive density to the held-out responses.

Figure~\ref{fig:NLPD_GPprior} shows lower NLPD for the isotropic fits at every reported dimension and training size. This is the more informative diagnostic because the generating GP has genuinely different coordinatewise lengthscales, so the isotropic model is misspecified by construction. Under these conditions, the benefit of estimating one shared scale outweighs the loss of ARD flexibility for predictive density. The gap generally narrows as more observations are added, which is also consistent with the finite-data interpretation. Appendix~\ref{app:BO trace NLPD} repeats the comparison using chronological splits of completed BO trajectories, and Appendix~\ref{app:synthetic-nlpd-diagnostic} gives complementary results on synthetic objectives.

\begin{figure}[t]
\centering
\includegraphics[width=0.85\textwidth]{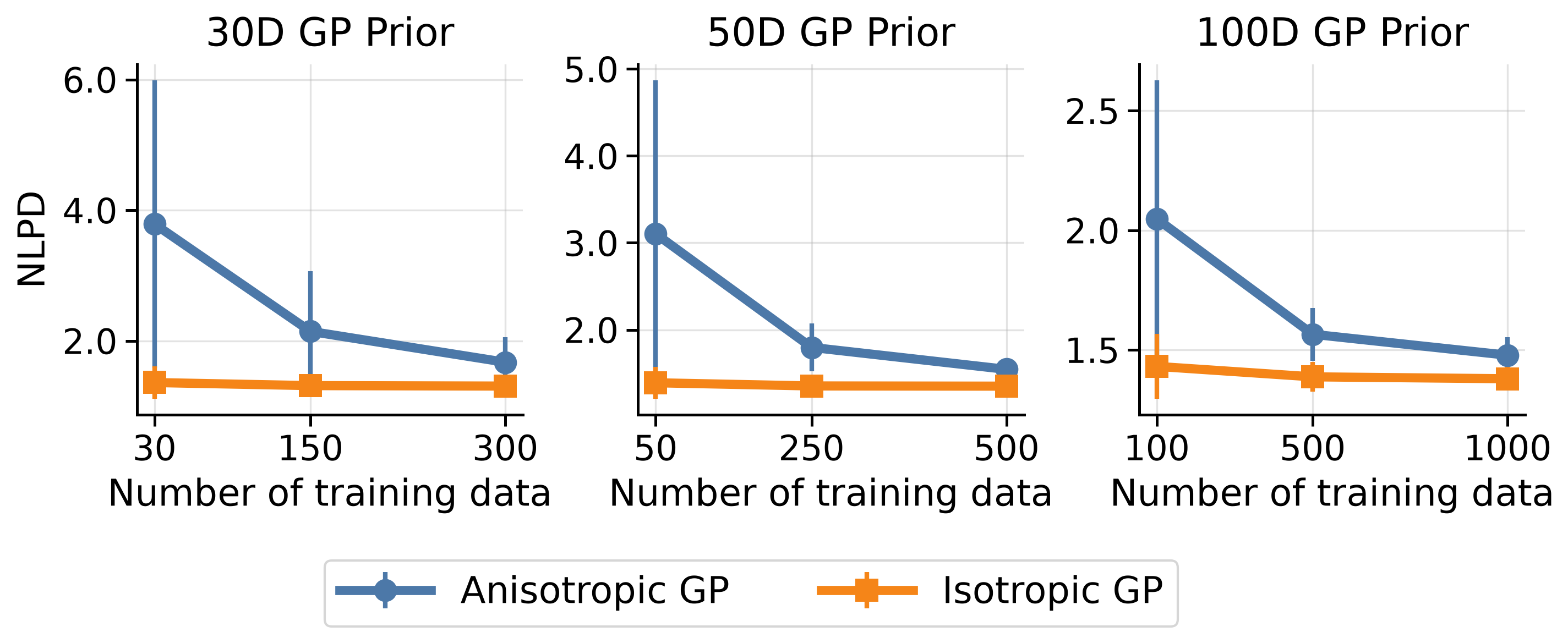}
\caption{Held-out NLPD for isotropic and ARD SE GP fits on functions sampled from anisotropic SE GP priors with $\gamma=0.3$. Both models are fit by MAP with L-BFGS-B using the same Sobol training inputs; held-out responses are standardized using training-set statistics. Curves show the mean and $95\%$ confidence interval over 20 independent GP-prior realizations. Lower NLPD is better, and the isotropic fit has lower NLPD throughout the reported finite-data regimes.}
\label{fig:NLPD_GPprior}
\end{figure}

% These diagnostics show that the lengthscale parameterization changes both the fitted model and its predictive distribution. We next test whether those differences are useful once the surrogate is embedded in the sequential BO loop.

\section{Bayesian optimization experiments}
\label{sec:experiments}

The surrogate diagnostics show that tying the lengthscales can change the fitted GP and its predictive distribution. We now ask whether that difference translates into better sequential optimization. The Vanilla-MAP--Iso-MAP comparison is our main test because the surrounding pipeline is held fixed; the additional HDBO methods provide broader context for how this surrogate change performs relative to methods that modify the search strategy or impose stronger structural assumptions.

\subsection{BO setup}
\label{sec:bo-setup}

The main experiments use an SE kernel for both Vanilla-MAP and Iso-MAP. As defined in Section~\ref{sec:setup-method}, the two methods differ in whether the GP fits $D$ ARD lengthscales or one shared lengthscale. Both begin with 10 Sobol points and use LogEI~\citep{ament2023unexpected}, MAP hyperparameter fitting, and the same RAASP-style acquisition initialization and optimization machinery~\citep{papenmeier2025understanding}. Appendix~\ref{app:Matern} repeats the matched comparison with Mat\'ern~$5/2$ kernels.

For broader HDBO context, we include Linear BO~\citep{doumont2025we}, BAxUS~\citep{papenmeier2022increasing}, TuRBO~\citep{eriksson2019scalable}, SAASBO~\citep{eriksson2021high}, CMA-ES~\citep{hansen2006cma}, and Sobol sampling. We also report the MLE-based Iso-MLE variant alongside MSR~\citep{papenmeier2025understanding}. The current Iso-MLE and MSR main runs use different kernel families, so we treat them as contextual baselines rather than as the controlled test of isotropic versus ARD fitting. Full implementation details are given in Appendix~\ref{app:experimental-details}.

Unless stated otherwise, all methods use a budget of 1000 evaluations; SAASBO is limited to 500 (or less) because of its substantially higher computational cost. The figures report the median best-observed objective value over 10 independent runs, with shaded regions showing the interquartile range (IQR). All objectives use a minimization convention, so lower values are better.

\subsection{Out-of-model anisotropic GP-prior objectives}
\label{sec:out-of-model-experiments}

We begin with the cleanest BO test of the finite-data hypothesis. Objective functions are sampled from anisotropic SE GP priors in 100D, with a separate generating lengthscale in every coordinate. These functions are therefore deliberately outside the isotropic model class. The coordinatewise lengthscales follow the dimension-scaled construction in Appendix~\ref{app:gp-prior-objectives}, with $\gamma\in\{0.3,0.5,0.7\}$ controlling their spread. Each realization is approximated using $M=4096$ random Fourier features.

This setting gives ARD a clear modeling advantage: it contains the covariance family used to generate the objectives, whereas Iso-MAP does not. Nevertheless, Figure~\ref{fig:GPprior} shows that Iso-MAP reaches lower median best-observed values than matched Vanilla-MAP for all three levels of anisotropy. The advantage appears early and persists through the evaluation budget. The paired final-value analysis in Appendix~\ref{app:paired-analysis} gives the strongest evidence at $\gamma=0.3$ and $0.5$, while the median advantage remains in the same direction at $\gamma=0.7$. These results provide direct evidence for the tradeoff motivating Iso-BO: at finite sample sizes, reliably estimating one shared scale can be more useful for optimization than fitting the more expressive but harder-to-determine ARD geometry.

Appendix~\ref{app:stress-tests} pushes the same construction to $\gamma=0.9$ and $0.99$ to identify where this tradeoff reverses. The latter is an intentionally severe anisotropy test: at $\gamma=0.99$, the support of the generating-lengthscale distribution spans a factor of 199 between its endpoints. Vanilla-MAP reaches lower final median values in this regime, showing that sufficiently strong coordinatewise variation can make the additional ARD flexibility worth estimating. Together, the main and stress-test results support the intended interpretation of Iso-BO as a finite-data modeling choice rather than an assumption that the objective is truly isotropic.

\begin{figure}[t]
\centering
\includegraphics[width=\textwidth]{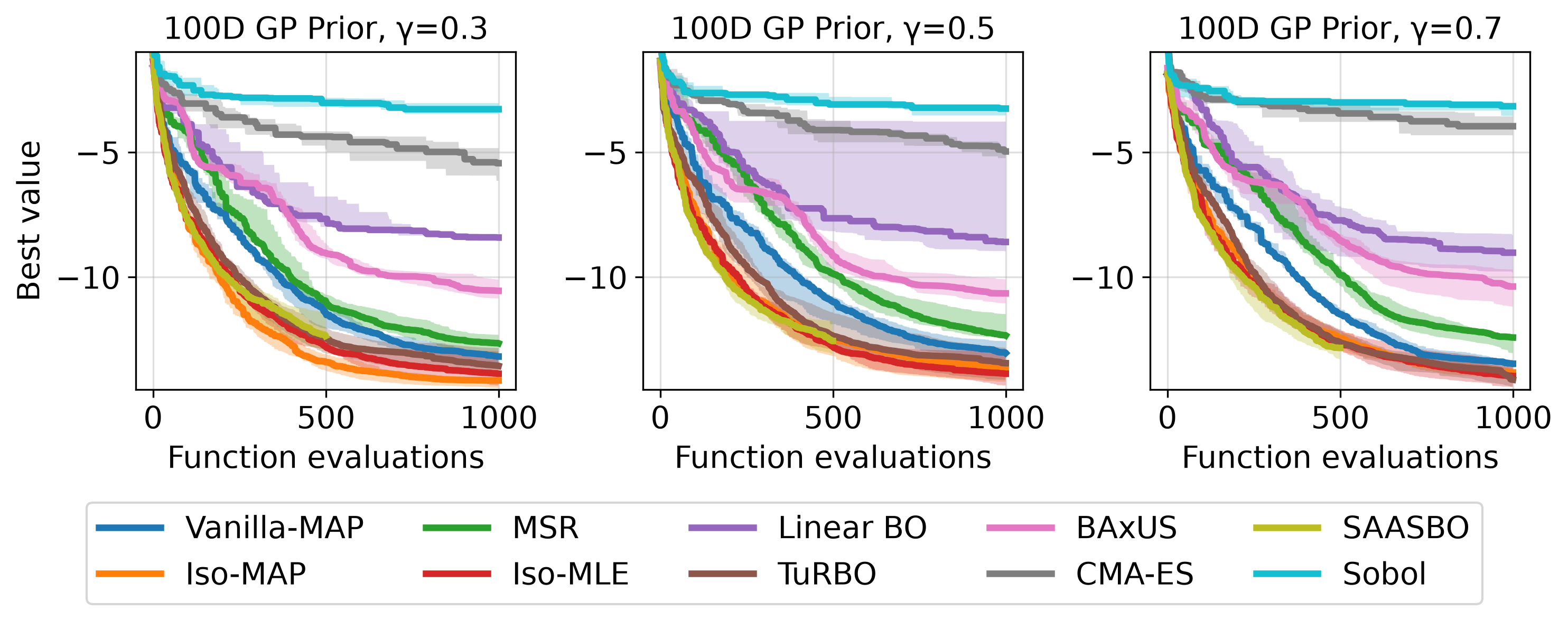}
\caption{Best-observed values on 100D anisotropic SE GP-prior objectives with $\gamma\in\{0.3,0.5,0.7\}$. Curves show the median and IQR over 10 runs; lower is better. The generating GP has coordinatewise lengthscales, so these objectives are deliberately misspecified for Iso-MAP. Despite this, Iso-MAP reaches lower median values than Vanilla-MAP in all settings.}
\label{fig:GPprior}
\end{figure}

\subsection{Synthetic non-GP benchmarks}
\label{sec:synthetic-experiments}

We next test whether the same benefit extends beyond GP-generated objectives. Figure~\ref{fig:synthetic} reports three standard 100D synthetic benchmarks with different objective geometries: Schwefel, Styblinski--Tang, and Powell. Their domains and definitions are given in Appendix~\ref{app:benchmark-definitions}.

Iso-MAP reaches lower final median values than matched Vanilla-MAP on all three functions, with the paired final-value analysis in Appendix~\ref{app:paired-analysis} further supporting this. The separation is substantial on Schwefel and Styblinski--Tang, while Powell provides a useful contrast in the broader HDBO comparison: TuRBO reaches the lowest median values there. These experiments show that the benefit is not confined to functions sampled from a GP prior and can extend to standard nonlinear black-box benchmarks. Schwefel and Styblinski--Tang are coordinate-symmetric, however, so we treat them as complementary non-GP evidence; the anisotropic GP-prior experiments remain the cleaner mechanistic test of the estimation-misspecification tradeoff.\footnote{BAxUS is omitted from the Styblinski--Tang and Powell panels because its standard subspace initialization produced evaluations very close to the global optimum in our implementation, causing the resulting traces to be dominated by initialization rather than subsequent optimization.}

\begin{figure}[t]
\centering
\includegraphics[width=\textwidth]{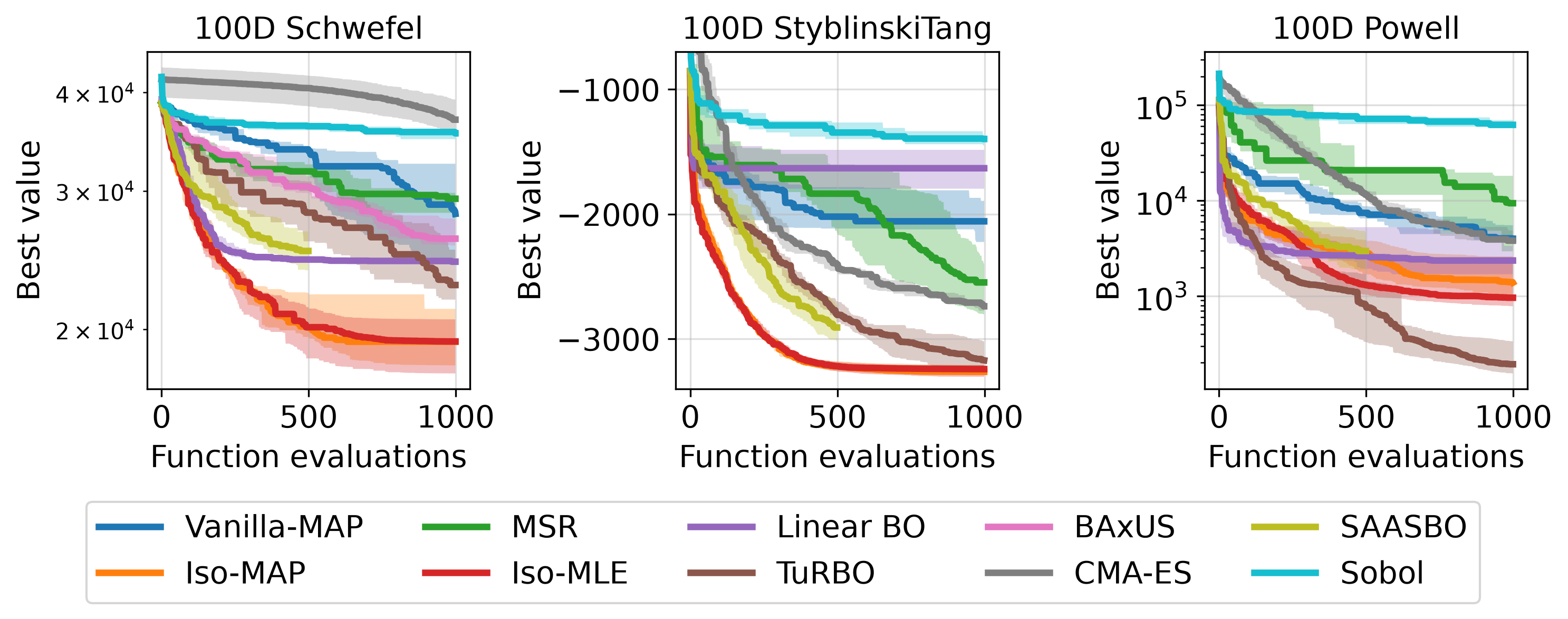}
\caption{Best-observed values on 100D Schwefel, Styblinski--Tang, and Powell. Curves show the median and IQR over 10 runs; lower is better. Iso-MAP reaches lower final median values than matched Vanilla-MAP on all three benchmarks. TuRBO attains the lowest median values on Powell.}
\label{fig:synthetic}
\end{figure}

\subsection{Real-world benchmarks}
\label{sec:trajectory-experiments}

We finally evaluate six high-dimensional benchmarks spanning path planning, controller tuning, action-sequence optimization, and hyperparameter optimization: 100D Rover~\citep{ngo2024high}, 102D HalfCheetah-v2~\citep{wang2020learning}, 300D Walker2d-v5, 388D SVM~\citep{eriksson2021high}, 1003D HumanoidStandup-v5~\citep{xu2024standard}, and 6392D Humanoid-v2~\citep{wang2020learning}. This range of dimensions and problem types provides a substantially less controlled test of whether the same simple surrogate modification remains useful outside synthetic settings.

Figure~\ref{fig:real-world} shows that the shared-lengthscale modification remains competitive across these less controlled settings. Iso-MAP is particularly strong on HalfCheetah-v2 and HumanoidStandup-v5, where the paired final-value analysis in Appendix~\ref{app:paired-analysis} supports lower final values than Vanilla-MAP. Rover also favors Iso-MAP, although the paired evidence is weaker. Iso-MAP reaches a lower final median on Walker2d-v5; on SVM the paired differences lean toward Vanilla-MAP, while Humanoid-v2 is mixed. The broader HDBO comparison varies across problems, as expected, but Iso-MAP remains among the strongest methods on several tasks. These results are notable because the primary comparison changes only the GP lengthscale parameterization, without introducing a trust region, embedding, variable-selection scheme, or new acquisition strategy.

Note that Appendices~\ref{app:extended-bo-evidence} and~\ref{app:batch BO} provide several additional results including paired final-value analyses, stress tests, kernel and computational checks, and extensions to parallel/batch BO and latent-space molecular design.

\begin{figure}[t]
\centering
\includegraphics[width=\textwidth]{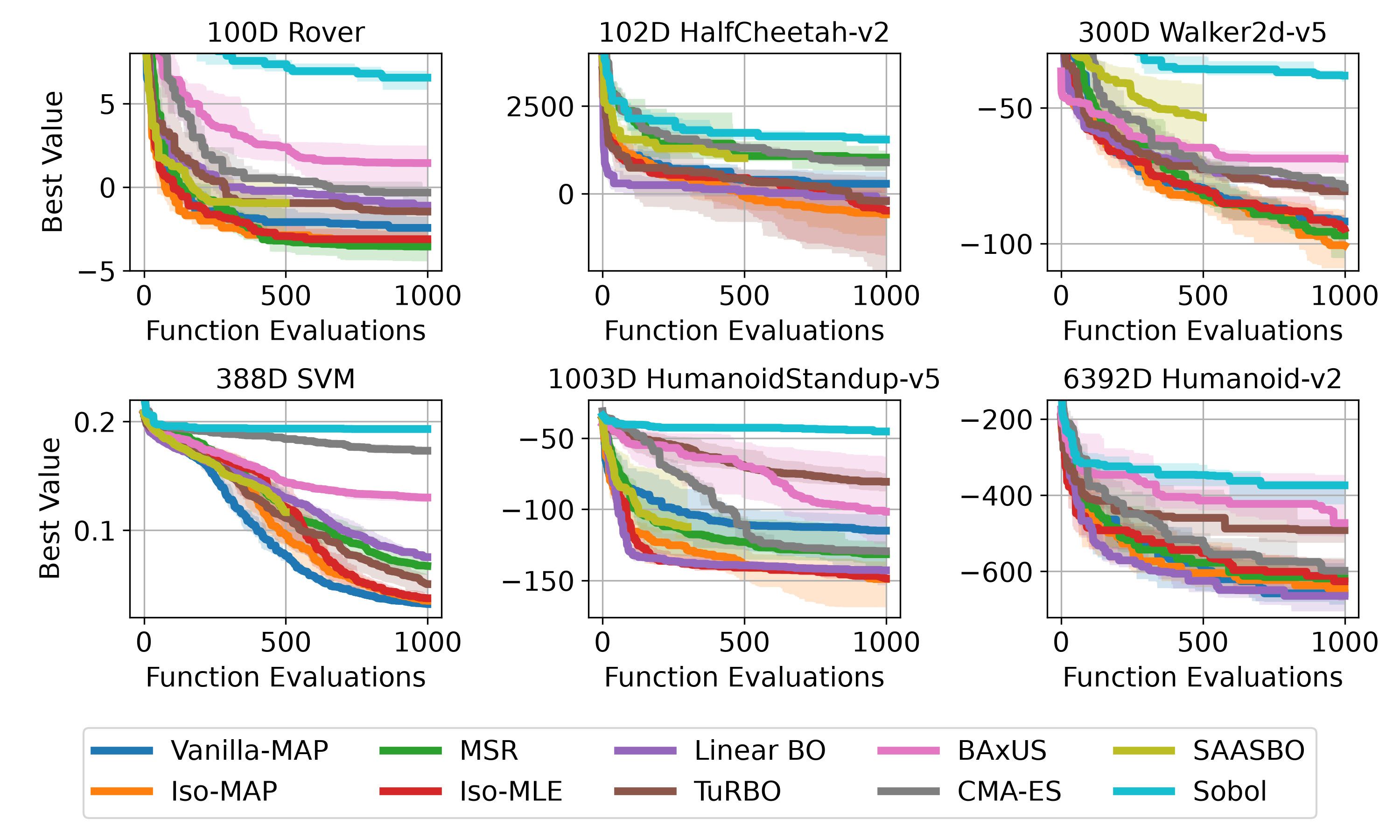}
\caption{Best-observed values on six high-dimensional non-synthetic benchmarks. Curves show the median and IQR over 10 runs; lower is better. Iso-MAP shows its clearest gains over matched Vanilla-MAP on Rover, HalfCheetah-v2, and HumanoidStandup-v5, while remaining competitive across the broader benchmark suite. SAASBO uses 300 evaluations on HumanoidStandup-v5 and is not run on Humanoid-v2. BAxUS is omitted from HalfCheetah-v2 because its standard subspace initialization produce substantially lower initial function value compared with other baseline.}
\label{fig:real-world}
\end{figure}

\section{Conclusion and future work}
\label{sec:conclusion}

In this work, we study a simple question in high-dimensional BO: how much coordinatewise GP structure can a finite dataset support learning? For radial kernels, the MLL depends on the inverse-squared ARD lengthscales only through weighted pairwise distances among the observed inputs, so some ARD directions can be exactly invisible to the MLL or only weakly constrained. Training-equivalent settings can nevertheless differ in their predictions at new candidate points. Iso-BO responds by fitting one shared lengthscale instead of $D$ coordinatewise values. Across surrogate diagnostics and BO experiments, this simpler parameterization is often beneficial under the tested budgets, including for deliberately anisotropic data, while also reducing hyperparameter fitting cost.

The results support an estimation-misspecification tradeoff rather than a universal preference for isotropy. Strong-anisotropy stress tests identify a regime where ARD flexibility becomes valuable, and our analysis does not establish a BO regret improvement. More broadly, the results suggest that surrogate complexity should reflect how much structure the current data can support learning. BO-FIM in Appendix~\ref{app:BO-FIM} provides an exploratory proof of concept for this idea and motivates more principled ways to increase surrogate flexibility as data accumulate.

\subsection*{AI use statement}

We used generative AI tools, primarily ChatGPT (GPT-5.6 Sol), to assist with manuscript editing, identifying potentially relevant literature, and critically checking and refining mathematical statements and proofs. All cited literature was verified against the original sources, and all AI-assisted technical content was independently reviewed by the authors. The authors take responsibility for the final content of this work.

\subsection*{Reproducibility statement}

Appendix~\ref{app:experimental-details} provides the benchmark definitions, GP and BO settings, baseline implementations, and compute details used in the experiments. Appendix~\ref{app:finite-data-fitting-proofs} gives the assumptions and proofs for the main theoretical results. Code for reproducing the reported experiments and figures is included with the supplementary material.

% The authors gratefully acknowledge support from the National Science Foundation (NSF) under CAREER Award No.~2621493.

\bibliographystyle{iclr2027_conference}
\bibliography{references}

\appendix

\section*{Appendix roadmap}

The appendix is organized to follow the logic of the paper. Appendix~\ref{app:finite-data-fitting-proofs} gives the full finite-data theory underlying the ARD lengthscale argument, including the role of MAP regularization, the bridge from ``training-equivalent fits'' to predictions at new candidate points, and the connection between weak ARD information and robust isotropic prediction. Appendix~\ref{app:additional-gp-diagnostics} provides additional surrogate-fitting and predictive-density diagnostics. Appendix~\ref{app:experimental-details} collects the benchmark definitions, GP and BO settings, baseline implementations, and other details needed to reproduce the experiments. Appendix~\ref{app:extended-bo-evidence} provides additional BO evidence and analysis, including final-value statistics, stress tests, robustness checks, query-spread behavior, and computational cost. Appendix~\ref{app:batch BO} extends the comparison to batch BO and latent space optimization for molecular design. Finally, Appendix~\ref{app:BO-FIM} presents BO-FIM as an exploratory information-adaptive extension that uses the Fisher information matrix (FIM) perspective developed in Appendix~\ref{app:finite-data-fitting-proofs} to adapt between isotropic and ARD models.

\section{Finite-data theory for ARD lengthscales}
\label{app:finite-data-fitting-proofs}

We use the notation from Section~\ref{sec:theory-package-b} and consider one fixed training set $\mathcal D_n=(X,y)$, with $X=(x_1,\ldots,x_n)^\top$, $m_X=(m(x_1),\ldots,m(x_n))^\top$, and $z=y-m_X$. We parameterize the ARD lengthscales by $\eta=(\ell_1^{-2},\ldots,\ell_D^{-2})^\top$. Throughout this appendix, the training inputs, mean function, signal variance, noise variance, kernel profile, and other non-lengthscale quantities are held fixed, and we consider feasible parameter values for which the regularized training covariance is positive definite. This isolates what the current training design can tell the likelihood about the ARD lengthscales. Allowing the other hyperparameters to be optimized does not remove the exact invariance below: for any fixed values of those quantities, the same lengthscale null directions leave the training kernel matrix unchanged.

\subsection{Pairwise-distance dependence of the ARD MLL}
\label{app:pairwise-design-proof}

We begin with the formal version of the result summarized in Section~\ref{sec:theory-package-b}. For each unordered training pair $i<j$, define $a_{ij}=(x_i-x_j)^{\odot 2}\in\mathbb R_+^D$. Let $P_n=n(n-1)/2$ and let $A_X\in\mathbb R^{P_n\times D}$ contain the vectors $a_{ij}^\top$ as its rows. Since $\rho_\eta(x_i,x_j)=a_{ij}^\top\eta$, the vector $A_X\eta$ collects all weighted squared distances between the observed input pairs.

\begin{proposition}[Pairwise-distance dependence of the ARD training covariance and MLL]
\label{prop:pairwise-design-main}
\label{prop:pairwise-design}
Fix $X$, $y$, $m_X$, $\sigma_f^2$, $\sigma_\epsilon^2$, and the radial kernel profile $\psi$. For the ARD kernel in Eq.~\ref{eq:ard-radial-kernel}, there exists a function $G$, defined on the relevant pairwise-distance values, such that $\mathcal L_{\mathrm{ARD}}(\eta)=G(A_X\eta)$. If $h\in\ker(A_X)$, $t\in\mathbb R$, and $\eta+th\in\mathbb R_{++}^D$, then $K_{\eta+th}=K_\eta$, $C_{\eta+th}=C_\eta$, and $\mathcal L_{\mathrm{ARD}}(\eta+th)=\mathcal L_{\mathrm{ARD}}(\eta)$. Hence, a nonzero exact MLL-null direction exists whenever $D>\operatorname{rank}(A_X)$, and therefore whenever $D>n(n-1)/2$.

If $G$ is twice differentiable at $A_X\eta$ and $\|\nabla G(A_X\eta)\|_2\le B_1$ and $\|\nabla^2G(A_X\eta)\|_2\le B_2$, then, for any direction $h$,
\[
|D_h\mathcal L_{\mathrm{ARD}}(\eta)|\le B_1\|A_Xh\|_2,
\qquad
|D_h^2\mathcal L_{\mathrm{ARD}}(\eta)|\le B_2\|A_Xh\|_2^2.
\]
Thus, when directions are compared at a fixed scale (which we take as $\|h\|_2=1$ in the inverse-squared lengthscale coordinates), small $\|A_Xh\|_2$ implies small local first- and second-order MLL sensitivity relative to the corresponding derivative scale of $G$.
\end{proposition}

\begin{proof}
Fix an ordering of the $P_n$ unordered training pairs. For a vector $v\in\mathbb R_+^{P_n}$ indexed by these pairs, define $K(v)$ to have diagonal entries $[K(v)]_{ii}=\sigma_f^2\psi(0)$ and off-diagonal entries $[K(v)]_{ij}=[K(v)]_{ji}=\sigma_f^2\psi(v_{ij})$ for $i<j$. Let $C(v)=K(v)+\sigma_\epsilon^2I_n$ and, wherever $C(v)\succ0$ (i.e., is positive definite), define
\[
G(v)=-\frac12z^\top C(v)^{-1}z-\frac12\log\det C(v)-\frac n2\log(2\pi).
\]
For the ARD radial kernel, the distance argument for training pair $(i,j)$ is $\rho_\eta(x_i,x_j)=a_{ij}^\top\eta$. The collection of all off-diagonal distance arguments is therefore $A_X\eta$, while the diagonal kernel entries do not depend on $\eta$. It follows that $K_\eta=K(A_X\eta)$ and $C_\eta=C(A_X\eta)$, which proves $\mathcal L_{\mathrm{ARD}}(\eta)=G(A_X\eta)$.

Now suppose $h\in\ker(A_X)$. Then, $A_X(\eta+th)=A_X\eta$ for every feasible $t$, so the training kernel matrix, regularized covariance, and MLL are all unchanged. By the rank-nullity theorem, $\dim\ker(A_X)=D-\operatorname{rank}(A_X)$, so a nonzero null direction exists whenever $D>\operatorname{rank}(A_X)$. Since $\operatorname{rank}(A_X)\le P_n=n(n-1)/2$, the stated sufficient condition follows. Because $\eta$ lies in the interior of $\mathbb R_{++}^D$, any nonzero null direction also admits sufficiently small nonzero values of $t$ for which $\eta+th$ remains feasible.

Finally, the chain rule applied to $\mathcal L_{\mathrm{ARD}}(\eta)=G(A_X\eta)$ gives $D_h\mathcal L_{\mathrm{ARD}}(\eta)=\langle\nabla G(A_X\eta),A_Xh\rangle$ and $D_h^2\mathcal L_{\mathrm{ARD}}(\eta)=(A_Xh)^\top\nabla^2G(A_X\eta)(A_Xh)$. The Cauchy-Schwarz inequality and the induced matrix $2$-norm then yield the following
\[
|D_h\mathcal L_{\mathrm{ARD}}(\eta)|\le B_1\|A_Xh\|_2,
\qquad
|D_h^2\mathcal L_{\mathrm{ARD}}(\eta)|\le B_2\|A_Xh\|_2^2,
\]
which completes the proof.
\end{proof}

The key point of this result is that the likelihood only sees the ARD lengthscales through the pairwise geometry of the observed inputs. If $A_Xh=0$, moving along $h$ changes none of the weighted distances among the training points, so the training covariance and MLL cannot distinguish those parameter settings at all. The counting condition above gives a simple regime where such directions must exist. Outside that regime, exact null directions are no longer guaranteed, but normalized directions with small $\|A_Xh\|_2$ still change the observed pairwise geometry only weakly and therefore have small local MLL sensitivity under the stated bounds. This is the finite-data limitation that motivates fitting a simpler shared lengthscale when the current design does not strongly constrain coordinatewise reweighting.

\subsection{MAP regularization and the isotropic restriction}
\label{app:isotropic-restriction-proof}

Proposition~\ref{prop:pairwise-design-main} characterizes the information about the ARD lengthscales that comes from the MLL. The models used in our main experiments are fit by MAP, so we also need to account for the lengthscale prior. Write the ARD MAP objective as $\mathcal J_{\mathrm{ARD}}(\eta)=\mathcal L_{\mathrm{ARD}}(\eta)+P_{\mathrm{ARD}}(\eta)$, where $P_{\mathrm{ARD}}$ denotes the prior contribution expressed in the inverse-squared coordinates used above.

\begin{corollary}[Prior selection along an exact MLL-null direction]
\label{cor:prior-selection-null}
If $h\in\ker(A_X)$ and $\eta+th\in\mathbb R_{++}^D$, then
\[
\mathcal J_{\mathrm{ARD}}(\eta+th)-\mathcal J_{\mathrm{ARD}}(\eta)
=
P_{\mathrm{ARD}}(\eta+th)-P_{\mathrm{ARD}}(\eta).
\]
\end{corollary}

\begin{proof}
Proposition~\ref{prop:pairwise-design-main} gives $\mathcal L_{\mathrm{ARD}}(\eta+th)=\mathcal L_{\mathrm{ARD}}(\eta)$ along any feasible exact MLL-null direction, so the likelihood terms cancel from the difference between the two MAP objectives.
\end{proof}

Thus, a MAP prior can regularize ARD directions that the likelihood does not distinguish and can select among otherwise training-equivalent fits. The important point is that this preference comes from the prior, not from additional likelihood information in the observed data. The prior can therefore stabilize the fit without changing the underlying fact that the current training design provides no MLL information along an exact null direction.

The theory is written in terms of $\eta_j=\ell_j^{-2}$ because this makes the pairwise-distance dependence linear, but the implemented prior is placed on the raw lengthscale $\ell_j$. Let $p_D(\ell)$ denote the dimension-scaled univariate prior used in the MAP models. In the current implementation,
$p_D(\ell)=\operatorname{LogNormal}(\sqrt{2}+\frac12\log D,\sqrt{3})$.
Vanilla-MAP applies this same ambient-$D$-scaled prior to each of its $D$ ARD lengthscales, whereas Iso-MAP applies it once to the shared lengthscale. In the inverse-squared coordinates, the corresponding prior contributions are
\[
P_{\mathrm{ARD}}(\eta)
=
\sum_{j=1}^D \log p_D(\eta_j^{-1/2}),
\qquad
P_{\mathrm{iso}}(\lambda)
=
\log p_D(\lambda^{-1/2}).
\]
These expressions evaluate the implemented raw-lengthscale priors at $\ell_j=\eta_j^{-1/2}$ and $\ell=\lambda^{-1/2}$. They are the original prior contributions to the MAP objectives expressed in the inverse-squared coordinates, not transformed probability densities on $\eta$ or $\lambda$.

The isotropic model further restricts the ARD parameterization to $\eta=\lambda\mathbf 1_D$. Since $(A_X\mathbf 1_D)_{ij}=\|x_i-x_j\|_2^2$, this restriction allows the model to fit one common scale for the ordinary squared Euclidean distances between the observed inputs but removes the ability to reweight coordinates separately. At the likelihood level this gives the exact identity
\[
\mathcal L_{\mathrm{iso}}(\lambda)
=
\mathcal L_{\mathrm{ARD}}(\lambda\mathbf 1_D)
=
G(\lambda A_X\mathbf 1_D).
\]
The same identity does not generally hold for the MAP objectives: $\mathcal J_{\mathrm{iso}}(\lambda)\neq\mathcal J_{\mathrm{ARD}}(\lambda\mathbf 1_D)$ because the ARD model contributes $D$ lengthscale-prior terms whereas the isotropic model contributes one. Iso-MAP should therefore be viewed as a more restrictive model with its own MAP objective, not simply as evaluating the full ARD MAP objective along the isotropic line.

Together with Proposition~\ref{prop:pairwise-design-main}, this makes the role of the prior and the shared-lengthscale restriction distinct. The prior regularizes ARD directions that are weakly or not at all informed by the training likelihood, whereas Iso-BO removes coordinatewise reweighting from the fitted model altogether. The next subsection shows why this distinction can matter even when two ARD settings are indistinguishable on the observed training points.

\subsection{Training-equivalent fits can differ at new candidate points}
\label{app:posterior-sensitivity}

Proposition~\ref{prop:pairwise-design-main} shows that different ARD lengthscale settings can produce exactly the same covariance matrix on the observed training inputs. We now show why this can still matter for BO. Two such settings are indistinguishable to the training likelihood, but a new candidate point introduces a different set of distances: those between the candidate and each observed input. Preserving the training geometry does not generally preserve this candidate-to-training geometry.

For a candidate $x$, let $B_X(x)\in\mathbb R^{n\times D}$ be the matrix whose $i$th row is $((x_i-x)^{\odot 2})^\top$. Thus, while $A_X$ describes coordinatewise squared differences among pairs of training points, $B_X(x)$ describes the corresponding differences between the candidate point $x$ and each training point.

\begin{corollary}[Training covariance invariance does not imply train-test covariance invariance]
\label{cor:training-test-invariance}
Assume $\sigma_f^2>0$ and that the radial profile $\psi$ is injective on the relevant nonnegative distance arguments (which holds for the SE and Mat\'ern~$5/2$ radial profiles used in this paper). Let $h\in\ker(A_X)$, let $t\neq0$, and suppose $\eta+th\in\mathbb R_{++}^D$. Then
\[
C_{\eta+th}=C_\eta,
\qquad
k_{\eta+th}(X,x)=k_\eta(X,x)
\quad\Longleftrightarrow\quad
B_X(x)h=0.
\]
For any fixed nonzero $h$, if the candidate domain contains a full-dimensional open set, then $B_X(x)h\neq0$ for almost every candidate $x$ in that set.
\end{corollary}

\begin{proof}
Because $h\in\ker(A_X)$, Proposition~\ref{prop:pairwise-design-main} immediately gives $C_{\eta+th}=C_\eta$. Consider now the covariance between a training point $x_i$ and a candidate $x$. Its kernel distance is $((x_i-x)^{\odot 2})^\top\eta$. After moving from $\eta$ to $\eta+th$, this distance changes by $t((x_i-x)^{\odot 2})^\top h$. Since $t\neq0$, $\sigma_f^2>0$, and $\psi$ is injective, the covariance with $x_i$ is unchanged if and only if $((x_i-x)^{\odot 2})^\top h=0$. Requiring this for every training point is exactly $B_X(x)h=0$.

It remains to show that this is exceptional for a nonzero $h$. For any training point $x_i$, define $q_i(x)=\sum_{j=1}^D h_j(x_{ij}-x_j)^2$. Since $h\neq0$, $q_i$ is a nonzero polynomial in $x$, whose zero set has Lebesgue measure zero. $\{x:B_X(x)h=0\}$ is contained in this zero set and thus also has measure zero.
\end{proof}

This result exposes the part of the ambiguity that matters for BO. An exact null direction can leave every covariance among the observed points unchanged while changing how a new candidate correlates with those observations. To see the posterior consequence directly, let $C=C_\eta=C_{\eta+th}$, $k_x=k_\eta(X,x)$, and $\delta k_x=k_{\eta+th}(X,x)-k_\eta(X,x)$. Because $C$ is unchanged,
\[
\mu_{\eta+th}(x)-\mu_\eta(x)=\delta k_x^\top C^{-1}z,
\qquad
s_{\eta+th}^2(x)-s_\eta^2(x)
=
-2\delta k_x^\top C^{-1}k_x
-\delta k_x^\top C^{-1}\delta k_x.
\]
Note that the variance identity uses the fact that, for the stationary radial kernels considered here, $k_\eta(x,x)=\sigma_f^2\psi(0)$ does not depend on the lengthscales.
Therefore, two ARD settings can be completely indistinguishable to the training likelihood while giving different posterior geometry at candidate points. These differences can change the posterior mean or variance and, in turn, posterior-based acquisition values. This does not mean that every exact-null perturbation changes every posterior quantity, since cancellations can occur for particular candidates or responses. In the next subsection, we consider what remains of this finite-data ambiguity beyond exact null directions and what it implies for isotropic prediction.

\subsection{Weak ARD information and robust isotropic prediction}
\label{app:weak-information}

Proposition~\ref{prop:pairwise-design-main} gives an exact failure mode for ARD lengthscale learning, but its rank condition is strongest early in BO. Two questions remain. Can some ARD directions still be poorly informed after $A_X$ becomes full rank? And, when an ARD direction is completely unresolved, is the isotropic prediction a principled way to avoid fitting that ambiguity? We address these questions below.

To quantify weak rather than exact information, let $I(\eta)\in\mathbb R^{D\times D}$ denote the Fisher information matrix (FIM) for the inverse-squared ARD lengthscales, with the GP mean and other non-lengthscale quantities fixed~\citep{kay1993fundamentals}. For $C_\eta=K_\eta(X,X)+\sigma_\epsilon^2I_n$,
\[
[I(\eta)]_{pq}
=
\frac{1}{2}
\operatorname{tr}\!\left(
C_\eta^{-1}
\frac{\partial C_\eta}{\partial\eta_p}
C_\eta^{-1}
\frac{\partial C_\eta}{\partial\eta_q}
\right).
\]
The next result shows that the same quantity $\|A_Xh\|_2$ used in Proposition~\ref{prop:pairwise-design-main} also controls the local Fisher information along $h$.

\begin{proposition}[Pairwise-distance sensitivity controls Fisher information]
\label{prop:fim-pairwise-bound}
Suppose $\lambda_{\min}(C_\eta)\geq\nu>0$ and $|\psi'(r)|\leq L_\psi$ over the weighted training distances $r=a_{ij}^\top\eta$. Then, for any $h\in\mathbb R^D$,
\[
h^\top I(\eta)h
\leq
\frac{\sigma_f^4L_\psi^2}{\nu^2}
\|A_Xh\|_2^2.
\]
Consequently, $\ker(A_X)\subseteq\ker(I(\eta))$, and $\lambda_{\min}(I(\eta))
\leq
\frac{\sigma_f^4L_\psi^2}{\nu^2}
\sigma_{\min}(A_X)^2$.
\end{proposition}

\begin{proof}
Consider a perturbation of the ARD parameters in direction $h$, and let
$\dot C_\eta[h]=\sum_{d=1}^D h_d\,\partial C_\eta/\partial\eta_d$
denote the resulting directional derivative of the training covariance.
The diagonal entries are zero because $k_\eta(x_i,x_i)$ does not depend on
the lengthscales. For each pair $i<j$, $[\dot C_\eta[h]]_{ij}
=
\sigma_f^2\psi'(a_{ij}^\top\eta)\,a_{ij}^\top h$.
Thus, $a_{ij}^\top h$ is exactly the first-order change in the weighted
distance for pair $(i,j)$. Summing these changes over all training pairs gives
\[
\|\dot C_\eta[h]\|_F^2
\leq
2\sigma_f^4L_\psi^2\|A_Xh\|_2^2,
\]
where the factor of two accounts for the two symmetric off-diagonal entries
associated with each unordered pair.
For a Gaussian covariance model, the Fisher information in direction $h$ is
\[
h^\top I(\eta)h
=
\frac{1}{2}
\left\|
C_\eta^{-1/2}\dot C_\eta[h]C_\eta^{-1/2}
\right\|_F^2.
\]
This is the covariance change after scaling by the covariance of the observed
data. Since $\lambda_{\min}(C_\eta)\geq\nu$,
$\|C_\eta^{-1/2}\|_2\leq\nu^{-1/2}$, and therefore
\[
h^\top I(\eta)h
\leq
\frac{1}{2\nu^2}\|\dot C_\eta[h]\|_F^2
\leq
\frac{\sigma_f^4L_\psi^2}{\nu^2}\|A_Xh\|_2^2.
\]

If $A_Xh=0$, the bound gives $h^\top I(\eta)h=0$. Since $I(\eta)$ is
positive semidefinite, this implies $I(\eta)h=0$, so
$\ker(A_X)\subseteq\ker(I(\eta))$.

Finally, let $v$ be a unit vector minimizing $\|A_Xv\|_2$. Then
$\|A_Xv\|_2=\sigma_{\min}(A_X)$, and the Rayleigh quotient gives
\[
\lambda_{\min}(I(\eta))
\leq
v^\top I(\eta)v
\leq
\frac{\sigma_f^4L_\psi^2}{\nu^2}\sigma_{\min}(A_X)^2,
\]
which completes the proof.
\end{proof}

% \begin{proof}
% For a direction $h$, let $\dot C_\eta[h]=\sum_{d=1}^D h_d\,\partial C_\eta/\partial\eta_d$. Its diagonal entries are zero, while for $i<j$, $[\dot C_\eta[h]]_{ij}
% =
% \sigma_f^2\psi'(a_{ij}^\top\eta)a_{ij}^\top h$.
% Therefore, $\|\dot C_\eta[h]\|_F^2\leq2\sigma_f^4L_\psi^2\|A_Xh\|_2^2$. The directional Fisher information satisfies
% \[
% h^\top I(\eta)h
% =
% \frac{1}{2}
% \left\|
% C_\eta^{-1/2}\dot C_\eta[h]C_\eta^{-1/2}
% \right\|_F^2
% \leq
% \frac{1}{2\nu^2}\|\dot C_\eta[h]\|_F^2,
% \]
% which gives the stated bound. If $A_Xh=0$, then $h^\top I(\eta)h=0$ and positive semidefiniteness of $I(\eta)$ gives $I(\eta)h=0$. The eigenvalue bound follows by choosing a unit right singular vector associated with $\sigma_{\min}(A_X)$.
% \end{proof}

This result extends the pairwise-distance argument beyond exact rank deficiency. Full column rank only rules out an exactly invisible ARD direction. If $\sigma_{\min}(A_X)$ is small, there is still a normalized direction with little local Fisher information. Thus, crossing the rank threshold does not mean that all coordinatewise lengthscale variation suddenly becomes well informed. 

Although this explains how weak ARD information can persist, it does not yet say why tying the lengthscales together is a reasonable response. It turns out that question can be answered cleanly in the exact-null case. Suppose two nearby ARD perturbations on opposite sides of an isotropic model give exactly the same training distribution. The data then provide no basis for choosing between their different predictions at a new candidate. We ask which prediction is most robust to that unresolved choice.

Fix $\eta_0=\lambda\mathbf 1_D$ and choose a nonzero vector $h$ satisfying $A_Xh=0$ and $\mathbf 1_D^\top h=0$. The second condition holds the average inverse-squared scale fixed while changing the relative coordinate weights. A nonzero such direction exists whenever $D>\operatorname{rank}(A_X)+1$. For sufficiently small $t>0$, let $\eta_\pm=\eta_0\pm th$. Proposition~\ref{prop:pairwise-design-main} gives $C_{\eta_+}=C_{\eta_-}=C_{\eta_0}=:C_0$.

For a fixed candidate $x$, write $k_\pm=k_{\eta_\pm}(X,x)$ and $k_0=k_{\eta_0}(X,x)$. With $Z=Y-m_X\sim\mathcal N(0,C_0)$, the corresponding posterior means are $\mu_\pm(Z;x)=m(x)+k_\pm^\top C_0^{-1}Z$ and $\mu_0(Z;x)=m(x)+k_0^\top C_0^{-1}Z$. Define
\[
R_t^\star(x)
=
\inf_g\max_{s\in\{+,-\}}
\mathbb E\!\left[
\bigl(g(Z)-\mu_s(Z;x)\bigr)^2
\right]
\]
over all square-integrable predictors $g(Z)$, and let
\[
R_t^{\mathrm{iso}}(x)
=
\max_{s\in\{+,-\}}
\mathbb E\!\left[
\bigl(\mu_0(Z;x)-\mu_s(Z;x)\bigr)^2
\right].
\]
These risks isolate the prediction error caused by not knowing which covariance geometry is correct.

\begin{proposition}[Local minimax prediction under an unresolved ARD contrast]
\label{prop:local-minimax-iso}
The minimax predictor is the midpoint of the two posterior means,
\[
g_t^\star(Z)
=
\frac{\mu_+(Z;x)+\mu_-(Z;x)}{2},
\]
with risk
\[
R_t^\star(x)
=
\frac{1}{4}
(k_+-k_-)^\top C_0^{-1}(k_+-k_-).
\]
If $k_{\eta_0+sh}(X,x)$ is differentiable at $s=0$ and $D_hk_{\eta_0}(X,x)\neq0$, define
\[
q_x(h)
=
D_hk_{\eta_0}(X,x)^\top
C_0^{-1}
D_hk_{\eta_0}(X,x).
\]
Then
\[
R_t^\star(x)
=
t^2q_x(h)+o(t^2),
\qquad
R_t^{\mathrm{iso}}(x)
=
t^2q_x(h)+o(t^2),
\]
and therefore
\[
\frac{R_t^{\mathrm{iso}}(x)}{R_t^\star(x)}
\longrightarrow1
\qquad\text{as }t\rightarrow0.
\]
\end{proposition}

\begin{proof}
Because $\eta_+$ and $\eta_-$ give the same training covariance, both risks are evaluated under the same distribution $Z\sim\mathcal N(0,C_0)$. Define
\[
R_\pm(g)
=
\mathbb E\!\left[(g(Z)-\mu_\pm(Z;x))^2\right],
\]
so that $R_t(g)=\max\{R_+(g),R_-(g)\}$. For any $g$, $R_t(g)
\geq
\frac{R_+(g)+R_-(g)}{2}$.
Let $\mu_{\mathrm{mid}}=(\mu_++\mu_-)/2$. Expanding the average on the right gives
\[
\frac{R_+(g)+R_-(g)}{2}
=
\mathbb E[(g-\mu_{\mathrm{mid}})^2]
+
\frac{1}{4}\mathbb E[(\mu_+-\mu_-)^2].
\]
The first term is nonnegative, so every predictor satisfies
\[
R_t(g)
\geq
\frac{1}{4}\mathbb E[(\mu_+-\mu_-)^2].
\]
For $g=\mu_{\mathrm{mid}}$, the errors relative to $\mu_+$ and $\mu_-$ have the same magnitude, and both risks equal this lower bound. Hence the midpoint is minimax and
\[
R_t^\star(x)
=
\frac{1}{4}\mathbb E[(\mu_+-\mu_-)^2]
=
\frac{1}{4}
(k_+-k_-)^\top C_0^{-1}(k_+-k_-).
\]

Now let $k(s)=k_{\eta_0+sh}(X,x)$ and $d=D_hk_{\eta_0}(X,x)$. Differentiability gives $k(\pm t)=k_0\pm td+o(t)$, so $k_+-k_-=2td+o(t)$. Substituting these expansions into $R_t^\star(x)$ and $R_t^{\mathrm{iso}}(x)$ gives $t^2d^\top C_0^{-1}d+o(t^2)$ in both cases. Since $d\neq0$ and $C_0\succ0$, the ratio converges to one.
\end{proof}

% \begin{proof}
% Let $\mu_{\mathrm{mid}}=(\mu_++\mu_-)/2$. For any $g$,
% \[
% \frac{1}{2}
% \mathbb E\!\left[(g-\mu_+)^2+(g-\mu_-)^2\right]
% =
% \mathbb E[(g-\mu_{\mathrm{mid}})^2]
% +
% \frac{1}{4}\mathbb E[(\mu_+-\mu_-)^2].
% \]
% The worst-case risk is at least this average. Choosing $g=\mu_{\mathrm{mid}}$ makes the two risks equal and attains the lower bound, giving the stated expression for $R_t^\star(x)$.

% Now let $k(s)=k_{\eta_0+sh}(X,x)$ and $d=D_hk_{\eta_0}(X,x)$. Differentiability gives $k(\pm t)=k_0\pm td+o(t)$, so $k_+-k_-=2td+o(t)$. Substituting these expansions into $R_t^\star(x)$ and $R_t^{\mathrm{iso}}(x)$ gives $t^2d^\top C_0^{-1}d+o(t^2)$ in both cases. Since $d\neq0$ and $C_0\succ0$, the ratio converges to one.
% \end{proof}

The two results sharpen different parts of the finite-data picture. Near-null directions can remain weakly informed even after the simple rank guarantee has disappeared. When the ambiguity is exact, the isotropic prediction is also locally robust in a precise sense: to first order, it incurs the same worst-case excess prediction risk as the best predictor that could be chosen without knowing which unresolved ARD perturbation is correct. This does not make isotropy optimal for general anisotropy; rather, it gives a concrete reason to avoid fitting coordinatewise structure when the current data do not resolve it. The predictive-density and BO experiments test how far this tradeoff remains useful beyond this local setting.

\section{Additional surrogate-fitting and predictive diagnostics}
\label{app:additional-gp-diagnostics}

The main text uses two surrogate-level diagnostics to connect the finite-data theory to the BO results: the controlled lengthscale-fitting experiment in Figure~\ref{fig:Fitted lscale} and the held-out predictive-density experiment in Figure~\ref{fig:NLPD_GPprior}. This appendix provides supporting checks on both. We first examine the sensitivity of the ARD MAP fits to the numerical optimizer, then extend the predictive-density comparison to data collected along BO trajectories and to deterministic non-GP objectives. We finally repeat the lengthscale-fitting diagnostic with different fitted kernels and data sources. These experiments are intended to characterize the fitting behavior observed in practice; optimizer sensitivity is supporting evidence, not the source of the MLL geometry established in Proposition~\ref{prop:pairwise-design-main}. When an ARD fit is represented by a single lengthscale curve, we plot the arithmetic mean across coordinates as a scalar summary of the fitted ARD vector.

\subsection{Optimizer sensitivity in the ARD MAP fits}
\label{app:optimizer-sensitivity-ard-map}

The controlled experiment in Figure~\ref{fig:Fitted lscale} compares L-BFGS-B with plain first-order gradient descent (GD) applied to the same MAP objective. GD uses a fixed learning rate and is run for at most 1000 steps, whereas L-BFGS-B is a limited-memory quasi-Newton method with bound constraints~\citep{zhu1997algorithm}. Figure~\ref{fig:Fitted lscale} shows that the resulting ARD lengthscale fits can differ substantially: L-BFGS-B often moves farther from the prior-mode region than GD. We next check whether this larger movement also corresponds to a higher attained MAP objective.

Figure~\ref{fig:MLL Comparison} reports the unnormalized log posterior per training point, i.e., the MLL plus the log-prior contribution, for L-BFGS-B and GD with learning rates $0.01$, $0.1$, and $1$. Across the GP-prior, Schwefel, and Styblinski-Tang settings, L-BFGS-B generally attains the highest training MAP objective. Thus, the different lengthscale fits are not simply different parameter values with indistinguishable attained objectives; in these diagnostics, the numerical fitting procedure can lead to meaningfully different ARD MAP solutions.

\begin{figure}[t]
\centering
\includegraphics[width=0.9\textwidth]{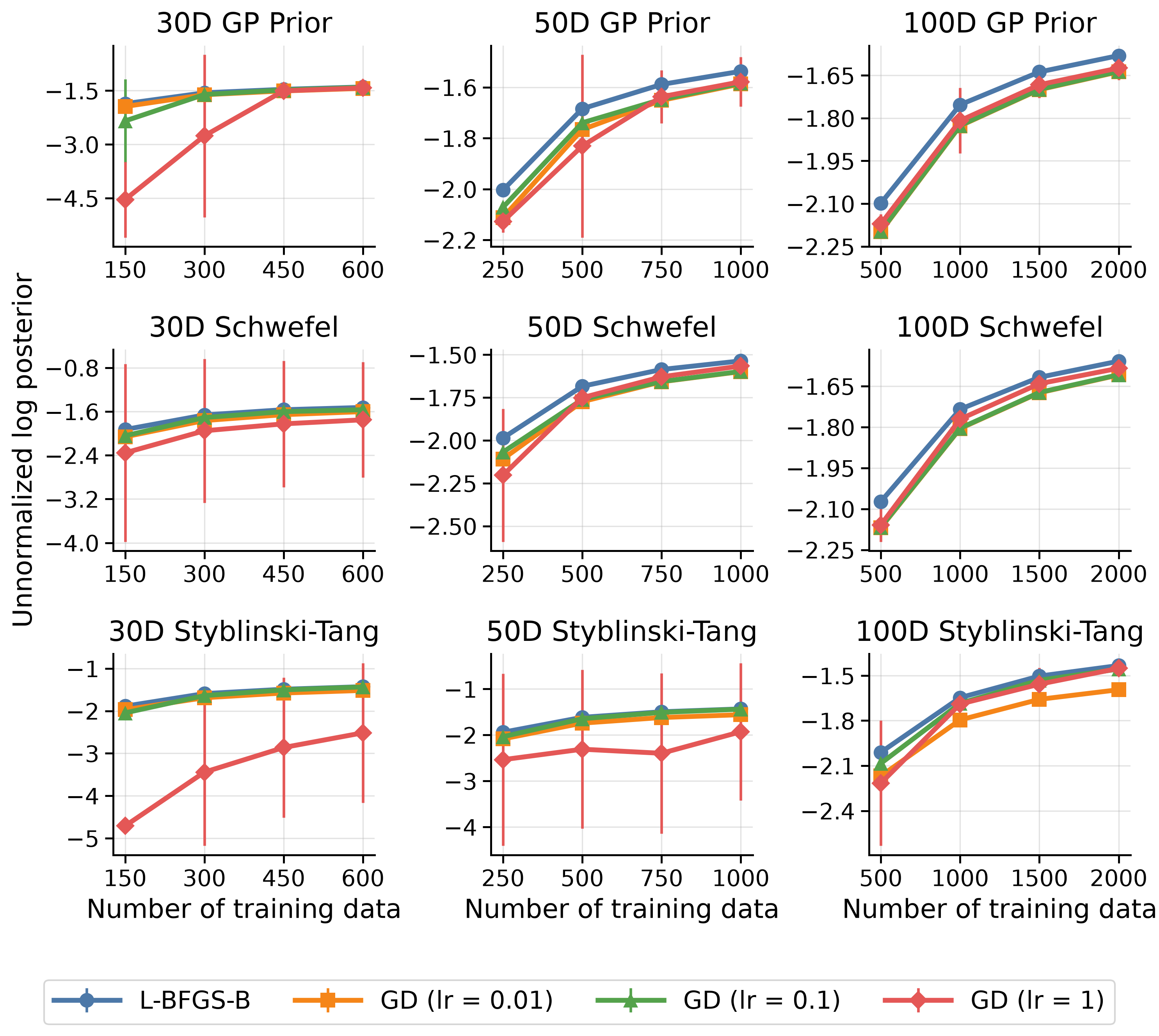}
\caption{Training MAP objective values for SE-kernel ARD GP fits using L-BFGS-B and GD. GD uses learning rates $0.01$, $0.1$, and $1$. The plotted quantity is the average unnormalized log posterior per training point, i.e., MLL plus the log-prior contribution; higher is better. \textbf{Top row:} isotropic GP-prior data with ground-truth lengthscale $\ell=1$; \textbf{middle row:} Schwefel; \textbf{bottom row:} Styblinski--Tang. Training inputs are Sobol points. Curves show the mean and $95\%$ confidence interval over 10 independent replicates. L-BFGS-B generally attains higher training MAP objective values than the tested GD fits. GD is run for at most 1000 steps.}
\label{fig:MLL Comparison}
\end{figure}

Figure~\ref{fig:Ablation learning rate} shows the corresponding learning-rate ablation for the fitted ARD lengthscales. Changing the GD learning rate affects the fitted values, most noticeably for the largest learning rate, but the broad behavior remains similar: the GD fits stay relatively close to the prior-mode region compared with the larger movements produced by L-BFGS-B in Figure~\ref{fig:Fitted lscale}. This does not imply that GD could not reach a different solution under another optimization schedule or computational budget. It shows that the optimizer dependence seen in the main diagnostic is not explained by one particular choice of GD learning rate.

\begin{figure}[t]
\centering
\includegraphics[width=0.9\textwidth]{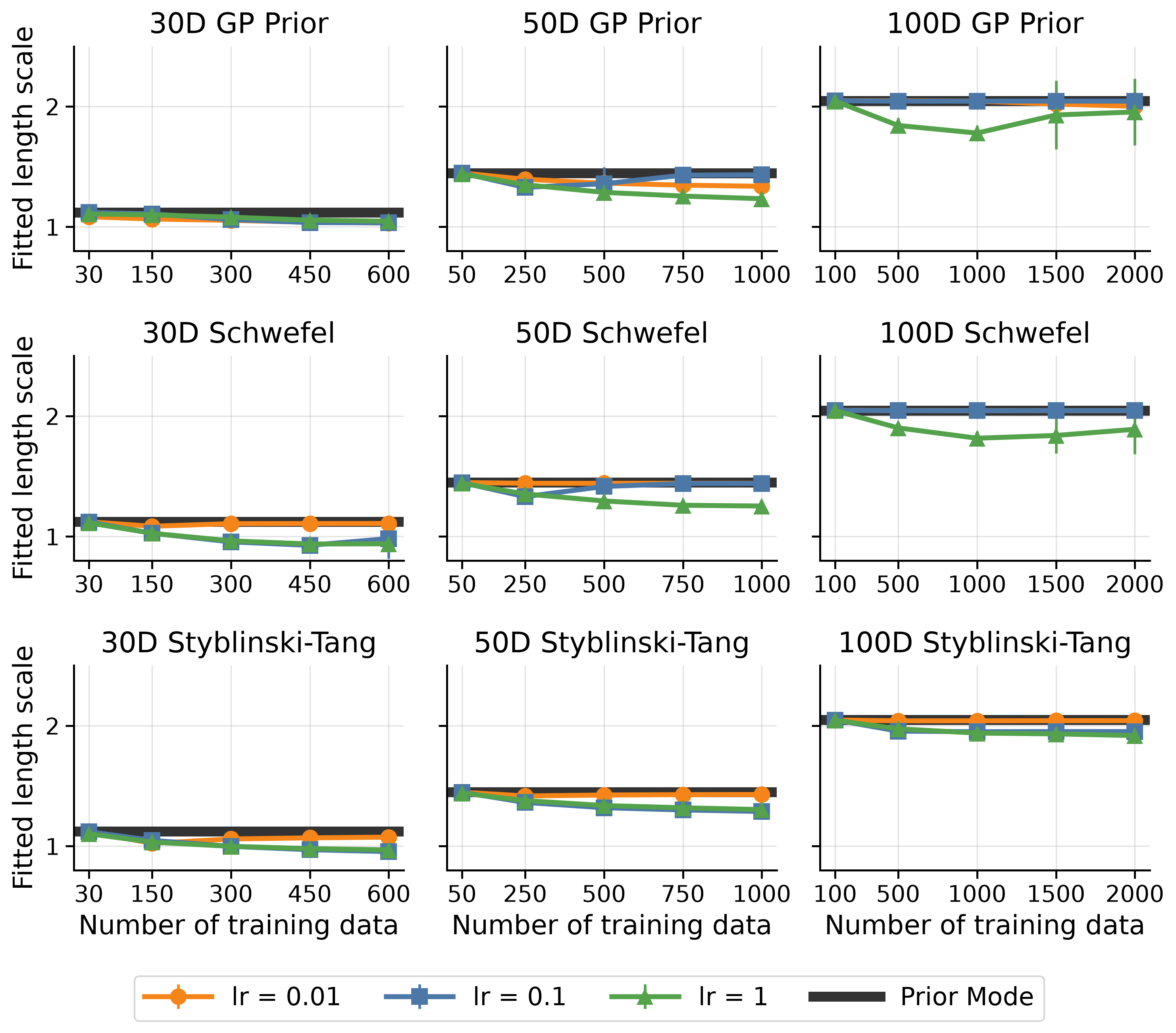}
\caption{Learning-rate ablation for SE-kernel ARD GP fits using GD. The figure reports the average fitted ARD lengthscale across coordinates for learning rates $0.01$, $0.1$, and $1$; the black line shows the prior mode. \textbf{Top row:} isotropic GP-prior data with ground-truth $\ell=1$; \textbf{middle row:} Schwefel; \textbf{bottom row:} Styblinski--Tang. Training inputs are Sobol points. Curves show the mean and $95\%$ confidence interval over 10 independent replicates. GD is run for at most 1000 steps.}
\label{fig:Ablation learning rate}
\end{figure}

\subsection{Predictive density along BO trajectories}
\label{app:BO trace NLPD}

The predictive-density experiment in Figure~\ref{fig:NLPD_GPprior} uses Sobol training inputs. BO instead produces an adaptive, nonuniform sequence of observations, so we repeat the comparison on completed Vanilla-MAP trajectories. The trajectories optimize anisotropic SE GP-prior objectives with $\gamma=0.3$ in 30D, 50D, and 100D, using total budgets of 300, 500, and 1000 evaluations, respectively.

For each training size in Figure~\ref{fig:NLPD_BOtrace}, the isotropic and ARD models are fit to the first $n$ observations in the trajectory and evaluated on observations collected later in the same trajectory. This chronological split avoids using later BO observations to predict earlier ones. It also gives both fitted models exactly the same training and evaluation locations, isolating the surrogate comparison from differences in the points that Iso-MAP and Vanilla-MAP would themselves select.

The isotropic fits have lower NLPD at every reported dimension and training size in Figure~\ref{fig:NLPD_BOtrace}. The difference narrows as more observations are added, particularly in 100D, but remains visible over the reported range. The predictive-density advantage observed under Sobol sampling therefore also appears on the adaptive designs encountered during BO.

\begin{figure}[t]
\centering
\includegraphics[width=0.9\textwidth]{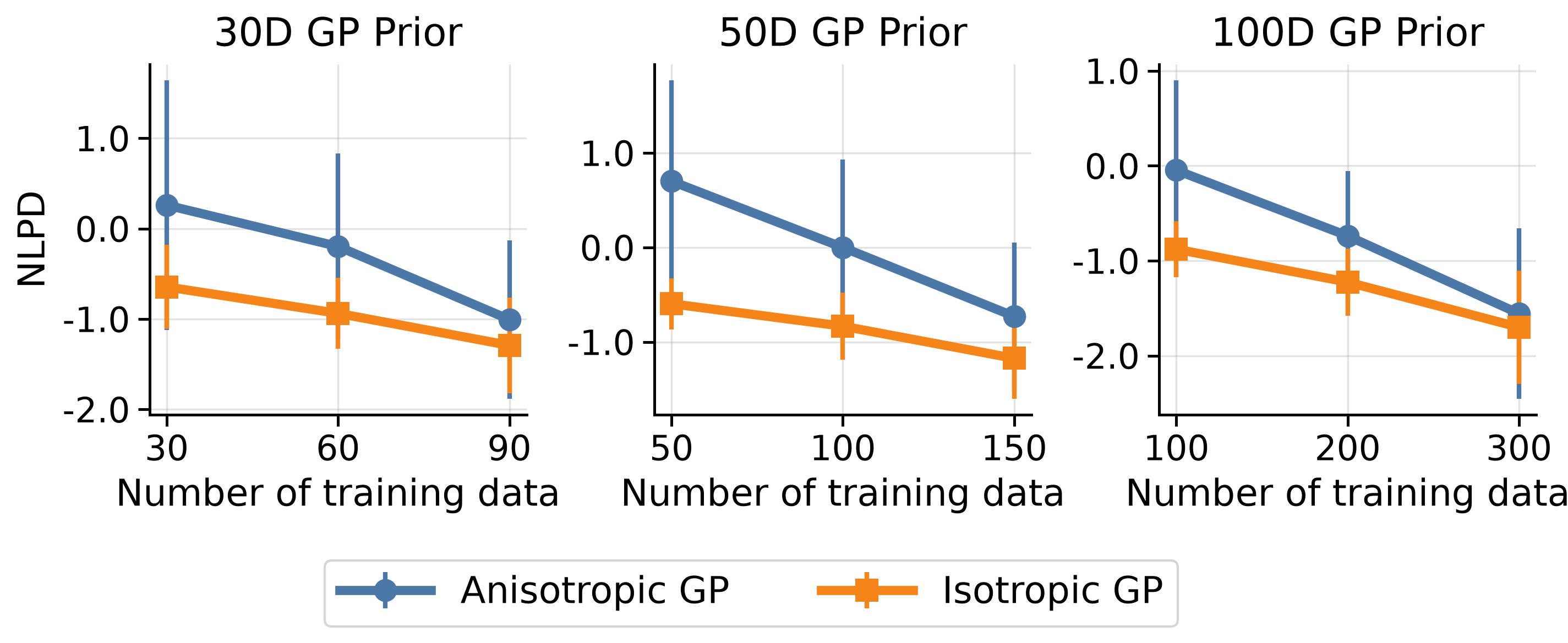}
\caption{Held-out NLPD for isotropic and ARD GP fits along Vanilla-MAP trajectories on anisotropic SE GP-prior objectives with $\gamma=0.3$. At each reported training size, both models are fit to the first $n$ observations and evaluated on later observations from the same completed trajectory. MAP estimation uses L-BFGS-B. Curves show the mean and $95\%$ confidence interval over 20 independent BO trajectories; lower NLPD is better.}
\label{fig:NLPD_BOtrace}
\end{figure}

\subsection{Predictive density on deterministic synthetic objectives}
\label{app:synthetic-nlpd-diagnostic}

We also test whether the predictive-density result extends beyond objectives sampled from a GP prior. Figure~\ref{fig:NLPD_Synthetic} considers Schwefel and Styblinski--Tang in 30D, 50D, and 100D. For each training size, isotropic and ARD GPs are fit by MAP with L-BFGS-B using the same Sobol training inputs, and predictive density is evaluated on held-out standardized responses.

The isotropic fits give lower NLPD in every panel, with the largest differences generally occurring at the smallest training sizes. Thus, the effect seen in Figure~\ref{fig:NLPD_GPprior} is not specific to GP-prior-generated objectives. Schwefel and Styblinski--Tang are coordinate-symmetric, so we use this result as a complementary non-GP predictive check; the anisotropic GP-prior experiment in Figure~\ref{fig:NLPD_GPprior} remains the more direct test when coordinatewise smoothness is heterogeneous.

\begin{figure}[t]
\centering
\includegraphics[width=0.9\textwidth]{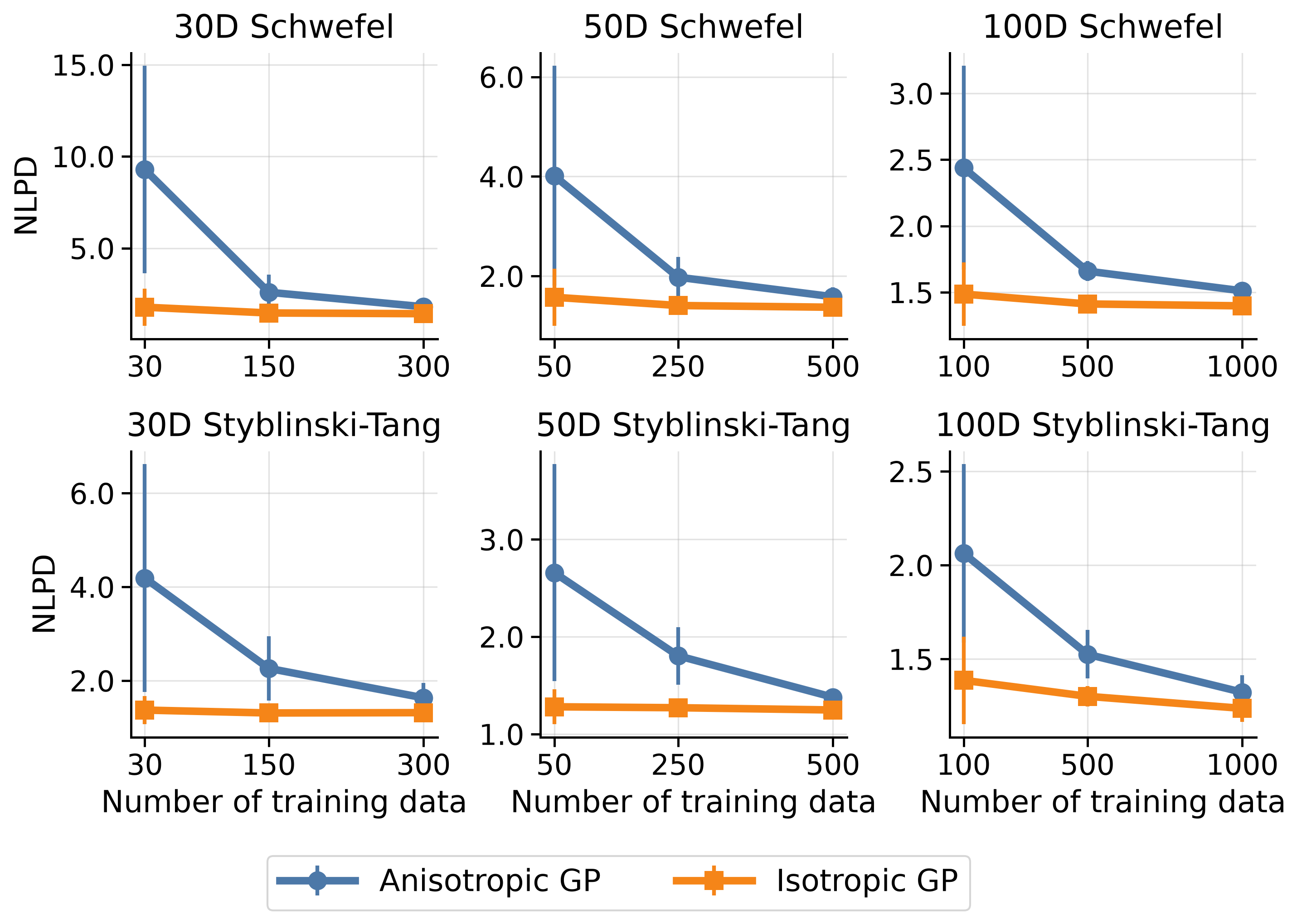}
\caption{Held-out NLPD for isotropic and ARD GP fits on Schwefel and Styblinski--Tang in 30D, 50D, and 100D. MAP estimation uses L-BFGS-B and training inputs are Sobol points. Curves show the median and $95\%$ confidence interval over 10 independent replicates; lower NLPD is better. The isotropic fits have lower NLPD throughout the reported low-data regimes.}
\label{fig:NLPD_Synthetic}
\end{figure}

\subsection{Additional lengthscale-fitting diagnostics}
\label{app:additional-lengthscale-diagnostics}

The main lengthscale diagnostic in Figure~\ref{fig:Fitted lscale} uses data from an isotropic SE GP prior and fits an SE kernel. Figures~\ref{fig:Iso_Fitted_Lscale_Matern}--\ref{fig:Fitted_Lscale_Synthetic} repeat the same fitting comparison under different kernels and data sources.

Figure~\ref{fig:Iso_Fitted_Lscale_Matern} retains the isotropic GP-prior data but fits a Mat\'ern~$5/2$ kernel. The qualitative behavior from the main SE diagnostic remains visible. In 50D and 100D, the L-BFGS-B ARD fit moves to substantially larger average lengthscales, while the isotropic L-BFGS-B fit remains much closer to the known shared scale. The GD fits remain nearer the prior-mode region. The optimizer-sensitive ARD behavior is therefore not specific to fitting the same SE kernel family used to generate the data.

\begin{figure}[t]
\centering
\includegraphics[width=0.9\textwidth]{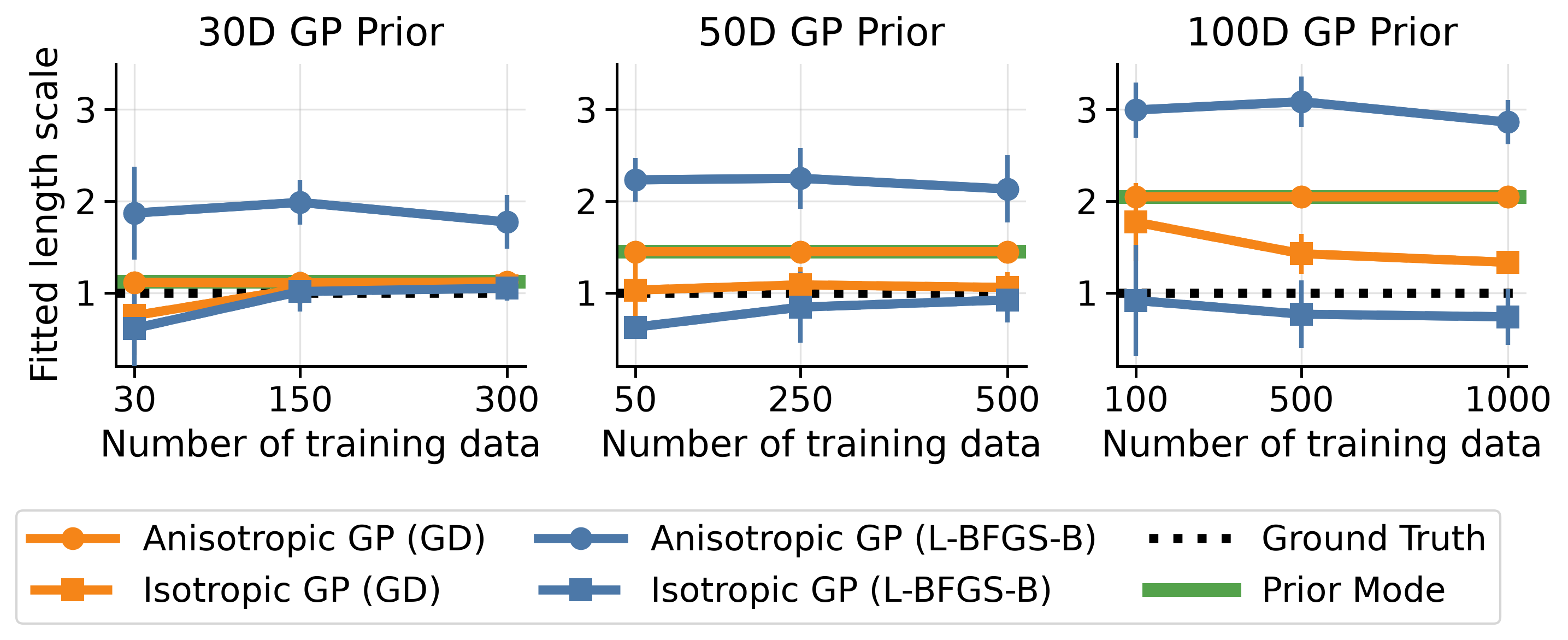}
\caption{Controlled lengthscale fitting on isotropic GP-prior data using Mat\'ern~$5/2$ fitted kernels. The data-generating shared lengthscale is $\ell=1$. MAP estimation uses L-BFGS-B or GD. Dashed curves denote isotropic GPs and solid curves denote ARD GPs; the black dotted line shows the ground truth and the green line shows the prior mode. Training inputs are Sobol points. Results show the mean and $95\%$ confidence interval over 10 independent GP-prior realizations. For ARD fits, the plotted value is the average fitted lengthscale across coordinates.}
\label{fig:Iso_Fitted_Lscale_Matern}
\end{figure}

Figures~\ref{fig:ARD_Fitted_Lscale_SE} and~\ref{fig:ARD_Fitted_Lscale_Matern} instead use data from anisotropic SE GP priors with $\gamma=0.5$, so the generating process has different lengthscales across coordinates. We fit these data with either an SE or Mat\'ern~$5/2$ kernel. In both cases, the average fitted ARD lengthscale remains strongly optimizer-dependent, particularly in 50D and 100D: L-BFGS-B moves to much larger average lengthscales than GD. The same pattern therefore persists when the data-generating process itself is anisotropic.

The black dotted line in these figures is the replicate-averaged mean of the generating coordinatewise lengthscales. It is included only as a scalar reference; there is no single ground-truth lengthscale for these anisotropic objectives, and the average fitted ARD lengthscale does not measure recovery of the full coordinatewise vector.

\begin{figure}[t]
\centering
\includegraphics[width=0.9\textwidth]{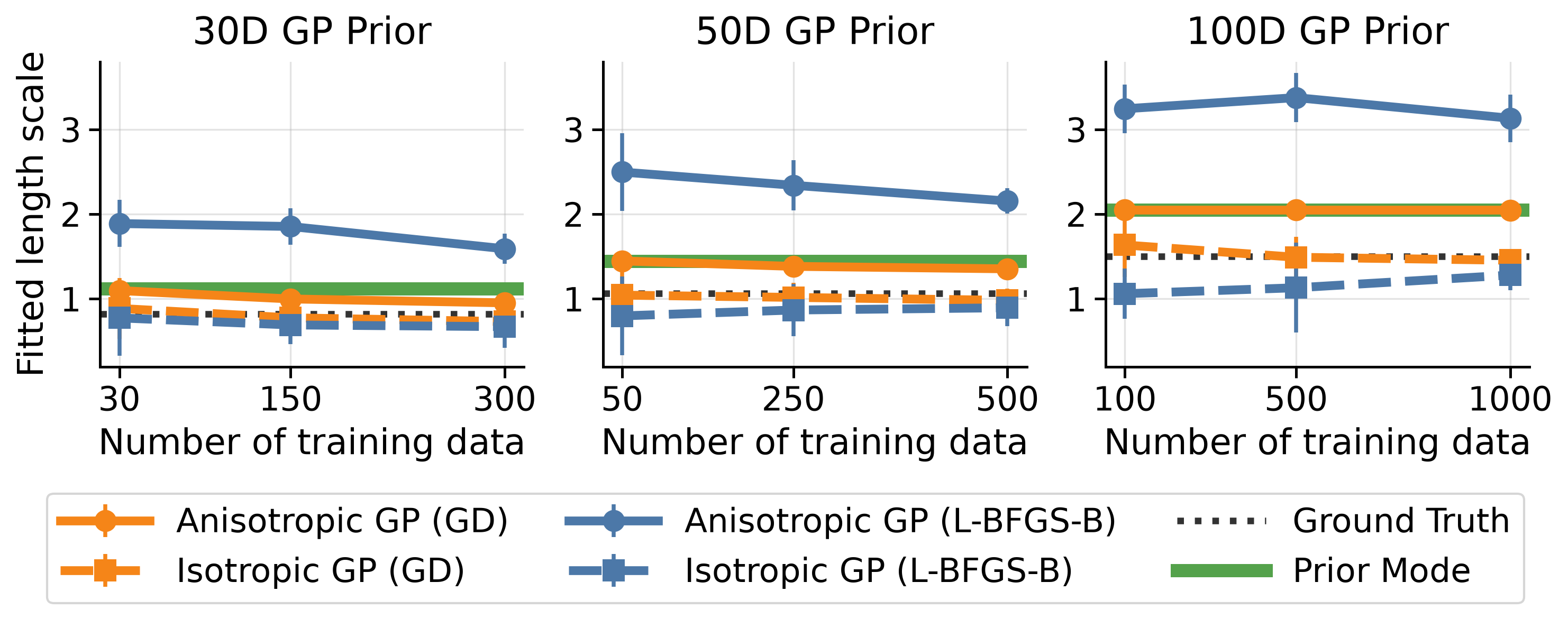}
\caption{Lengthscale fitting on anisotropic GP-prior data using SE fitted kernels. The data-generating coordinatewise lengthscales are sampled from Eq.~\ref{eq:length scale distribution} with $\gamma=0.5$. MAP estimation uses L-BFGS-B or GD. Dashed curves denote isotropic GPs and solid curves denote ARD GPs. The black dotted line shows the replicate-averaged mean of the generating coordinatewise lengthscales and the green line shows the prior mode. Results show the mean and $95\%$ confidence interval over 10 independent GP-prior realizations. For ARD fits, the plotted value is the average fitted lengthscale across coordinates. }
\label{fig:ARD_Fitted_Lscale_SE}
\end{figure}

\begin{figure}[t]
\centering
\includegraphics[width=0.9\textwidth]{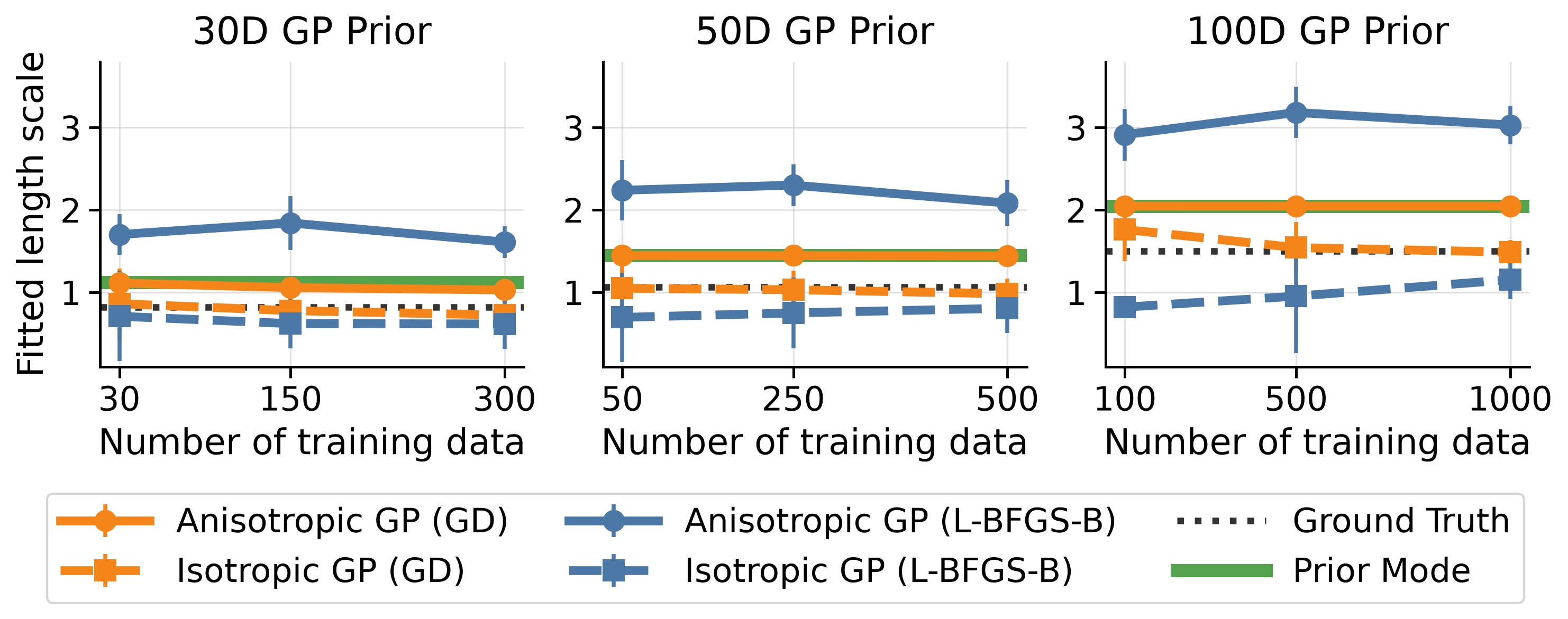}
\caption{Lengthscale fitting on the same anisotropic GP-prior setting as Figure~\ref{fig:ARD_Fitted_Lscale_SE}, now using Mat\'ern~$5/2$ fitted kernels. The data-generating coordinatewise lengthscales are sampled from Eq.~\ref{eq:length scale distribution} with $\gamma=0.5$. The black dotted line shows their replicate-averaged mean and the green line shows the prior mode. Results show the mean and $95\%$ confidence interval over 10 independent GP-prior realizations. For ARD fits, the plotted value is the average fitted lengthscale across coordinates. GD is run for at most 1000 steps.}
\label{fig:ARD_Fitted_Lscale_Matern}
\end{figure}

Finally, Figure~\ref{fig:Fitted_Lscale_Synthetic} applies the fitting diagnostic to deterministic Schwefel and Styblinski-Tang objectives. These functions do not have a GP-generating lengthscale, so the curves are not estimates of a known true parameter. They instead show how the fitted surrogate changes as additional observations are supplied. The optimizer dependence seen in the GP-prior experiments remains clear: L-BFGS-B produces substantially larger average ARD lengthscales than GD in many of the higher-dimensional settings, while GD remains closer to the prior-mode region.

\begin{figure}[t]
\centering
\includegraphics[width=0.9\textwidth]{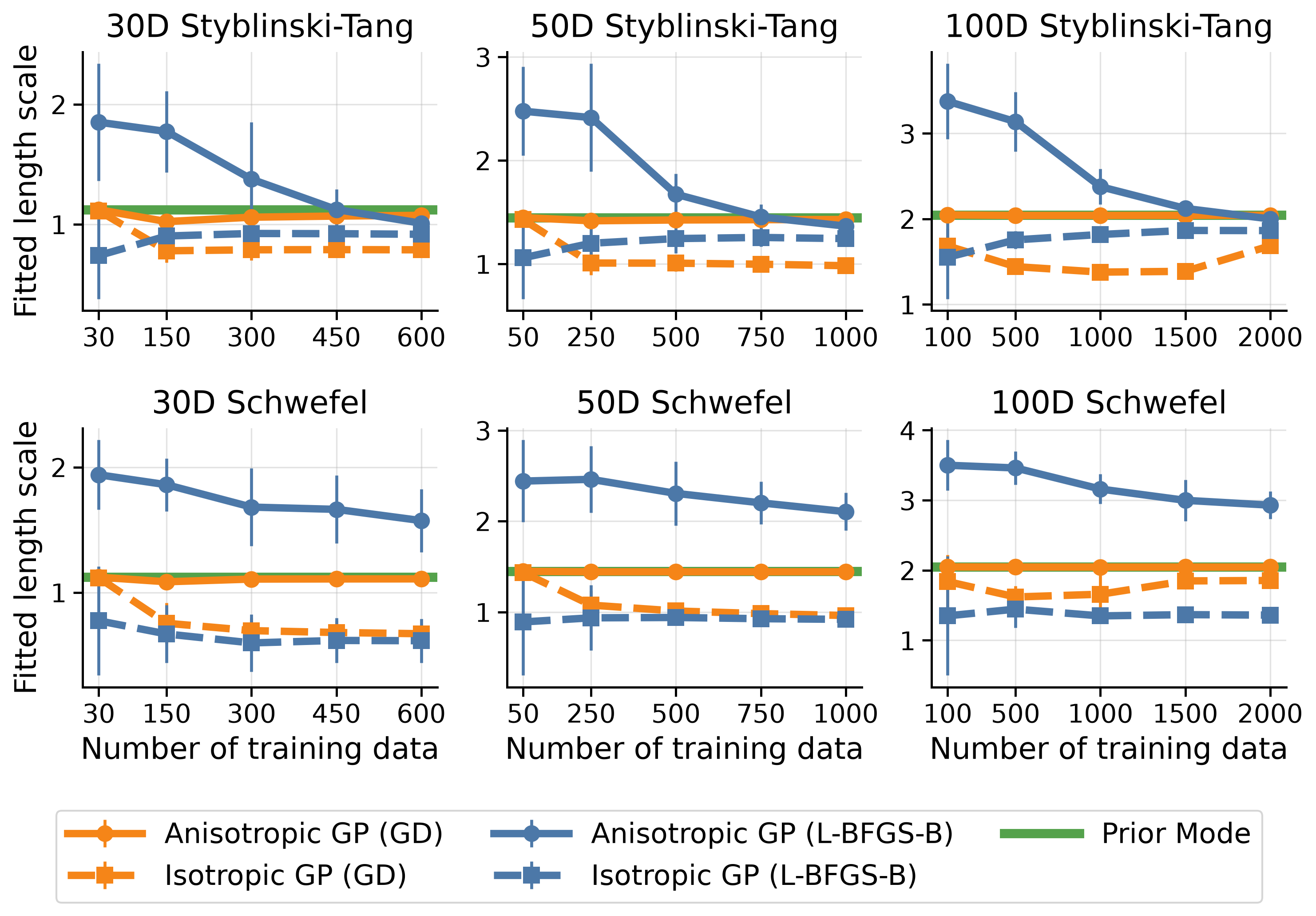}
\caption{Fitted lengthscales on Styblinski--Tang (\textbf{top row}) and Schwefel (\textbf{bottom row}) using SE fitted kernels. MAP estimation uses L-BFGS-B or GD with Sobol training inputs. The green line shows the prior mode. Results show the mean and $95\%$ confidence interval over 10 independent replicates; for ARD fits, the plotted value is the average fitted lengthscale across coordinates. These deterministic objectives do not have a GP-generating lengthscale. GD is run for at most 1000 steps.}
\label{fig:Fitted_Lscale_Synthetic}
\end{figure}

These checks support two points used in the main paper. First, the fitted ARD model can be sensitive to the numerical hyperparameter fit across kernels and objective classes. Second, the difference between isotropic and ARD fitting has predictive consequences: the isotropic models achieve lower held-out NLPD not only under the controlled Sobol designs in the main text, but also along BO trajectories and on deterministic objectives. The BO experiments then test whether this finite-data predictive advantage is useful for optimization when balanced against the loss of coordinatewise flexibility.

\section{Experimental and implementation details}
\label{app:experimental-details}

This appendix collects the experimental details used throughout the paper. We first state the common reporting conventions and benchmark definitions, then describe the anisotropic GP-prior objectives, GP fitting, BO settings, and baseline implementations. Additional settings that apply only to a particular experiment are given with that experiment in the corresponding appendix section.

\subsection{Objective and reporting conventions}
\label{app:objective-plotting-conventions}

All BO results use a minimization convention. Schwefel, Styblinski--Tang, Powell, SVM, and MOPTA08 are evaluated directly as minimization problems. For benchmarks whose native objective is a reward, we report negative reward, so lower values are better throughout the BO traces and final-value comparisons.

Unless stated otherwise, the sequential BO experiments use 10 independent runs with different random seeds and initial designs. The main BO figures report the median best-observed value, with shaded regions showing the interquartile range. Diagnostic figures use the summary statistics stated in their captions. GP inputs are represented in normalized coordinates for fitting and acquisition optimization and are mapped to the native benchmark domain by the corresponding objective wrapper.

For every GP fit, the observed responses are standardized before fitting. The standardization statistics are computed from the training responses used in that fit. When NLPD is evaluated on held-out responses, those responses are transformed using the same training set statistics; no held-out responses are used to determine the standardization.

\subsection{Benchmark definitions}
\label{app:benchmark-definitions}

\paragraph{Synthetic functions.}
For the $D$-dimensional Schwefel function,
\[
f(x)=418.9829D-\sum_{i=1}^{D}x_i\sin\!\left(\sqrt{|x_i|}\right),
\]
we use the domain $[-500,500]^D$. Its global minimum is $f(x^\star)=0$ at $x^\star=(420.9687,\ldots,420.9687)$. The Styblinski--Tang function is
\[
f(x)=\frac{1}{2}\sum_{i=1}^{D}\left(x_i^4-16x_i^2+5x_i\right),
\]
and is evaluated on $[-5,5]^D$. Its global minimum is approximately $-39.166D$ at $x^\star=(-2.9035,\ldots,-2.9035)$.
For Powell, the standard block-of-four form used in the manuscript is
\[
f(x)=\sum_{i=1}^{\lfloor D/4\rfloor}
\left[
(x_{4i-3}+10x_{4i-2})^2
+5(x_{4i-1}-x_{4i})^2
+(x_{4i-2}-2x_{4i-1})^4
+10(x_{4i-3}-x_{4i})^4
\right],
\]
with domain $[-4,5]^D$ and minimum value $0$. When $D$ is not divisible by $4$, the remaining coordinates are treated as inactive and do not affect the objective; $x^\star=0$ is therefore a global minimizer in all dimensions considered here.

The main synthetic experiments use $D=100$ and a budget of 1000 evaluations. Appendix~\ref{app:extended-bo-evidence} also reports 50D experiments with a budget of 500 evaluations.

\paragraph{Rover.}
Rover is a path-planning benchmark used in HDBO studies~\citep{wang2018batched,eriksson2019scalable,ngo2024high}. The main 100D version parameterizes the path with 50 two-dimensional B-spline control points. Inputs are optimized in normalized coordinates and mapped to the control-point range used by the benchmark wrapper. We also report a 60D variant using 30 two-dimensional control points in Appendix~\ref{app:extended-bo-evidence}.

\paragraph{MuJoCo controller-tuning benchmarks.}
HalfCheetah-v2 and Humanoid-v2 optimize the parameters of linear control policies in MuJoCo-v2 environments following the benchmark implementation used in prior HDBO work~\citep{wang2020learning}. The reported search spaces have dimensions 102 and 6392, respectively. We use the implementation available at
\url{https://github.com/DonneyF/mujoco-v2-for-global-optimization}
and evaluate each policy using 10 rollouts rather than the default 5 to reduce variability from randomized initial states. Appendix~\ref{app:extended-bo-evidence} additionally considers the 102D Walker2d-v2 controller-tuning problem under the same protocol.

\paragraph{Action-sequence benchmarks.}
Walker2d-v5 and HumanoidStandup-v5 directly optimize action sequences. Walker2d-v5 uses a horizon of 50 time steps and 6 action variables per step, giving $D=300$. HumanoidStandup-v5 uses 59 time steps and 17 action variables per step, giving $D=1003$~\citep{xu2024standard,towers2024gymnasium}. For both tasks, inputs are optimized in $[0,1]^D$ and mapped to the environment action range by the wrapper, and the initial simulator state is fixed throughout optimization. Both problems are simulated in the Mujoco-v5 environment.

\paragraph{SVM.}
The SVM benchmark is a 388D hyperparameter-optimization problem consisting of 385 kernel lengthscale parameters and 3 regularization parameters~\citep{eriksson2021high}. We use the implementation from
\url{https://github.com/XZT008/Standard-GP-is-all-you-need-for-HDBO/blob/main/benchmark/svm_benchmark.py}.

\paragraph{MOPTA08.}
MOPTA08 is a 124D vehicle-design problem with 68 constraints~\citep{jones2008large}. Following prior HDBO use of this benchmark~\citep{eriksson2021high,xu2024standard}, the experiment in Appendix~\ref{app:mopta08-extended} converts the constrained problem to the penalized scalar objective
\[
f_{\mathrm{pen}}(x)
=
f_{\mathrm{raw}}(x)
+
10\sum_{j=1}^{68}\max\{c_j(x),0\},
\]
where $c_j(x)>0$ denotes violation of constraint $j$.

\paragraph{Molecular optimization.}
The batch experiments in Appendix~\ref{app:batch BO} consider four GuacaMol goal-directed molecular-design tasks~\citep{brown2019guacamol} in the 256D latent space of a pretrained SELFIES-VAE~\citep{maus2022local}, following the setup used by \citet{doumont2025we}. Fexofenadine MPO (\texttt{fexo}) and Ranolazine MPO (\texttt{rano}) combine similarity to a reference molecule with physicochemical objectives, while Median Molecules 1 and 2 (\texttt{med1}, \texttt{med2}) reward structural compromise between two reference compounds. Each latent coordinate is restricted to $[-8,8]$. The batch sizes and evaluation budgets are given in Appendix~\ref{app:batch BO}.

\subsection{Anisotropic GP-prior objectives}
\label{app:gp-prior-objectives}

The controlled GP-prior experiments use random functions generated from anisotropic SE GP priors. Each coordinate has its own generating lengthscale, so these objectives are deliberately outside the isotropic model class. Following prior work~\citep{muller2021local,hennig2012entropy}, we scale the lengthscales with the expected Euclidean distance $\Delta(D)=\mathbb E[\|x-x'\|_2]$ between two independent uniform points. Using the upper bound from \citet{anderssen1976concerning},
\[
\Delta(D)
\leq
\sqrt{\frac{D}{6}}
\sqrt{
\frac{1+2\sqrt{1-\frac{3}{5D}}}{3}
},
\]
we denote the right-hand side by $\Delta(D)_{\mathrm{sub}}$ and rescale it so that $\Delta(2)_{\mathrm{sub}}=0.1$. The coordinatewise generating lengthscales are then sampled independently as
\begin{equation}
\ell_k\sim
\mathcal U\!\left(
2\Delta(D)_{\mathrm{sub}}(1-\gamma),
2\Delta(D)_{\mathrm{sub}}(1+\gamma)
\right),
\label{eq:length scale distribution}
\end{equation}
where $\gamma$ controls the amount of anisotropy (with larger values of $\gamma$ producing a wider spread of coordinatewise lengthscales). 

Each random GP-prior objective is approximated with $M=4096$ random Fourier features~\citep{rahimi2007random,wilson2020efficiently}. The main BO experiments use 100D objectives with $\gamma\in\{0.3,0.5,0.7\}$ and a budget of 1000 evaluations. The corresponding 50D experiments use a budget of 500 evaluations. Appendix~\ref{app:extended-bo-evidence} additionally considers the more strongly anisotropic cases $\gamma=0.9$ and $0.99$, as well as objectives with inactive dimensions. The predictive-density experiment in Figure~\ref{fig:NLPD_GPprior} uses $\gamma=0.3$, while the anisotropic lengthscale diagnostics in Figures~\ref{fig:ARD_Fitted_Lscale_SE} and~\ref{fig:ARD_Fitted_Lscale_Matern} use $\gamma=0.5$.

\subsection{GP models and hyperparameter fitting}
\label{app:gp-model-fitting}

The GP-based methods are implemented with BoTorch~0.16.1~\citep{balandat2020botorch} and GPyTorch~1.14.3~\citep{gardner2018gpytorch}. The central comparison is between Vanilla-MAP and Iso-MAP. Vanilla-MAP uses one ARD lengthscale per input coordinate, whereas Iso-MAP replaces these $D$ parameters with one shared lengthscale. Apart from this parameterization change, the two methods use the same GP-fitting and BO machinery.

For Vanilla-MAP and Iso-MAP, the kernel signal variance is fixed at one and the free lengthscale and noise parameters are fit by MAP. Fixing the signal variance is paired with the output standardization described above. Vanilla-MAP uses BoTorch's \texttt{get\_covar\_module\_with\_dim\_scaled\_prior} with \texttt{ard\_num\_dims=D}. Iso-MAP uses the corresponding custom isotropic covariance module with \texttt{ard\_num\_dims=None}, while retaining the ambient dimension $D$ when constructing the prior. In both cases the raw lengthscale prior is
\[
\ell\sim
\operatorname{LogNormal}
\!\left(
\sqrt{2}+\frac{1}{2}\log D,\sqrt{3}
\right),
\]
following the dimension-scaled prior of \citet{hvarfner2024vanilla}. Vanilla-MAP applies this univariate prior to each of its $D$ ARD lengthscales; Iso-MAP applies the same ambient-$D$-scaled prior once to the shared lengthscale. Appendix~\ref{app:isotropic-restriction-proof} explains how this implemented prior relates to the inverse-squared parameterization used in the theory.

We also report MLE-based variants. MSR~\citep{papenmeier2025understanding} uses an ARD GP with dimension-aware lengthscale initialization, while Iso-MLE uses the corresponding initialization strategy with one shared lengthscale. The current main experiments use an SE kernel for Vanilla-MAP, Iso-MAP, and Iso-MLE, whereas MSR uses its Mat\'ern~$5/2$ kernel configuration. This distinction is summarized in Table~\ref{tab:gp-variant-summary}. The Iso-MLE--MSR results therefore provide useful contextual evidence, but they are not a kernel-matched isolation of the ARD-to-isotropic change. The controlled lengthscale comparison in the main paper is Vanilla-MAP versus Iso-MAP; additional matched kernel results are reported in Appendix~\ref{app:Matern}.

\begin{table}[t]
\centering
\caption{GP variants used in the main BO experiments. Vanilla-MAP and Iso-MAP form the primary matched comparison: they use the same kernel family and fitting procedure and differ in the lengthscale parameterization. The current MSR and Iso-MLE main runs use different kernel families and are therefore treated as a secondary comparison.}
\small
\setlength{\tabcolsep}{7pt}
\begin{tabular}{llll}
\hline
Method & Lengthscale model & Hyperparameter fit & Main kernel \\
\hline
Vanilla-MAP & ARD ($D$ parameters) & MAP & SE \\
Iso-MAP & Shared (1 parameter) & MAP & SE \\
MSR & ARD ($D$ parameters) & MLE & Mat\'ern~$5/2$ \\
Iso-MLE & Shared (1 parameter) & MLE & SE \\
\hline
\end{tabular}
\label{tab:gp-variant-summary}
\end{table}

\subsection{BO and acquisition-optimization settings}
\label{app:shared-bo-settings}

The sequential Vanilla-MAP and Iso-MAP runs begin with 10 Sobol points and select subsequent evaluations using LogEI~\citep{ament2023unexpected}. Acquisition optimization uses BoTorch's \texttt{optimize\_acqf} with five restarts and a maximum of 300 optimizer iterations. We use RAASP-style acquisition initialization~\citep{papenmeier2025understanding}, activated through \texttt{sample\_around\_best=True}. These settings are shared between Vanilla-MAP and Iso-MAP so that their direct comparison isolates the GP lengthscale parameterization.

GP hyperparameters are fit with \texttt{fit\_gpytorch\_mll}. In the reported sequential experiments, Vanilla-MAP, Iso-MAP, MSR, and Iso-MLE refit their GP hyperparameters every 5 BO evaluations in 50D and every 10 BO evaluations in all higher-dimensional settings. In the batch experiments of Appendix~\ref{app:batch BO}, the GP is refit at every BO iteration and the acquisition function is qLogEI with 256 Monte Carlo samples.

The evaluation budget is 1000 for the main 100D GP-prior and synthetic experiments and for the main real-world benchmarks. The 50D GP-prior and synthetic experiments use 500 evaluations. SAASBO is run with smaller budgets in the experiments where its computational cost makes the full budget impractical; the exact budget is stated in the corresponding figure caption. Batch and molecular-optimization budgets are stated in Appendix~\ref{app:batch BO}.

\subsection{Baseline implementations}
\label{app:baselines}

The additional methods provide HDBO context for the matched Vanilla-MAP versus Iso-MAP comparison. We use their existing implementations where possible and state the manuscript-level changes below.

\paragraph{MSR.}
MSR~\citep{papenmeier2025understanding} is implemented from the authors' public code at
\url{https://github.com/lpapenme/understanding-hdbo-code}.
It uses MLE, dimension-aware lengthscale initialization, an ARD Mat\'ern~$5/2$ GP, LogEI, and RAASP-style acquisition initialization.

\paragraph{Iso-MLE.}
Iso-MLE uses the same MLE and dimension-aware initialization principle as MSR but replaces the $D$ ARD lengthscales with one shared lengthscale through the isotropic covariance module. As noted above, the current main Iso-MLE runs use an SE kernel, so this comparison with MSR is not kernel-matched.

\paragraph{TuRBO.}
TuRBO~\citep{eriksson2019scalable} is based on the BoTorch TuRBO tutorial implementation,
\url{https://botorch.org/docs/tutorials/turbo_1/}.
The reported runs use LogEI for candidate selection. RAASP-style acquisition initialization is not used for this baseline.

\paragraph{BAxUS.}
BAxUS~\citep{papenmeier2022increasing} is based on the BoTorch BAxUS tutorial,
\url{https://botorch.org/docs/tutorials/baxus/},
with LogEI used for the reported experiments. BAxUS initializes and optimizes within a lower-dimensional subspace, so its initial best-observed value need not match methods initialized directly in the full input space.

\paragraph{Linear BO.}
Linear BO~\citep{doumont2025we} uses the public implementation at
\url{https://github.com/colmont/linear-bo}.

\paragraph{SAASBO.}
SAASBO~\citep{eriksson2021high} uses the BoTorch tutorial implementation,
\url{https://botorch.org/docs/tutorials/saasbo},
with the tutorial hyperparameter settings. Its substantially larger computational cost motivates the reduced evaluation budgets used in several experiments.

\paragraph{CMA-ES.}
CMA-ES~\citep{hansen2006cma} is implemented using \texttt{pycma}, with initial sampling standard deviation $0.5$. CMA-ES begins from one starting point rather than the Sobol initial design used by the matched GP methods, so the initial best-observed values in the traces can differ.

\paragraph{Sobol.}
The Sobol baseline uses \texttt{torch.quasirandom.SobolEngine} to sample the normalized search domain without fitting a surrogate.

Note that we also propose \textbf{BO-FIM} in Appendix \ref{app:BO-FIM}, which implements a switching rule to decide between using the isotropic or ARD GP at each BO iteration. The switching rule depends on the Fisher information matrix of the GP lengthscales, which are estimated from the observation data.

\subsection{Code and compute}
\label{app:code-compute-note}

The experiment code included with the submission contains the benchmark wrappers, model implementations, and BO loops used for the reported experiments. We plan to release the complete codebase publicly upon acceptance.

Experiments were run on institutional computing clusters. The isolated GP-fitting timing experiment reported in Appendix~\ref{app:extended-bo-evidence} uses 4-core CPU jobs on nodes with Intel Xeon CPU Max 9470 processors and 512~GB DDR5 RAM. The end-to-end GPU runtime comparison in the same appendix uses an NVIDIA L40 GPU with 45~GB of GPU memory. Because the timing studies compare methods under common hardware, their reported wall-clock values should be interpreted for those specific implementations and compute environments rather than as hardware-independent complexity measurements.

\section{Additional sequential BO evidence and analysis}
\label{app:extended-bo-evidence}

This appendix collects the additional sequential BO results that support and qualify the main experiments. We first examine the final-value distributions and paired Iso-MAP versus Vanilla-MAP comparisons, then stress-test the isotropic assumption under stronger anisotropy and inactive dimensions. We also report results in additional dimensions and benchmark settings, examine sensitivity to the kernel family, compare where the two methods sample, and summarize the computational cost of fitting the different surrogate parameterizations.

\subsection{Final-value distributions and paired Iso-MAP--Vanilla-MAP analysis}
\label{app:final-value-distributions}
\label{app:paired-analysis}

The main BO figures emphasize optimization trajectories through the median best-observed value. Figures~\ref{fig:synthetic_violin}--\ref{fig:realwworld_violin} show the corresponding final best-observed values across the independent runs, which are meant to make the final run-to-run variation explicit and complement the paired analysis below.

\begin{figure}[t]
\centering
\includegraphics[width=\textwidth]{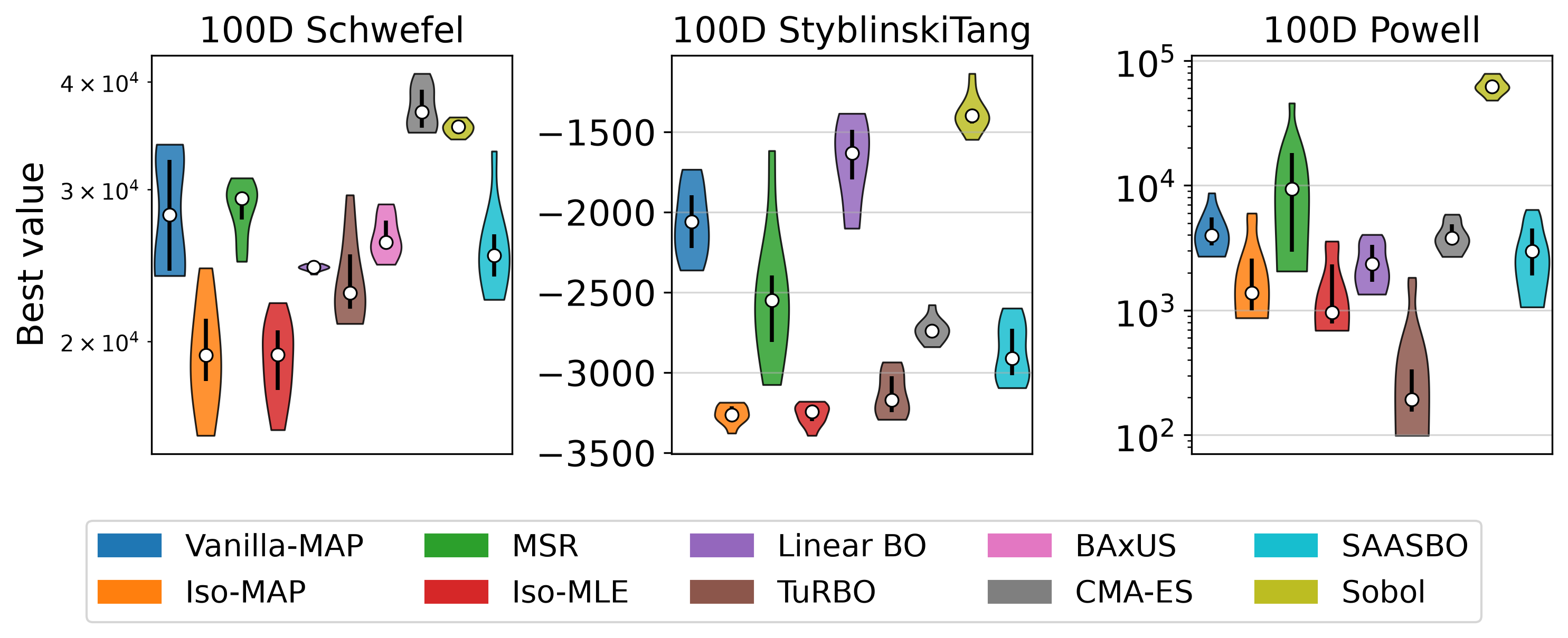}
\caption{Final best-observed value distributions for the 100D Schwefel, Styblinski--Tang, and Powell experiments. Each violin summarizes the final best-observed value over 10 independent runs. Lower values are better. This figure complements the optimization trajectories in Figure~\ref{fig:synthetic} by showing the run-to-run variation at the final budget.}
\label{fig:synthetic_violin}
\end{figure}

\begin{figure}[t]
\centering
\includegraphics[width=\textwidth]{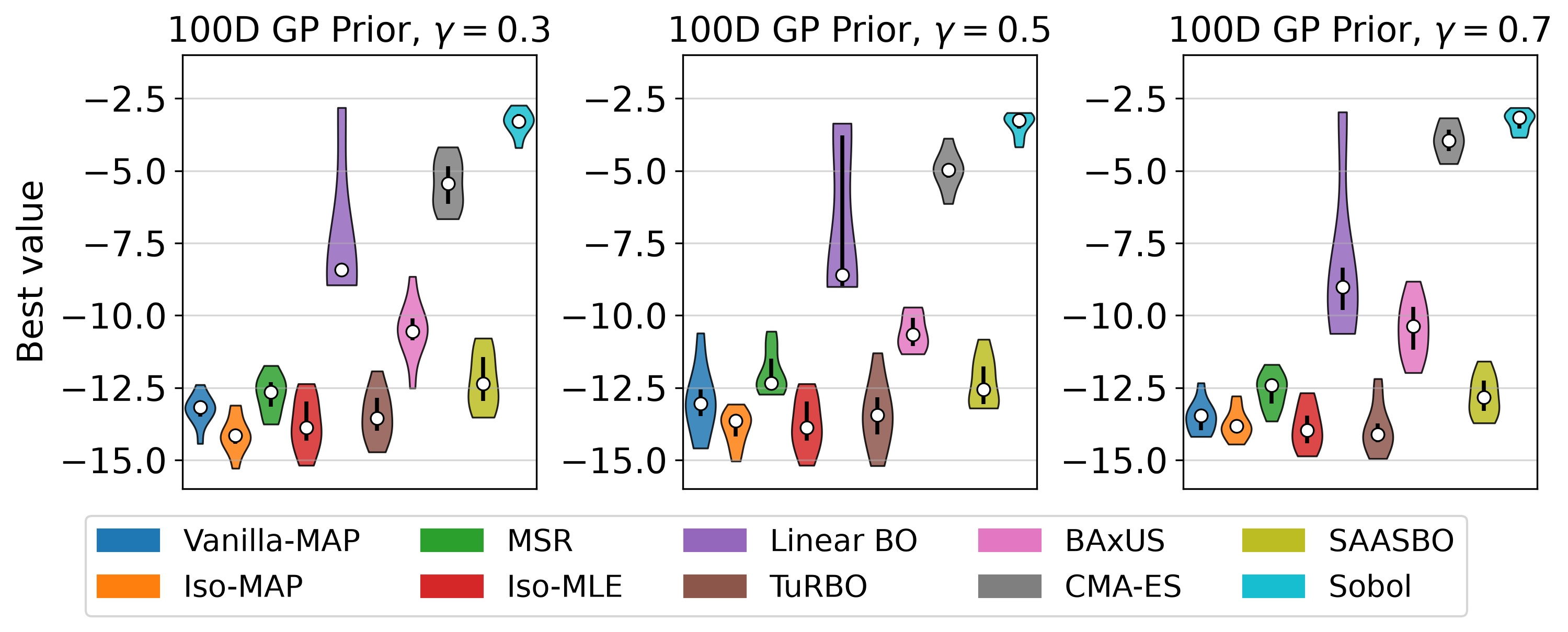}
\caption{Final best-observed value distributions for the 100D anisotropic GP-prior experiments with $\gamma\in\{0.3,0.5,0.7\}$. Each violin summarizes the final best-observed value over 10 independent runs. Lower values are better. The objectives and baseline implementations are the same as in Figure~\ref{fig:GPprior}.}
\label{fig:GPprior_violin}
\end{figure}

\begin{figure}[t]
\centering
\includegraphics[width=\textwidth]{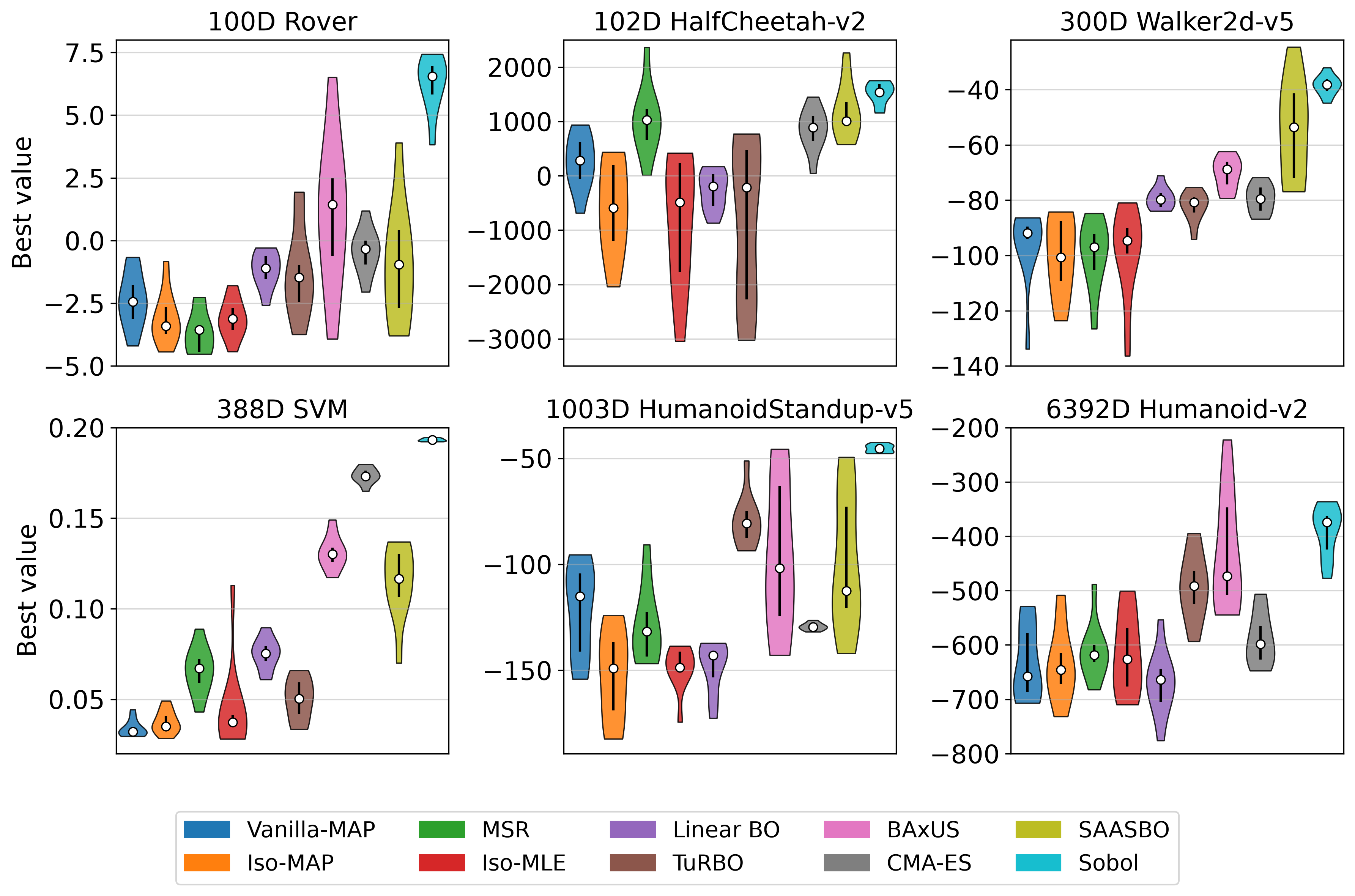}
\caption{Final best-observed value distributions for the real-world benchmarks. Each violin summarizes the final value over 10 independent runs. The final budget is 1000 evaluations for the main methods; SAASBO uses 500 evaluations on Rover, HalfCheetah-v2, Walker2d-v5, and SVM and 300 evaluations on HumanoidStandup-v5, and is not run on Humanoid-v2. Lower values are better; reward-maximization tasks are reported as negative reward.}
\label{fig:realwworld_violin}
\end{figure}

Because Iso-MAP and Vanilla-MAP use matched initial designs in each replicate, we compare their final values using paired differences. Let $\Delta=f_{\mathrm{Iso}}^*-f_{\mathrm{Vanilla}}^*$ denote the difference in final best-observed value for a given replicate, so $\Delta<0$ favors Iso-MAP. Table~\ref{tab:paired-analysis-all} reports the median paired difference together with a $95\%$ bootstrap percentile confidence interval for the median. For each problem, the interval is computed by resampling the 10 paired differences with replacement 10,000 times, calculating the median for each bootstrap sample, and taking the $2.5$th and $97.5$th percentiles of the resulting bootstrap distribution. We also report the number of replicates in which Iso-MAP finishes lower and a one-sided Wilcoxon signed-rank test for a negative paired shift. Note that the reported Wilcoxon $p$-values are unadjusted for multiple comparisons and are used as benchmark-wise paired evidence.

The synthetic results are the clearest: Iso-MAP has a negative median paired difference in all three cases, wins at least 9 of 10 replicates, and gives $p<0.05$ in each one-sided test. For the anisotropic GP-prior objectives, the paired test is significant for $\gamma=0.3$ and $0.5$, but not for $\gamma=0.7$. Among the real-world benchmarks, the test provides evidence of lower final values for Rover, HalfCheetah-v2, and HumanoidStandup-v5. The differences on Walker2d-v5, SVM, and Humanoid-v2 are not significant at the $0.05$ level. 
% These results are useful for interpreting the trajectory plots: visible median separation should not be treated as a definitive improvement when the paired run-level evidence is weak.

\begin{table}[t]
\centering
\caption{Paired final-value comparison between Iso-MAP and Vanilla-MAP. Here $\Delta=f_{\mathrm{Iso}}^*-f_{\mathrm{Vanilla}}^*$, so negative values favor Iso-MAP. The $95\%$ interval is the bootstrap percentile confidence interval for the median paired difference. ``Iso wins'' is the number of paired replicates for which Iso-MAP has a lower final best-observed value. The final column gives the one-sided Wilcoxon signed-rank $p$-value for a negative paired shift.}
\small
\setlength{\tabcolsep}{4.5pt}
\renewcommand{\arraystretch}{1.08}
\begin{tabular}{l r c c r}
\hline
Problem & Median $\Delta$ & $95\%$ paired CI & Iso-MAP wins & $p$ \\
\hline
\multicolumn{5}{l}{\emph{Synthetic objectives}} \\
100D Schwefel          & -9451.71 & [-14169.33, -4001.42] & 9/10  & 0.002 \\
100D Styblinski--Tang  & -1150.74 & [-1339.72, -1049.55]  & 10/10 & 0.001 \\
100D Powell            & -3898.00 & [-5283.70, -2611.16]  & 9/10  & 0.002 \\
\hline
\multicolumn{5}{l}{\emph{Anisotropic GP-prior objectives}} \\
100D, $\gamma=0.3$     & -0.88    & [-1.33, -0.09]        & 8/10  & 0.010 \\
100D, $\gamma=0.5$     & -0.56    & [-1.20, -0.14]        & 8/10  & 0.007 \\
100D, $\gamma=0.7$     & -0.31    & [-0.77, 0.26]         & 8/10  & 0.120 \\
\hline
\multicolumn{5}{l}{\emph{Real-world benchmarks}} \\
100D Rover                 & -0.49   & [-2.00, 0.04]          & 8/10 & 0.040 \\
102D HalfCheetah-v2        & -516.17 & [-1413.75, -244.53]    & 9/10 & 0.002 \\
300D Walker2d-v5           & -1.01   & [-13.39, 5.87]         & 6/10 & 0.250 \\
388D SVM                   & 0.003   & [-0.0002, 0.006]       & 2/10 & 0.980 \\
1003D HumanoidStandup-v5   & -36.40  & [-48.05, -9.67]       & 8/10 & 0.005 \\
6392D Humanoid-v2          & -39.43  & [-95.53, 85.06]       & 6/10 & 0.460 \\
\hline
\end{tabular}
\label{tab:paired-analysis-all}
\label{tab:paired analysis synthetic}
\label{tab:paired analysis gp prior}
\label{tab:paired analysis realworld}
\end{table}

\subsection{Stress tests: strong anisotropy and inactive dimensions}
\label{app:stress-tests}
\label{app: highly anisotropic GP prior}

The main GP-prior experiments use $\gamma\leq0.7$. We first push the same construction to $\gamma=0.9$ and $0.99$ to test where the isotropic restriction becomes too costly. Figure~\ref{fig:GP prior large gamma} shows results in 50D and 100D. At $\gamma=0.9$, Iso-MAP remains close to Vanilla-MAP over the reported budgets. At $\gamma=0.99$, however, the matched comparison reverses: Vanilla-MAP reaches lower final median values than Iso-MAP in both dimensions. Iso-MAP still makes rapid early progress, but its single shared scale eventually limits its ability to represent the much stronger coordinatewise variation.

This is an important boundary case for the paper in that Iso-BO is intended as a finite-data bias, not as a claim that isotropy is always preferable. When the objective is sufficiently anisotropic and the available data begin to support learning that structure, the flexibility removed by isotropic tying can become valuable.

\begin{figure}[t]
\centering
\includegraphics[width=\textwidth]{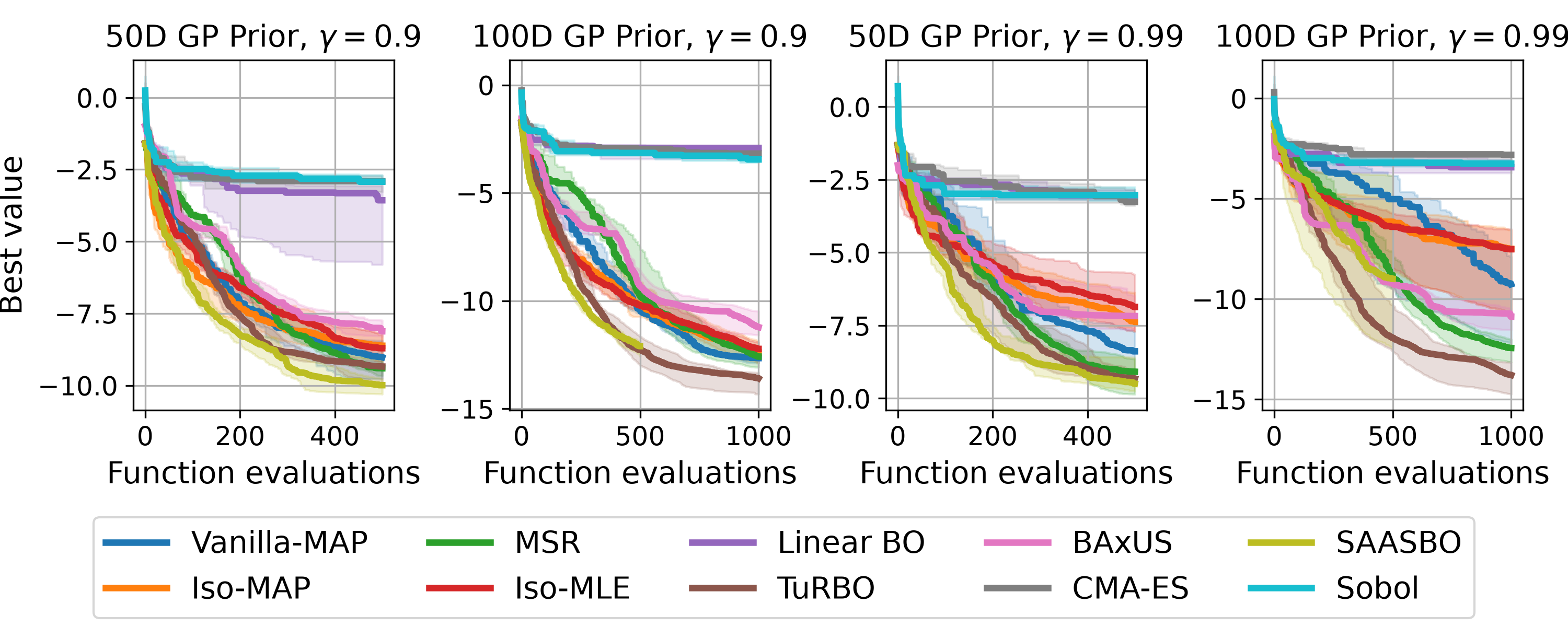}
\caption{Best-observed values on highly anisotropic GP-prior objectives in 50D and 100D with $\gamma\in\{0.9,0.99\}$. The budgets are 500 evaluations in 50D and 1000 in 100D. Iso-MAP remains close to matched Vanilla-MAP at $\gamma=0.9$, while Vanilla-MAP reaches lower final median values at $\gamma=0.99$.}
\label{fig:GP prior large gamma}
\end{figure}

We next consider a different type of structural mismatch: inactive dimensions. The objectives in Figure~\ref{fig:GP prior sparse} have 100 input dimensions, but only 50 affect the function value. An ARD model can in principle represent this structure through very long lengthscales on inactive coordinates, whereas the isotropic model cannot treat those coordinates separately. Despite this mismatch, Iso-MAP remains competitive with Vanilla-MAP across the three tested anisotropy levels and reaches a lower median trace for $\gamma=0.5$ under the 500-evaluation budget. These experiments show that the finite-data advantage of tying lengthscales can remain useful even when the objective contains axis-aligned sparsity, although Iso-BO does not explicitly model that sparsity.

\begin{figure}[t]
\centering
\includegraphics[width=\textwidth]{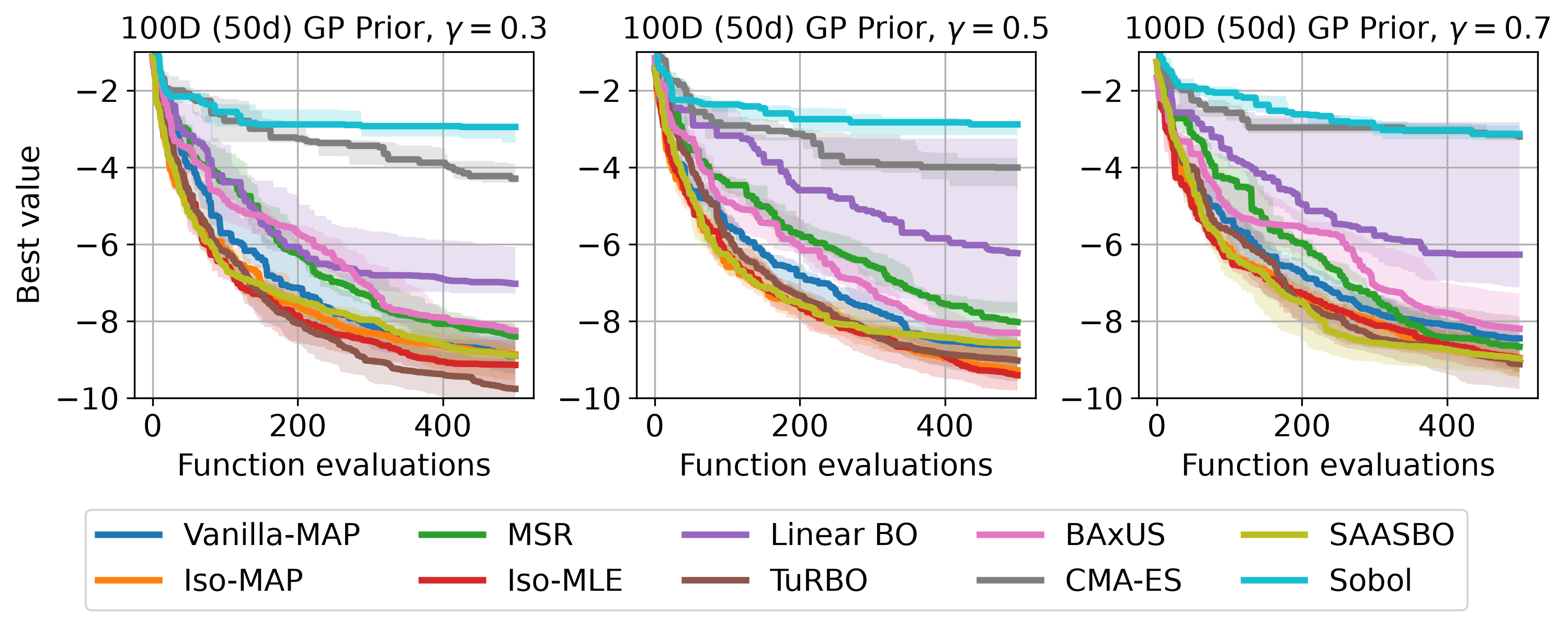}
\caption{Best-observed values on 100D anisotropic GP-prior objectives with 50 inactive dimensions. The budget is 500 evaluations. Iso-MAP remains competitive with matched Vanilla-MAP for $\gamma=0.3$ and $0.7$ and reaches a lower median trace for $\gamma=0.5$.}
\label{fig:GP prior sparse}
\end{figure}

\subsection{Additional benchmark coverage}
\label{app:additional-benchmark}
\label{app:additional benchmark}

\paragraph{50D GP-prior and synthetic objectives.}
The main text emphasizes 100D versions of the controlled GP-prior and synthetic benchmarks. Figures~\ref{fig:GPprior_50D} and~\ref{fig:Synthetic_50D} report their 50D counterparts with a budget of 500 evaluations. On the anisotropic GP-prior objectives, Iso-MAP reaches lower final median values than Vanilla-MAP for $\gamma=0.3$ and $0.5$ and remains close at $\gamma=0.7$. The 50D synthetic results show the same broad separation as the 100D experiments, with Iso-MAP finishing below Vanilla-MAP on Schwefel, Styblinski--Tang, and Powell.

\label{sec:50D Synthetic and GP Prior}

\begin{figure}[t]
\centering
\includegraphics[width=\textwidth]{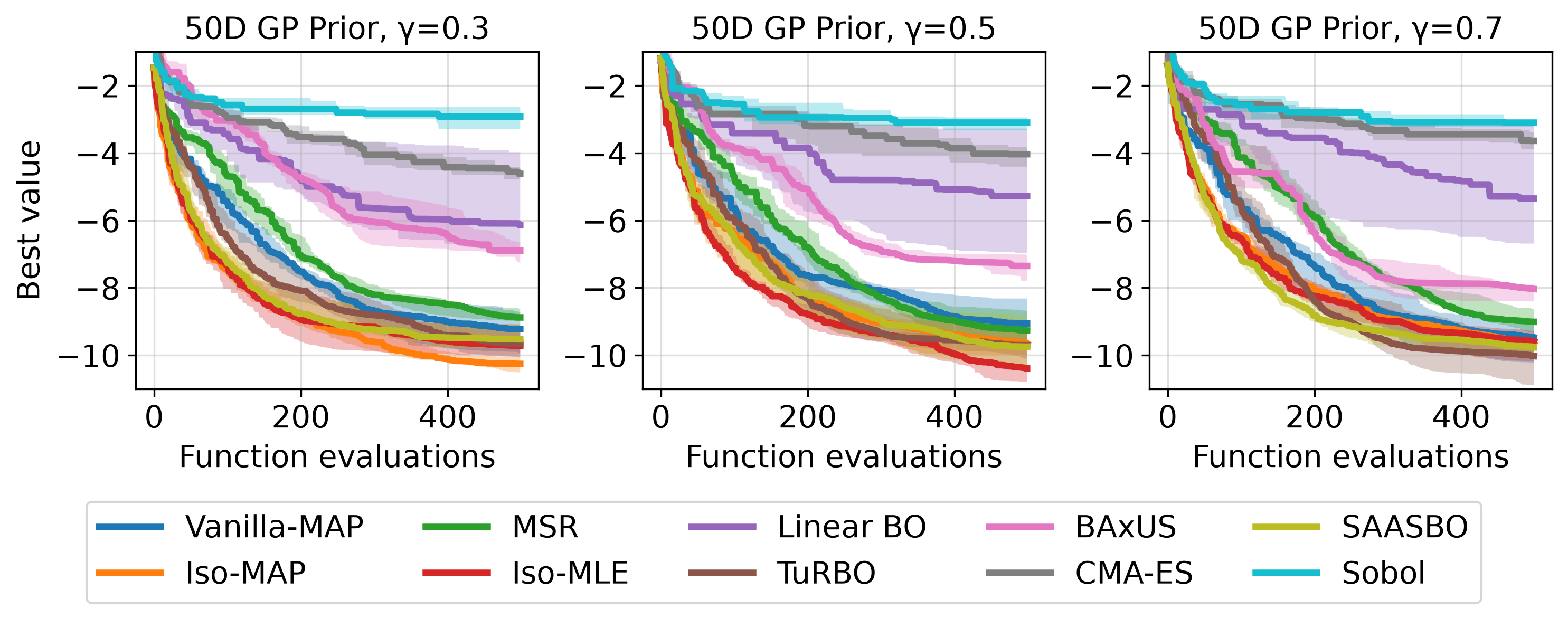}
\caption{Best-observed values on anisotropic GP-prior objectives in 50D with $\gamma\in\{0.3,0.5,0.7\}$. The budget is 500 evaluations. Iso-MAP reaches lower final median values than matched Vanilla-MAP for $\gamma=0.3$ and $0.5$ and remains close for $\gamma=0.7$.}
\label{fig:GPprior_50D}
\end{figure}

\begin{figure}[t]
\centering
\includegraphics[width=\textwidth]{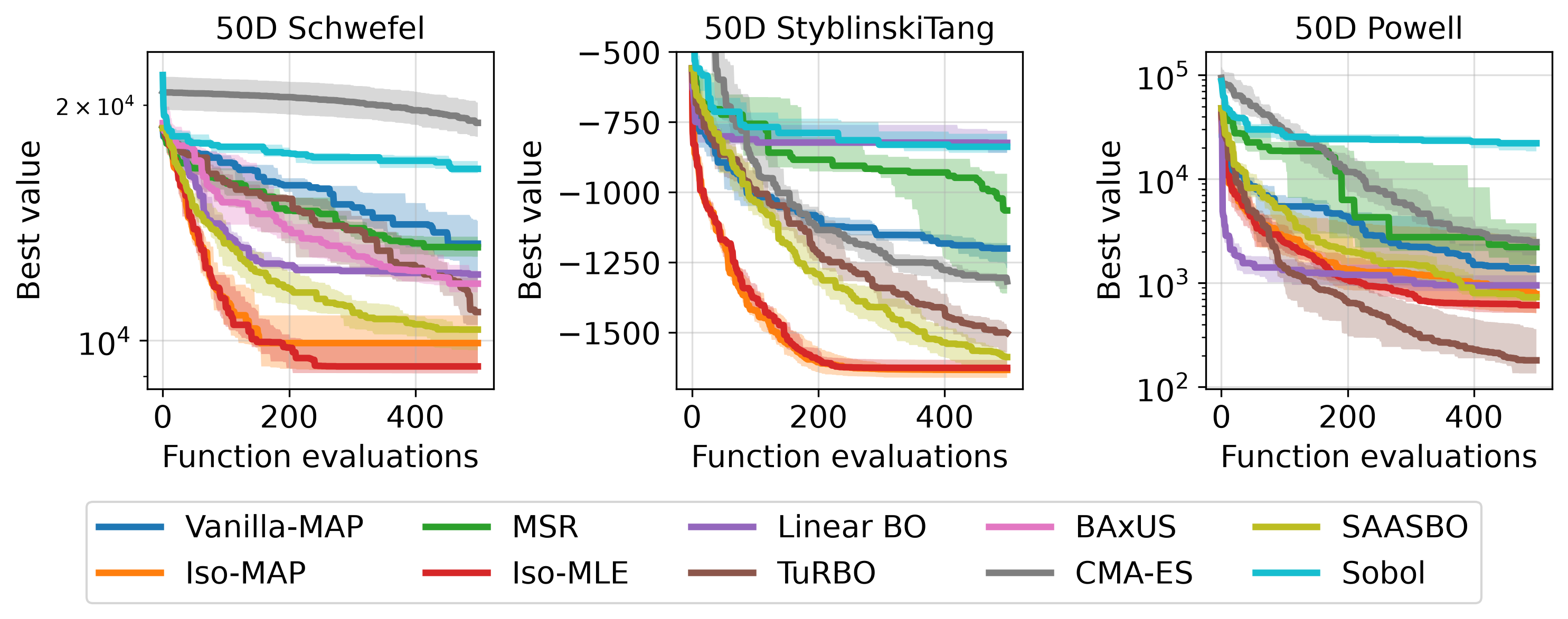}
\caption{Best-observed values on the 50D Schwefel, Styblinski--Tang, and Powell benchmarks with a budget of 500 evaluations. Iso-MAP finishes below matched Vanilla-MAP in all three reported settings.}
\label{fig:Synthetic_50D}
\end{figure}

\paragraph{Additional real-world benchmarks.}
Figure~\ref{fig:Original RealWorld} adds the 60D Rover and 102D Walker2d-v2 benchmarks described in Appendix~\ref{app:benchmark-definitions}. The 60D Rover result is broadly similar for Iso-MAP and Vanilla-MAP. The separation is larger on Walker2d-v2, where Iso-MAP reaches substantially lower median objective values than Vanilla-MAP. These results provide additional coverage across two benchmark parameterizations without changing the central experimental comparison.

\begin{figure}[t]
\centering
\includegraphics[width=\textwidth]{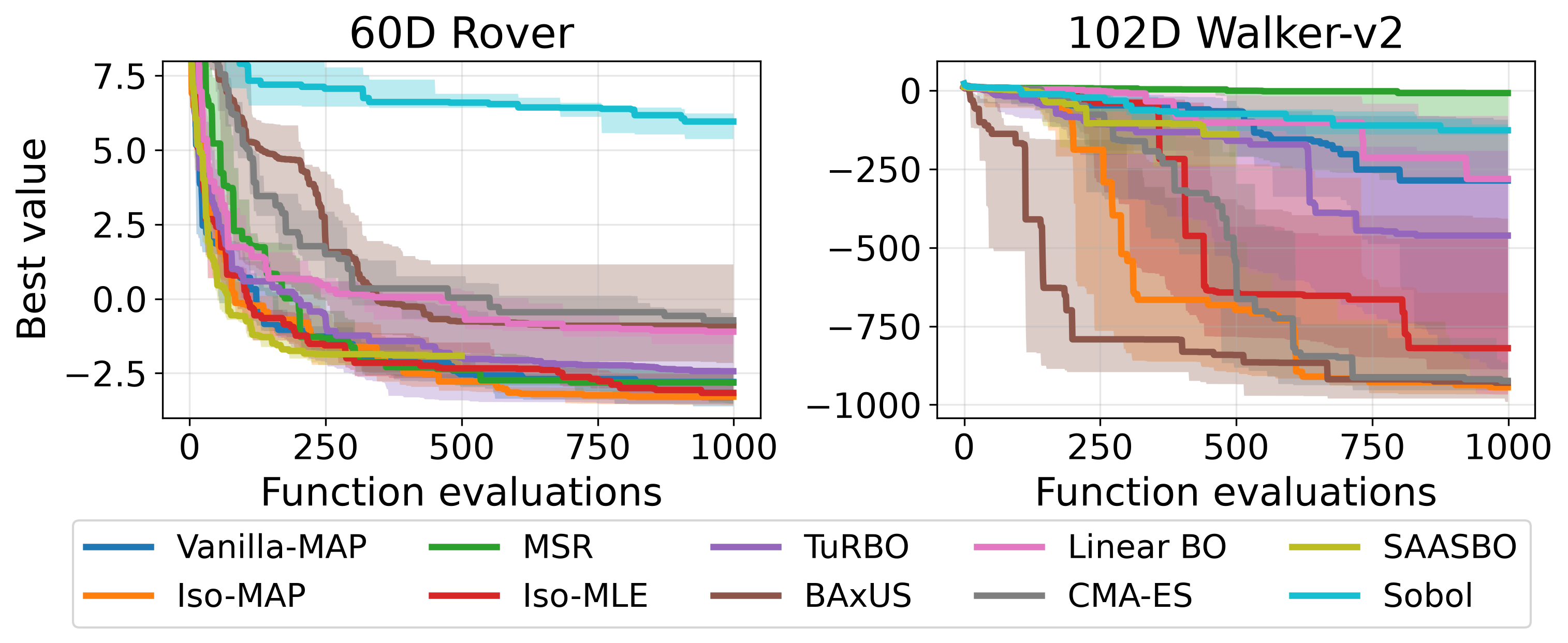}
\caption{Best-observed values on the 60D Rover and 102D Walker2d-v2 benchmarks. The budget is 1000 evaluations for the main methods and 500 for SAASBO. Iso-MAP remains close to Vanilla-MAP on 60D Rover and reaches substantially lower median values on Walker2d-v2.}
\label{fig:Original RealWorld}
\end{figure}

\paragraph{MOPTA08.}
\label{app:mopta08-extended}
Figure~\ref{fig:mopta} adds the 124D MOPTA08 engineering benchmark with a larger budget of 3000 evaluations. This experiment uses a targeted comparison among Iso-MAP, Linear BO, and CMA-ES. Iso-MAP reaches lower best-observed values than the two included comparators. Because Vanilla-MAP is not included in this experiment, we use MOPTA08 only as additional benchmark coverage; it does not test whether an isotropic model remains preferable to ARD at larger budgets.

\begin{figure}[t]
\centering
\includegraphics[width=0.45\textwidth]{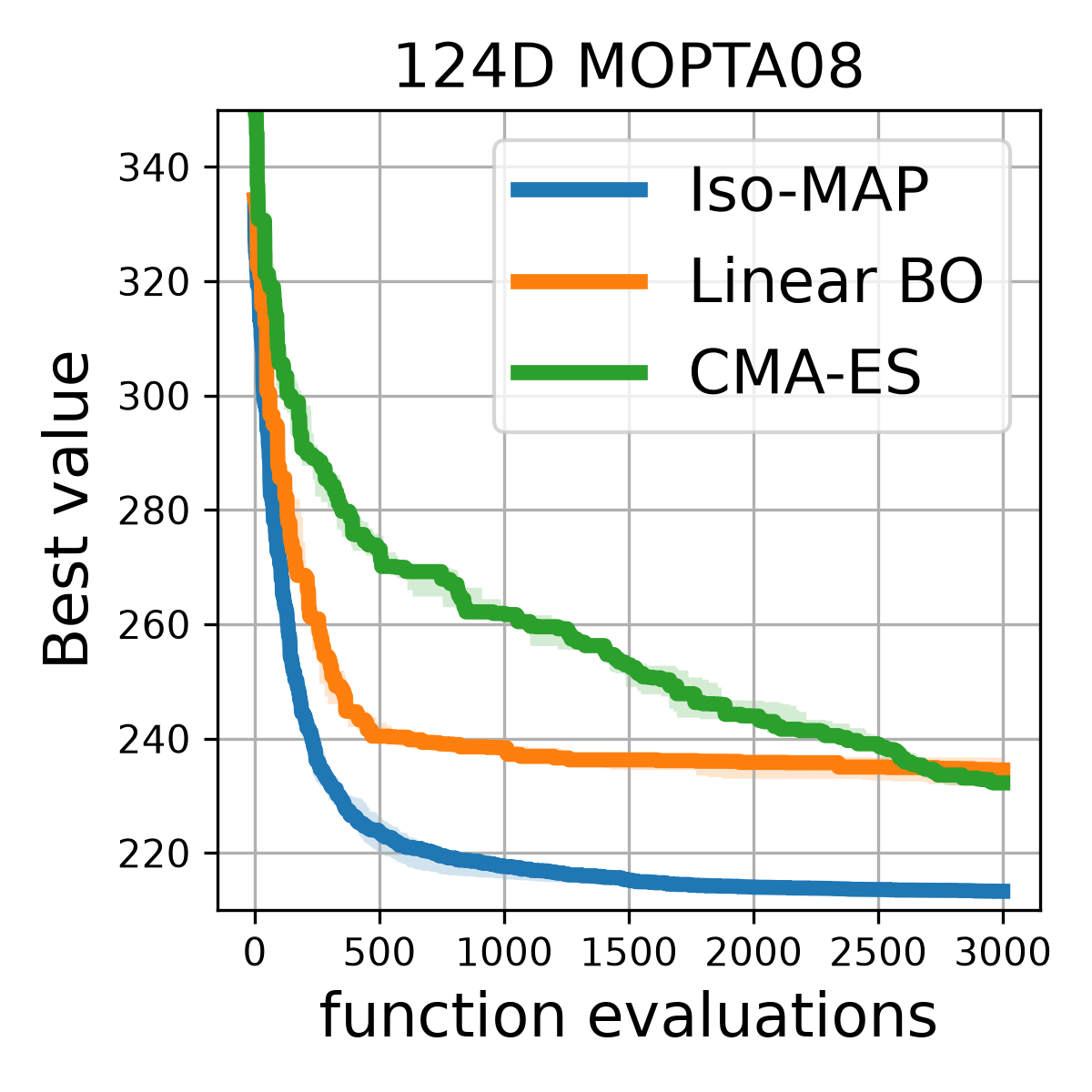}
\caption{Best-observed values on the 124D MOPTA08 vehicle-design benchmark with 68 constraints and $N=3000$ evaluations. The constrained benchmark is converted to an unconstrained objective using the soft penalty specified in Appendix~\ref{app:benchmark-definitions}. The comparison is restricted to Iso-MAP, Linear BO, and CMA-ES; lower values are better.}
\label{fig:mopta}
\end{figure}

\subsection{Sensitivity to the kernel family}
\label{app:Matern}

The main Iso-MAP versus Vanilla-MAP experiments use the SE kernel. To check whether the observed behavior is tied to this kernel choice, we rerun Vanilla-MAP, Iso-MAP, and Iso-MLE with Mat\'ern~$5/2$ kernels and compare them with their SE counterparts. Figures~\ref{fig:synthetic_matern}--\ref{fig:realworld_matern} summarize the results.
On the synthetic and GP-prior objectives, the SE and Mat\'ern~$5/2$ traces for Iso-MAP are generally close, whereas Vanilla-MAP shows larger changes between kernel families in several settings. The same comparison is more mixed on the real-world problems: neither kernel dominates uniformly, but the Iso-MAP traces again tend to be similar across the two choices. We therefore view the main isotropic-versus-ARD result as not being specific to the SE kernel, while avoiding a stronger claim that kernel choice is unimportant in general.

\begin{figure}[t]
\centering
\includegraphics[width=\textwidth]{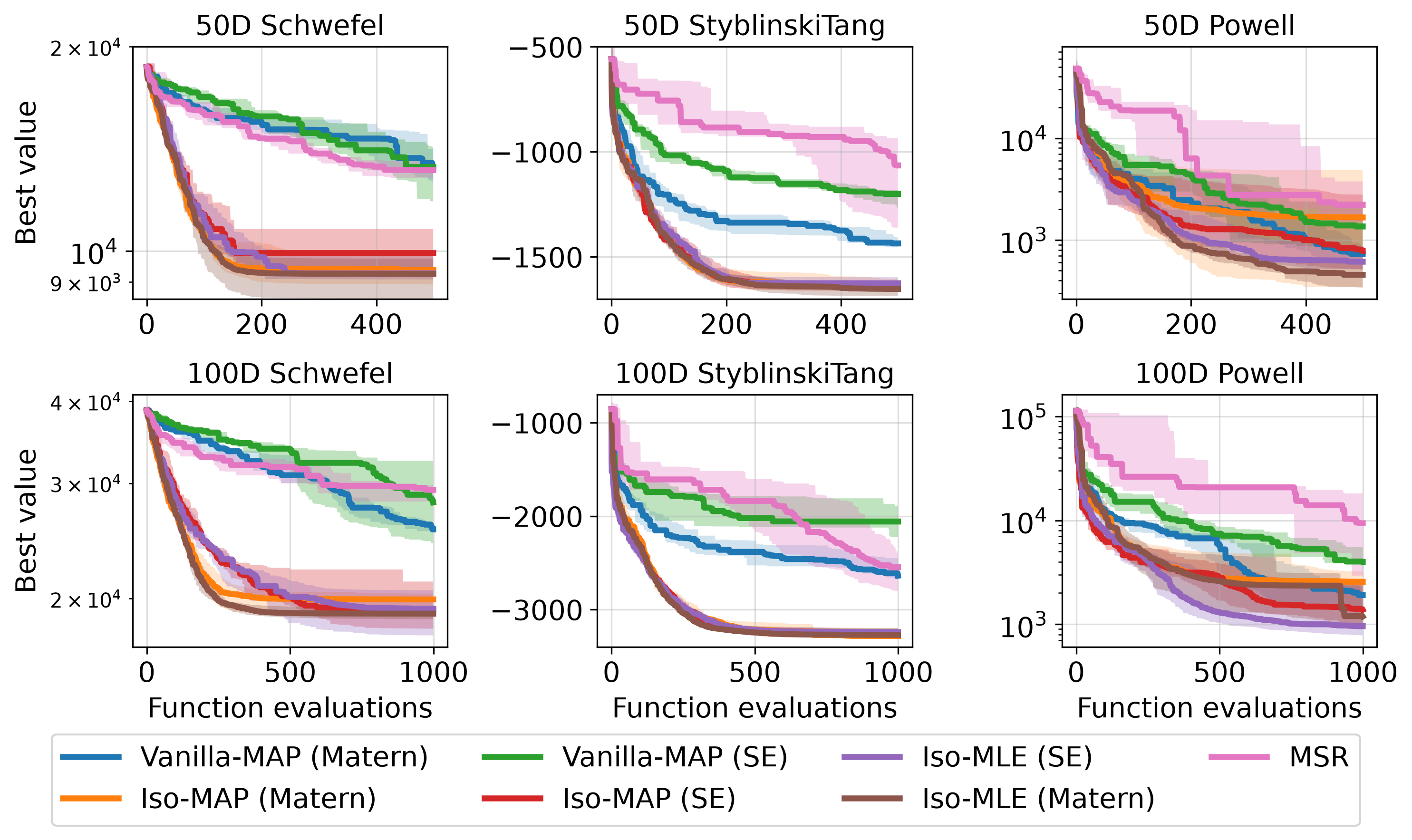}
\caption{SE versus Mat\'ern~$5/2$ kernels on the 50D and 100D Schwefel, Styblinski--Tang, and Powell benchmarks. Budgets are 500 evaluations in 50D and 1000 in 100D. The two kernel choices give similar Iso-MAP traces across these settings, while Vanilla-MAP is more sensitive to the kernel choice in several panels.}
\label{fig:synthetic_matern}
\end{figure}

\begin{figure}[t]
\centering
\includegraphics[width=\textwidth]{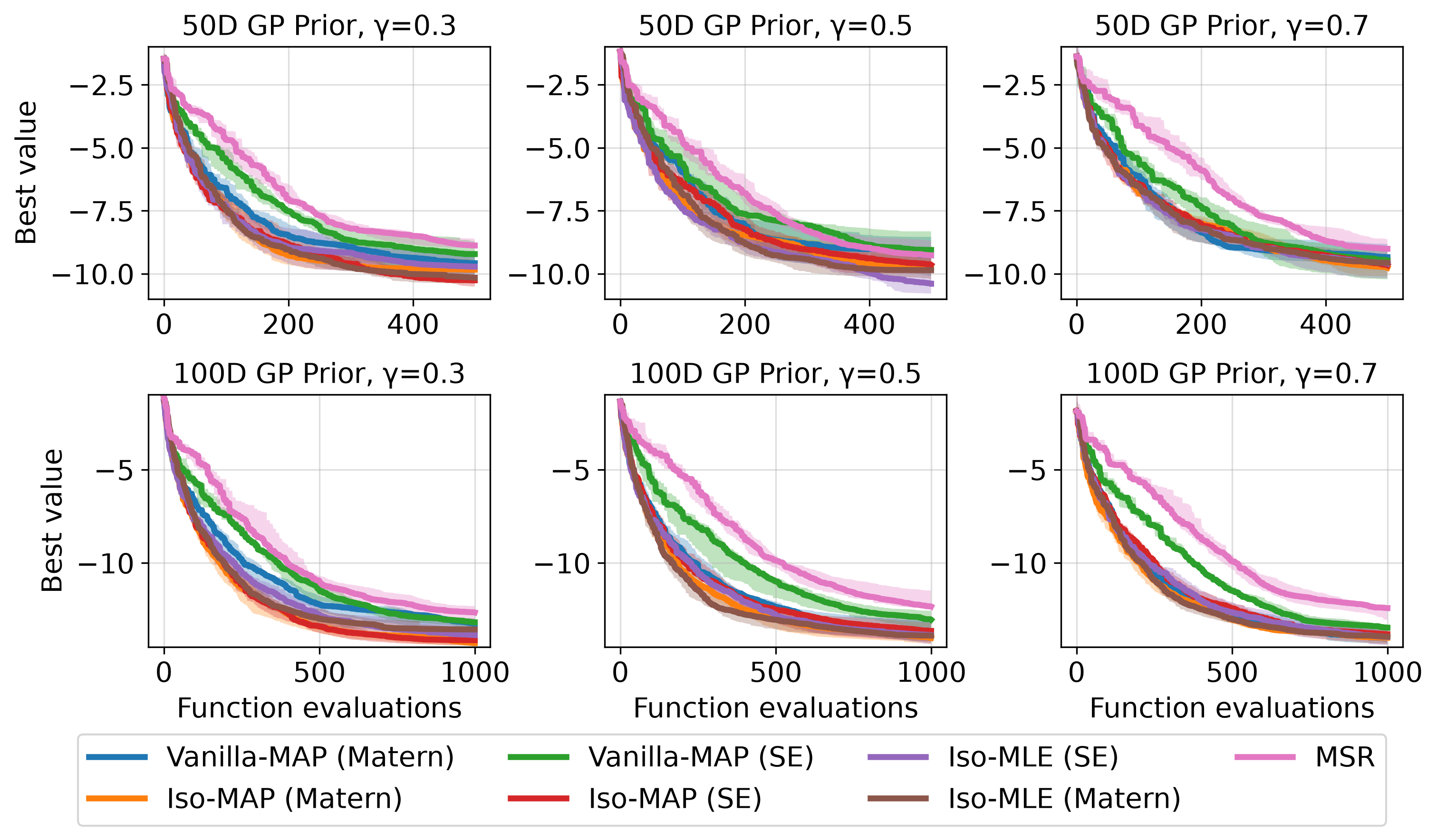}
\caption{SE versus Mat\'ern~$5/2$ kernels on anisotropic GP-prior objectives in 50D and 100D with $\gamma\in\{0.3,0.5,0.7\}$. Budgets are 500 evaluations in 50D and 1000 in 100D. Iso-MAP shows similar behavior under the two kernel families across the reported settings.}
\label{fig:GPprior_Matern}
\end{figure}

\begin{figure}[t]
\centering
\includegraphics[width=\textwidth]{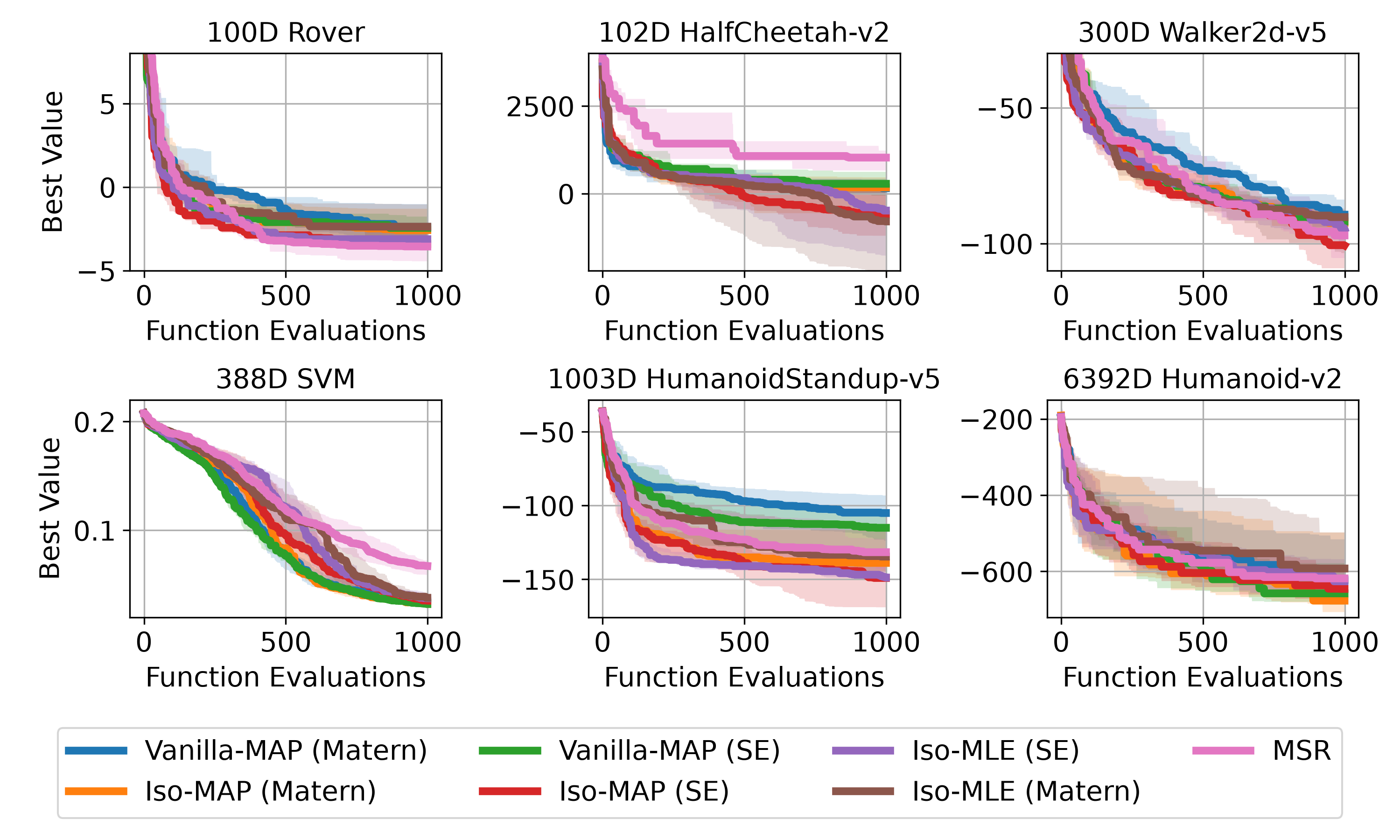}
\caption{SE versus Mat\'ern~$5/2$ kernels on the real-world benchmarks. The budget is 1000 evaluations for each task. Kernel choice affects performance differently across problems; the Iso-MAP traces are comparatively similar under the two kernel families in most of the reported settings.}
\label{fig:realworld_matern}
\end{figure}

\subsection{Query spread and exploration behavior}
\label{app:query-spread-exploration}

The BO results above compare objective values. We also examine how the surrogate parameterization changes the locations that BO samples. Following \citet{papenmeier2025exploring}, we compute observation traveling salesman distance (OTSD), normalized OTSD, and observation entropy (OE). OTSD measures the length of the shortest path through the observed query locations, normalized OTSD adjusts this quantity for dimension and sample size, and OE estimates the differential entropy of the query locations using the Kozachenko--Leonenko estimator~\citep{kozachenko1987sample}. Larger values indicate a more spatially dispersed set of observations.

Figures~\ref{fig:query-spread-50d} and~\ref{fig:query-spread-100d} report these quantities for the SE-kernel Vanilla-MAP and Iso-MAP traces on anisotropic GP-prior objectives with $\gamma=0.3$. Iso-MAP has lower OTSD, normalized OTSD, and OE in both dimensions, indicating that its queries are more spatially concentrated in these runs. These metrics are descriptive: they show a change in sampling geometry, not that more concentrated sampling is intrinsically better.

\begin{figure}[t]
\centering
\includegraphics[width=0.9\textwidth]{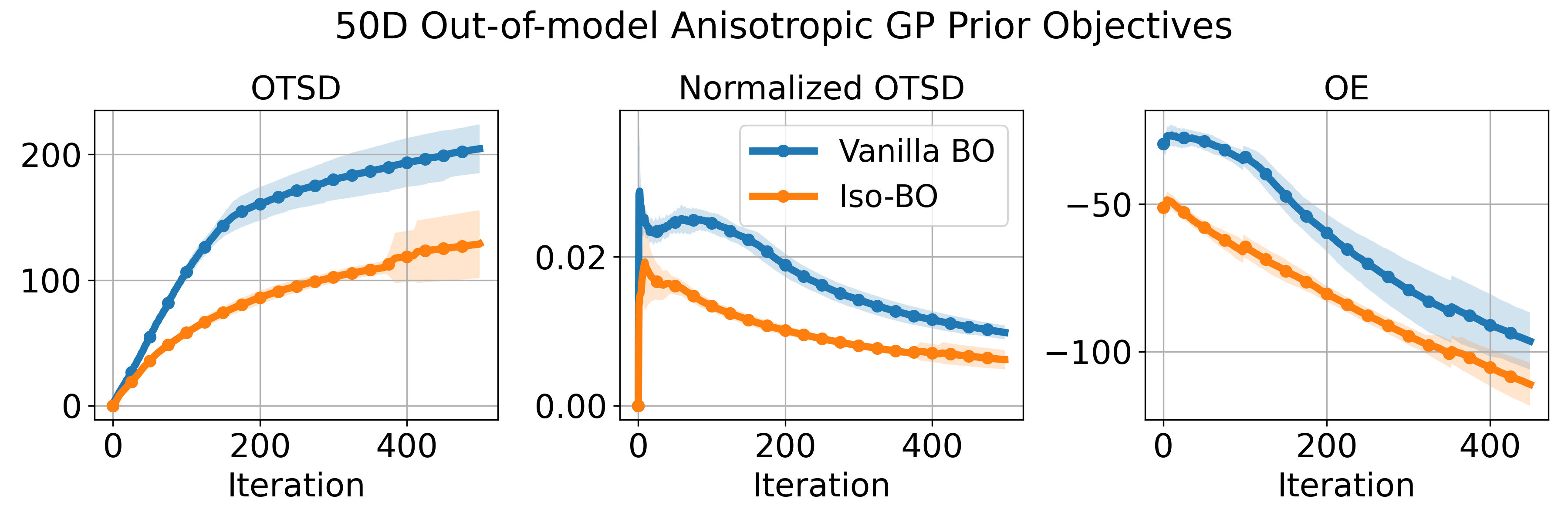}
\caption{Query-spread metrics for SE-kernel Vanilla-MAP and Iso-MAP on the 50D anisotropic GP-prior objectives with $\gamma=0.3$. \textbf{Left:} OTSD; \textbf{middle:} normalized OTSD; \textbf{right:} observation entropy (OE). Curves show the mean and $\pm1$ standard deviation over 10 runs. Larger values indicate more spatially dispersed query locations.}
\label{fig:query-spread-50d}
\end{figure}

\begin{figure}[t]
\centering
\includegraphics[width=0.9\textwidth]{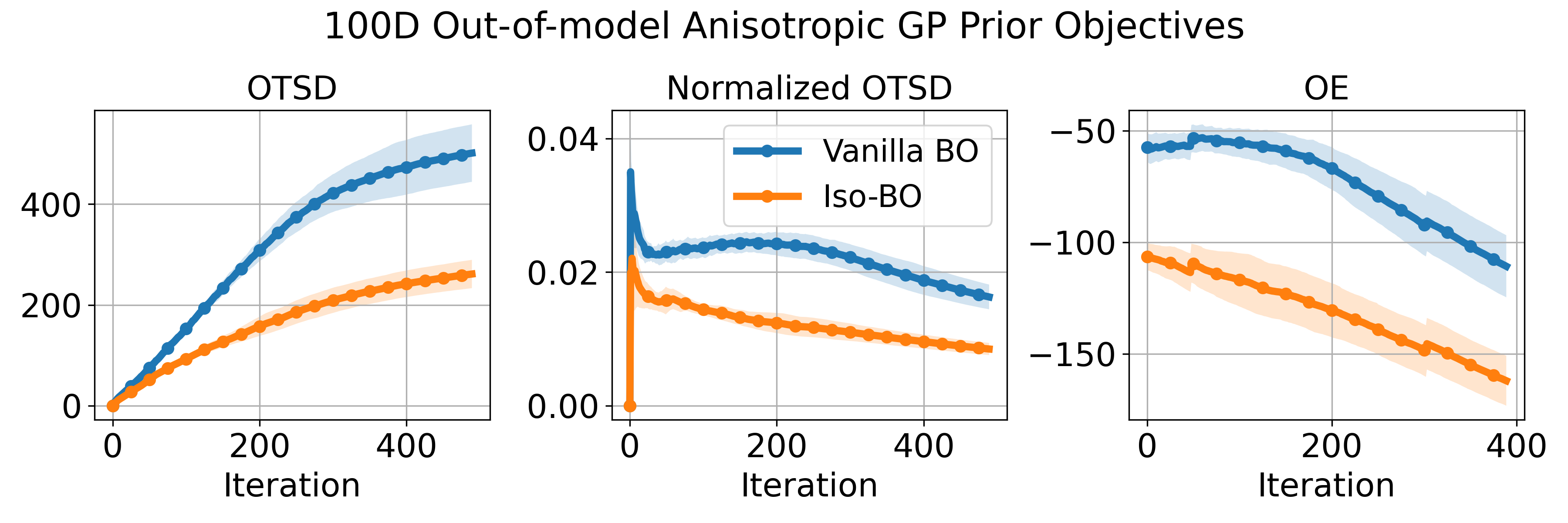}
\caption{Query-spread metrics for SE-kernel Vanilla-MAP and Iso-MAP on the 100D anisotropic GP-prior objectives with $\gamma=0.3$. \textbf{Left:} OTSD; \textbf{middle:} normalized OTSD; \textbf{right:} observation entropy (OE). Curves show the mean and $\pm1$ standard deviation over 10 runs. Larger values indicate more spatially dispersed query locations.}
\label{fig:query-spread-100d}
\end{figure}

The fitted lengthscales provide a simple geometric explanation for this difference. For an SE ARD kernel, candidate points with correlation at least $\rho_0\in(0,1)$ to an incumbent $x_{\mathrm{inc}}$ satisfy
\[
\sum_{k=1}^{D}
\frac{(x_k-x_{\mathrm{inc},k})^2}{\ell_k^2}
\leq
-2\log\rho_0.
\]
With one shared lengthscale, this region is the Euclidean ball
$\|x-x_{\mathrm{inc}}\|_2\leq\ell\sqrt{-2\log\rho_0}$.
With ARD, long fitted lengthscales can instead create elongated directions in which a candidate moves far in Euclidean distance while remaining close in kernel distance. This geometry is relevant to BO because, in the local-uninformative setting analyzed by \citet{hvarfner2024vanilla}, EI favors candidates that remain sufficiently correlated with the incumbent. The lower query-spread metrics are therefore consistent with Iso-MAP removing independently stretched ARD directions, although the metrics alone do not establish a causal explanation for the performance differences.

\subsection{Computational cost}
\label{app:computational-cost}
\label{app:fitting-time-diagnostic}

Iso-BO reduces the lengthscale optimization problem from $D$ ARD parameters to one shared parameter. This does not change the leading $\mathcal O(n^3)$ Cholesky cost of exact GP inference, but it can substantially reduce the cost of repeated hyperparameter optimization.
Figure~\ref{fig:Fitted_Lscale_CPU} isolates this effect. Objective functions are sampled from isotropic SE GP priors, and isotropic and ARD models are fit to the same Sobol datasets using Mat\'ern~$5/2$ kernels and MAP estimation with L-BFGS-B. The isotropic fits are consistently faster, with the difference increasing as the input dimension and training-set size grow. This experiment measures the GP-fitting step only; it is not an end-to-end BO timing comparison.

\begin{figure}[t]
\centering
\includegraphics[width=\textwidth]{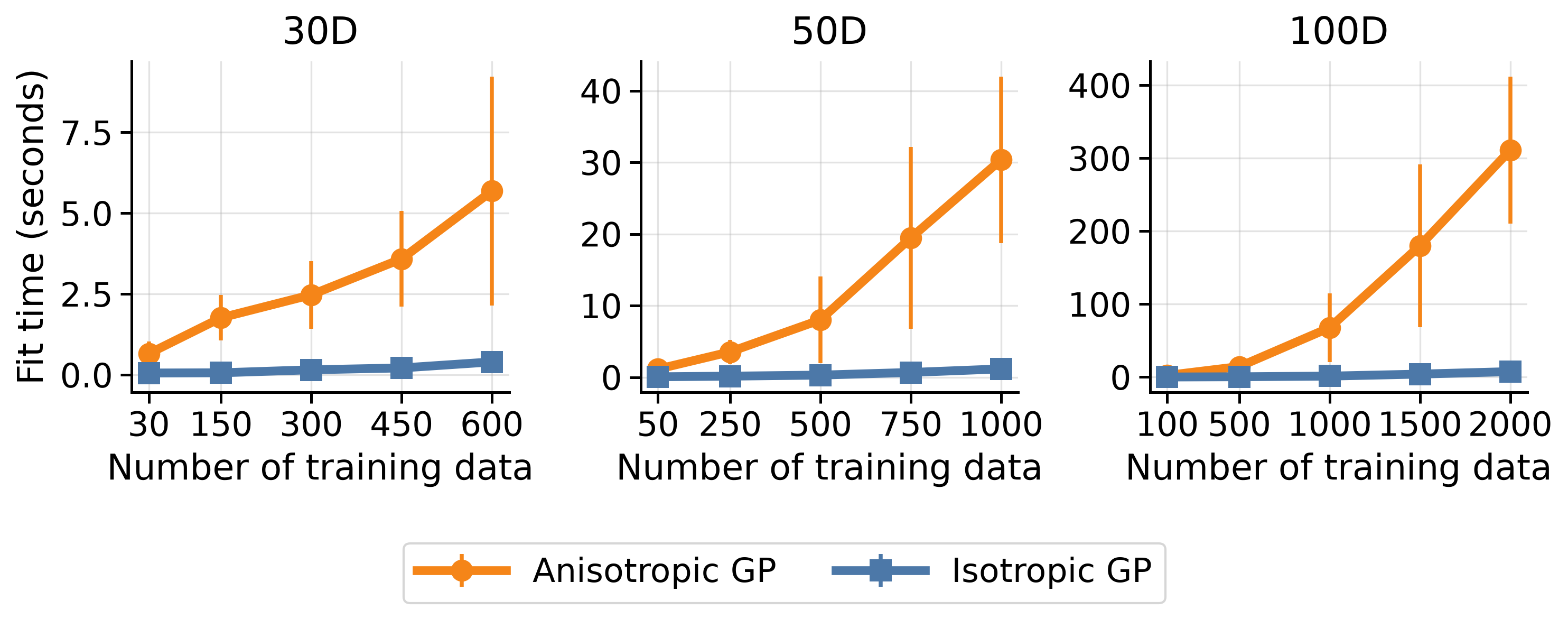}
\caption{Wall-clock MAP hyperparameter fitting time for isotropic and ARD GPs. Objective functions are sampled from isotropic SE GP-prior realizations, while both fitted models use Mat\'ern~$5/2$ kernels and L-BFGS-B. Training inputs are Sobol points. Curves show the mean and $95\%$ confidence interval over 20 independent replicates. Experiments use 4-core CPU jobs on nodes with Intel Xeon CPU Max 9470 processors and 512~GB DDR5 RAM. Both models retain the same leading $\mathcal O(n^3)$ exact-GP inference cost.}
\label{fig:Fitted_Lscale_CPU}
\end{figure}

Table~\ref{tab:GPU run time} gives a complementary end-to-end timing comparison on the 100D anisotropic GP-prior objective with $\gamma=0.3$. For this timing experiment, GP hyperparameters are refit at every BO iteration, and all methods are run on the same NVIDIA L40 GPU with 45~GB of memory. In the primary matched comparison, Iso-MAP requires 1783.0 seconds versus 3255.2 seconds for Vanilla-MAP. The remaining entries provide broader computational context; in particular, SAASBO requires substantially more runtime despite being run for only 500 evaluations. The Iso-MLE and MSR timings should not be interpreted as a clean parameterization-only comparison because those implementations use different kernel families, as described in Appendix~\ref{app:gp-model-fitting}.

\begin{table}[t]
\centering
\small
\setlength{\tabcolsep}{10pt}
\caption{End-to-end GPU runtime on the 100D anisotropic GP-prior objective with $\gamma=0.3$. GP hyperparameters are refit at every BO iteration for this timing experiment. All runs use an NVIDIA L40 GPU with 45~GB of memory. SAASBO uses 500 evaluations; all other methods use 1000.}
\begin{tabular}{l r r}
\hline
Method & Evaluations & GPU time (s) \\
\hline
Vanilla-MAP & 1000 & 3255.2 \\
Iso-MAP     & 1000 & 1783.0 \\
MSR         & 1000 & 7334.5 \\
Iso-MLE     & 1000 & 1431.1 \\
Linear BO   & 1000 & 1982.0 \\
TuRBO       & 1000 & 3556.6 \\
BAxUS       & 1000 & 1467.8 \\
SAASBO      & 500  & 76369.4 \\
\hline
\end{tabular}
\label{tab:GPU run time}
\end{table}

\section{Batch BO and latent space molecular optimization}
\label{app:batch BO}

The main experiments consider sequential BO, where the surrogate is updated after selecting one new point. We next ask whether the same lengthscale parameterization matters when several points are selected from each fitted surrogate (i.e., batch or parallel BO). We compare Vanilla-MAP and Iso-MAP with batch sizes $q=10$ and $q=50$. Both methods use the same GP models and fitting settings described in Appendix~\ref{app:gp-model-fitting}, with qLogEI~\citep{ament2023unexpected} as the batch acquisition function using 256 Monte Carlo samples. GP hyperparameters are refit at every batch iteration. Evaluation budgets below count individual objective evaluations, so increasing $q$ reduces the number of surrogate updates within a fixed budget.
All batch experiments use 10 independent runs. Solid curves show the median best-observed value across runs, and shaded regions show the IQR.

\subsection{Batch BO on GP-prior and real-world benchmarks}
\label{app:batch-main-results}

We first repeat the anisotropic GP-prior experiments with $\gamma\in\{0.3,0.5,0.7\}$ in 50D and 100D. The total budgets are 500 and 1000 evaluations, respectively. Figures~\ref{fig:GPprior_Batch_10} and~\ref{fig:GPprior_Batch_50} show a similar median pattern under batch selection. With $q=10$, Iso-MAP reaches lower median best-observed values for most of the run in all six settings, with the smallest final separation for the 50D $\gamma=0.7$ objective. With $q=50$, Iso-MAP reaches lower median values across all six reported settings.

These experiments show that the finite-budget behavior observed in the sequential GP-prior experiments also appears when each fitted surrogate selects multiple candidates at once.

\begin{figure}[t]
\centering
\includegraphics[width=0.9\textwidth]{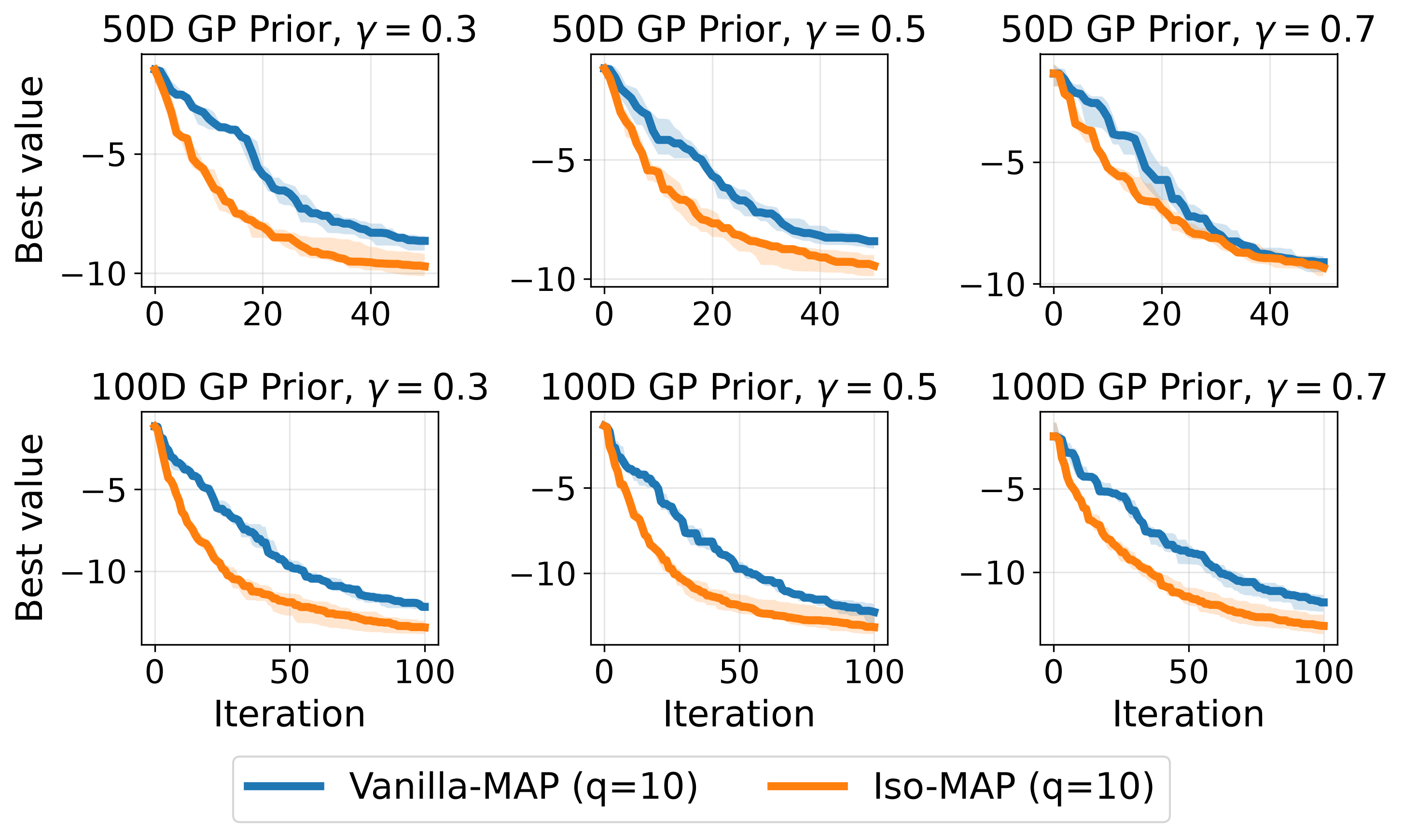}
\caption{Batch BO on anisotropic GP-prior objectives in 50D and 100D with $\gamma\in\{0.3,0.5,0.7\}$. Each BO iteration selects $q=10$ points with qLogEI. The total budgets are 500 evaluations in 50D and 1000 in 100D. Lower best-observed values are better. Iso-MAP reaches lower median values over most of the budget in all six settings, with the smallest final separation in the 50D $\gamma=0.7$ case.}
\label{fig:GPprior_Batch_10}
\end{figure}

\begin{figure}[t]
\centering
\includegraphics[width=0.9\textwidth]{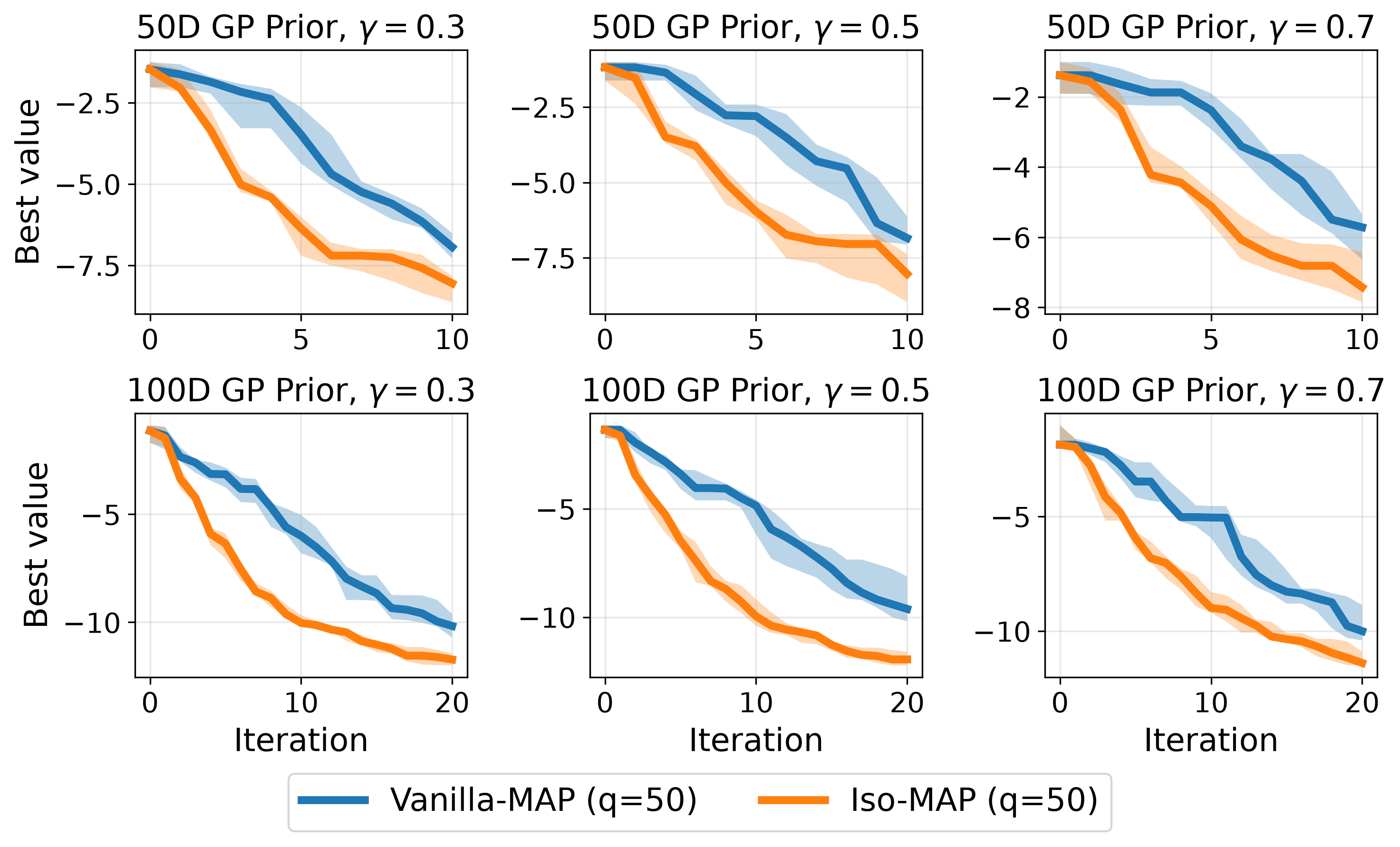}
\caption{Batch BO on the same anisotropic GP-prior objectives as Figure~\ref{fig:GPprior_Batch_10}, now selecting $q=50$ points per iteration. The total budgets are 500 evaluations in 50D and 1000 in 100D. Lower best-observed values are better. Iso-MAP reaches lower median values across all six reported settings.}
\label{fig:GPprior_Batch_50}
\end{figure}

We next consider the real-world benchmark suite shown in Figures~\ref{fig:realworld_Batch_10} and~\ref{fig:realworld_Batch_50}: 100D Rover, 102D Walker2d-v2, 102D HalfCheetah-v2, 300D Walker2d-v5, 388D SVM, and 1003D HumanoidStandup-v5. Each problem uses a total budget of 1000 evaluations. We omit the 6392D Humanoid-v2 benchmark because the Monte Carlo acquisition calculation can exceed the available memory in this implementation.

The results are more task-dependent than for the controlled GP-prior objectives. With $q=10$, Iso-MAP reaches lower final median values on Rover, HalfCheetah-v2, Walker2d-v5, and HumanoidStandup-v5, while Vanilla-MAP finishes lower on Walker2d-v2 and SVM. With $q=50$, Iso-MAP reaches lower median values on Walker2d-v2, HalfCheetah-v2, Walker2d-v5, and SVM, while the two methods finish at similar levels on Rover and HumanoidStandup-v5. These experiments do not show a uniform advantage on every task, but they show that the benefit of isotropic fitting is not limited to sequential candidate selection.

\begin{figure}[t]
\centering
\includegraphics[width=0.9\textwidth]{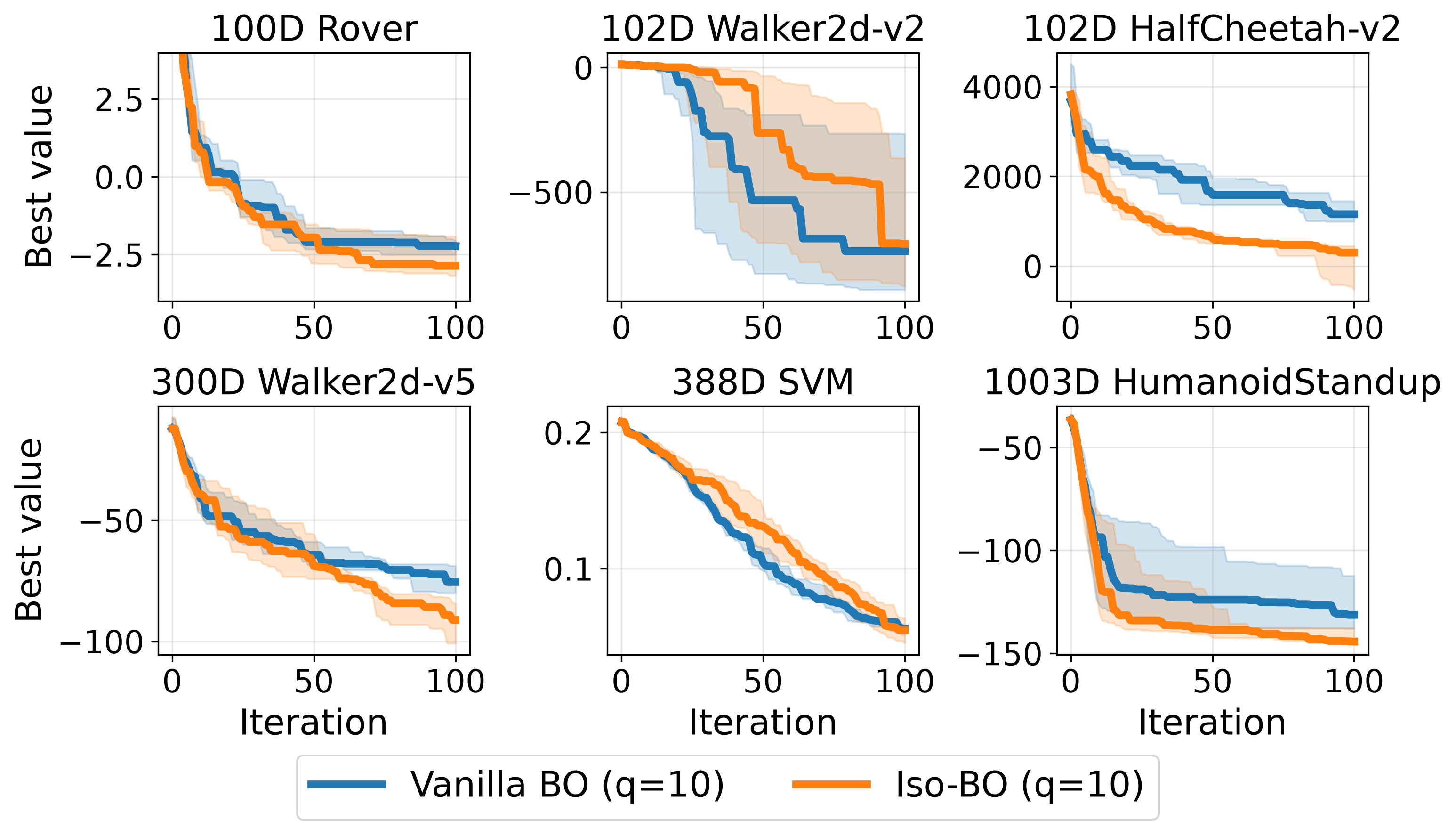}
\caption{Batch BO on real-world benchmarks with $q=10$ and a total budget of 1000 evaluations per problem. Lower best-observed values are better; reward-maximization tasks are reported as negative reward. Iso-MAP reaches lower final median values on Rover, HalfCheetah-v2, Walker2d-v5, and HumanoidStandup-v5, while Vanilla-MAP finishes lower on Walker2d-v2 and SVM.}
\label{fig:realworld_Batch_10}
\end{figure}

\begin{figure}[t]
\centering
\includegraphics[width=0.9\textwidth]{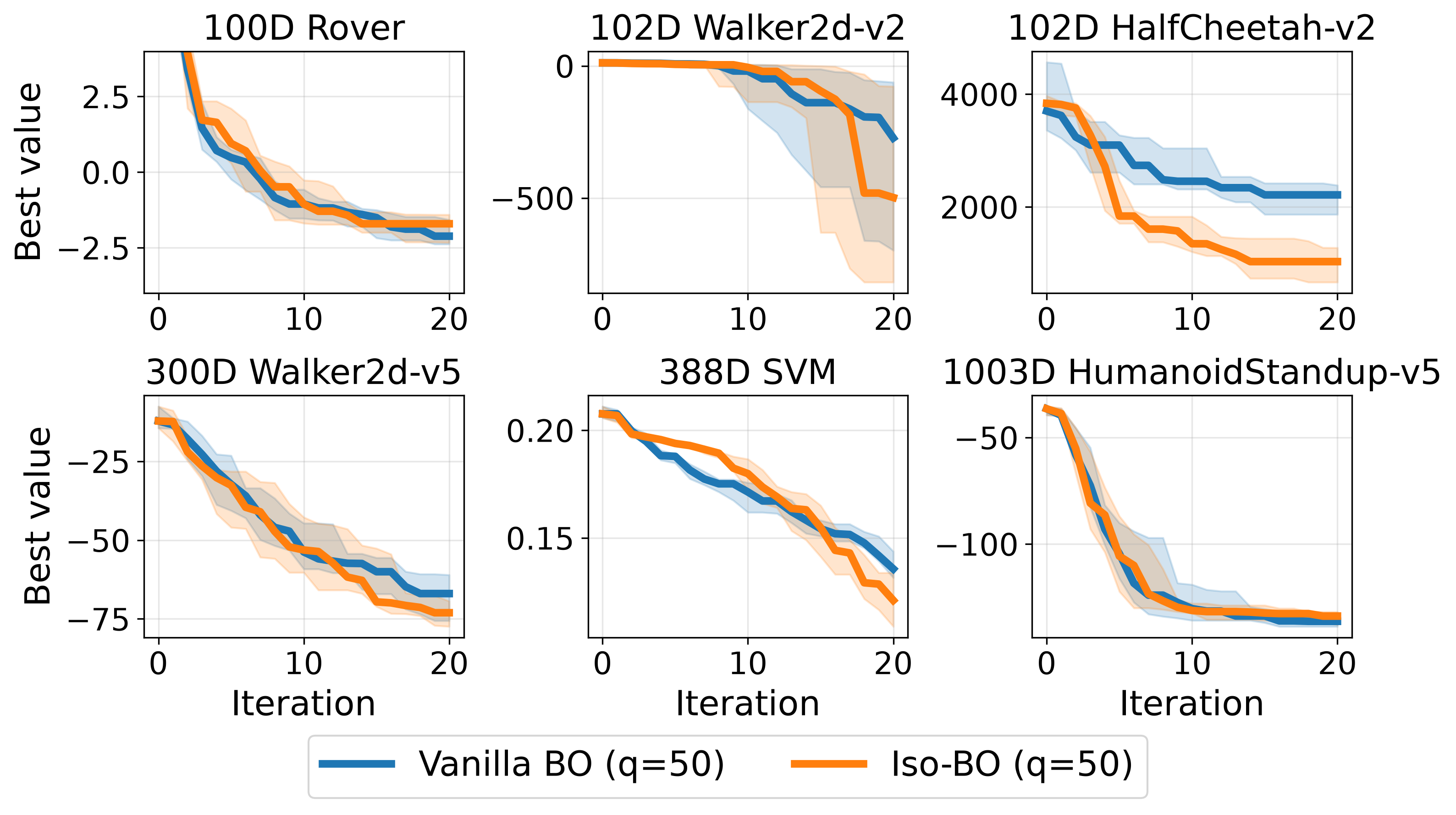}
\caption{Batch BO on the same real-world benchmarks as Figure~\ref{fig:realworld_Batch_10}, now with $q=50$. The total budget is 1000 evaluations per problem. Iso-MAP reaches lower final median values on Walker2d-v2, HalfCheetah-v2, Walker2d-v5, and SVM; the final median values are similar for the two methods on Rover and HumanoidStandup-v5.}
\label{fig:realworld_Batch_50}
\end{figure}

\subsection{Latent space molecular optimization}
\label{app:molecular-batch}

We also test the two GP parameterizations on structured molecular design problems. Following prior HDBO work~\citep{doumont2025we}, we optimize four GuacaMol goal-directed tasks~\citep{brown2019guacamol} in the 256-dimensional latent space of a pretrained SELFIES-VAE~\citep{maus2022local}. Fexofenadine MPO (\texttt{fexo}) and Ranolazine MPO (\texttt{rano}) are multi-parameter optimization tasks that combine similarity to a reference molecule with physicochemical objectives. Median Molecules 1 and 2 (\texttt{med1} and \texttt{med2}) instead reward molecules that balance structural similarity to two reference compounds. The component scores within each task are combined using a geometric mean. Following \citet{doumont2025we}, each latent coordinate is restricted to $[-8,8]$. We use a total budget of 3000 molecular evaluations and consider $q=10$ and $q=50$. Following the paper-wide minimization convention, the plotted objective is the negative task score, so lower values are better.

Figures~\ref{fig:molecular_Batch_10} and~\ref{fig:molecular_Batch_50} show that the effect of isotropic fitting varies across the four objectives. With $q=10$, Iso-MAP reaches a slightly lower final median value on \texttt{fexo} and a clearly lower value on \texttt{med2}; Vanilla-MAP finishes lower on \texttt{rano}, while the two methods converge to similar values on \texttt{med1}. With $q=50$, Iso-MAP reaches lower median values on \texttt{fexo}, \texttt{med1}, and \texttt{med2}, whereas the two methods are nearly indistinguishable on \texttt{rano}. The strongest and most consistent separation occurs on \texttt{med2}, where Iso-MAP remains below Vanilla-MAP for both batch sizes.

These molecular tasks provide a qualitatively different test from the continuous engineering and control benchmarks: BO operates on learned latent coordinates rather than directly interpretable physical variables. The results show that a shared lengthscale can remain useful in this setting, although the tradeoff between pooled estimation and coordinatewise flexibility is clearly task-dependent.

\begin{figure}[t]
\centering
\includegraphics[width=\textwidth]{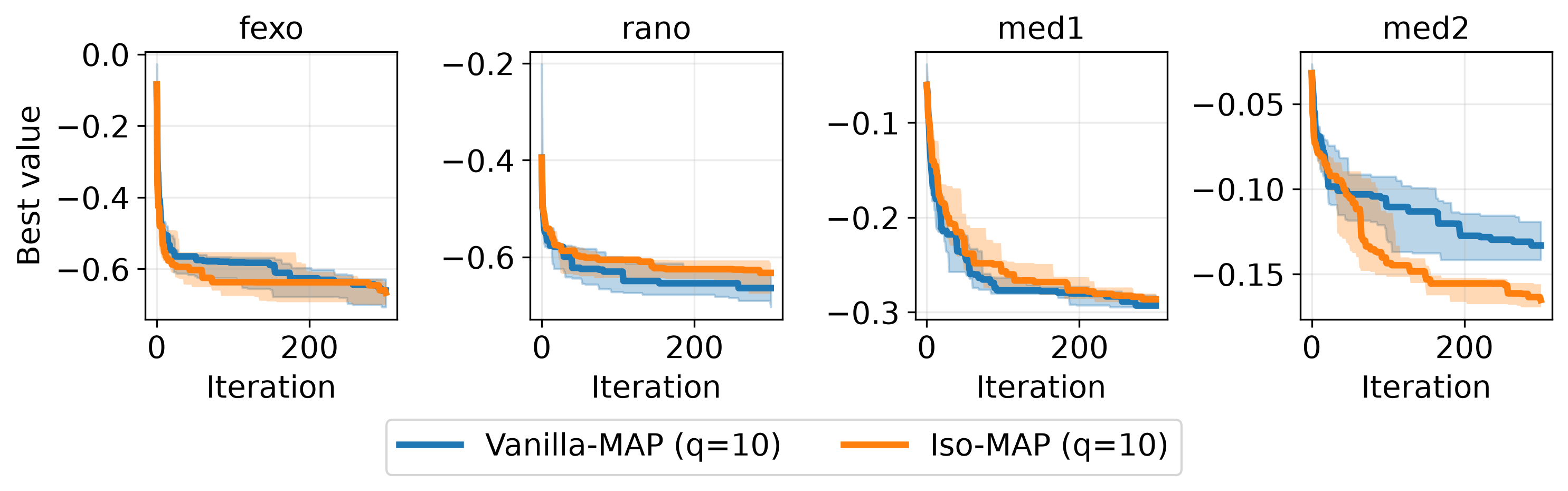}
\caption{Batch BO on four 256D GuacaMol latent-space molecular optimization tasks with $q=10$ and a total budget of 3000 evaluations. The plotted objective is the negative GuacaMol task score, so lower values are better. Iso-MAP reaches lower final median values on \texttt{fexo} and \texttt{med2}, Vanilla-MAP finishes lower on \texttt{rano}, and the two methods reach similar final values on \texttt{med1}.}
\label{fig:molecular_Batch_10}
\end{figure}

\begin{figure}[t]
\centering
\includegraphics[width=\textwidth]{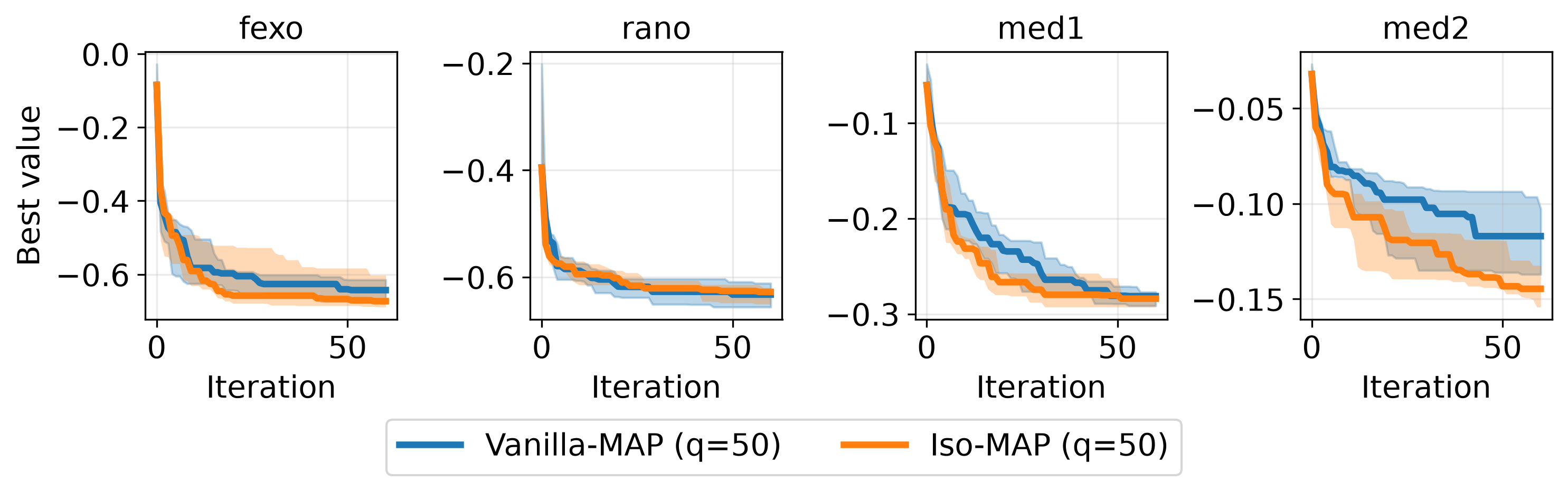}
\caption{Batch BO on the same 256D molecular optimization tasks as Figure~\ref{fig:molecular_Batch_10}, now with $q=50$. The total budget is 3000 evaluations. Iso-MAP reaches lower final median values on \texttt{fexo}, \texttt{med1}, and \texttt{med2}; the two methods finish at similar values on \texttt{rano}.}
\label{fig:molecular_Batch_50}
\end{figure}

Overall, the batch experiments give a consistent extension of the sequential results without changing their scope. Isotropic fitting remains effective when multiple candidates are selected from each surrogate fit and on high-dimensional molecular latent spaces, but the size of the benefit depends on the objective. This is consistent with the paper's finite-data interpretation: tying lengthscales reduces the amount of coordinatewise structure that must be learned from the available observations, while retaining ARD can be useful when that additional structure is informative.

\section{Information-adaptive surrogate flexibility with BO-FIM}
\label{app:BO-FIM}

The results in this paper suggest that the appropriate amount of surrogate flexibility may change as BO collects data. A shared lengthscale can be useful when coordinatewise ARD structure is poorly informed, while additional flexibility can become valuable once the data support it. We explore this idea with \emph{BO-FIM}, which uses the Fisher information matrix (FIM) of the ARD lengthscales to choose between isotropic and ARD GP models during the BO run.

\subsection{Fisher information for ARD lengthscales}
\label{app:fim-information}

Appendix~\ref{app:weak-information} shows how the pairwise-distance geometry of the observed design controls local Fisher information about ARD lengthscales. BO-FIM uses the same quantity as a practical diagnostic, evaluating the FIM at the fitted ARD model after each observation.

The implementation computes the FIM with respect to the raw lengthscales $\ell$, whereas the theory in Appendix~\ref{app:weak-information} uses $\eta=\ell^{-2}$. These parameterizations are related by a smooth one-to-one transformation for positive lengthscales, so exact rank deficiency is unchanged. The numerical eigenvalues do depend on the parameterization, which is one reason we treat the switching thresholds below as heuristic. The reported BO-FIM experiments use the SE kernel.

\subsection{BO-FIM switching rule}
\label{app:bo-fim-rule}

At each BO iteration, BO-FIM first fits the ARD GP and computes its FIM. If the chosen criterion indicates little local information about one or more ARD directions, an isotropic GP is fit to the same data and used for acquisition optimization. Otherwise, the fitted ARD GP is retained. This decision is recomputed after every new observation.

We consider two criteria. \emph{Eigen Product} uses the product of the FIM eigenvalues, equivalently its determinant, while \emph{Eigen Minimum} uses the smallest eigenvalue. The isotropic model is selected when the eigenvalue product falls below $0.1$ or when the minimum eigenvalue falls below $0.01$, respectively. The product is evaluated in log space for numerical stability. Each threshold is kept fixed across the four experiments below.

The pairwise-distance theory provides a useful reference point. With the 10-point initialization, there are only $45$ training-pair distance summaries, so the ARD FIM is rank deficient in both 50D and 100D. As observations accumulate, this rank restriction eventually disappears and the realized design determines how much information is available about different ARD directions. BO-FIM is intended to respond to this change rather than commit to one lengthscale parameterization for the full run.

The thresholds are heuristic rather than generally calibrated. FIM eigenvalues depend on the chosen parameterization and scale, and their product also depends strongly on dimension. Our goal here is not to propose a universal switching threshold, but to test whether local information about ARD structure can be used to adapt surrogate flexibility.

Both branches use the same response standardization, SE kernel, dimension-scaled lengthscale-prior construction, LogEI acquisition, and acquisition optimization settings as the corresponding MAP models. BO-FIM refits the ARD model at every iteration to update the FIM and also fits the isotropic model when it is selected. The fixed Iso-MAP and Vanilla-MAP runs use their standard refitting schedules of every 5 iterations in 50D and every 10 iterations in 100D. We therefore view BO-FIM as a proof of concept for information-adaptive modeling rather than a controlled comparison of refitting schedules.

\subsection{BO-FIM on strongly anisotropic objectives}
\label{app:bo-fim-results}

The strongly anisotropic GP-prior objectives from Appendix~\ref{app:stress-tests} provide a useful setting for testing this idea. At $\gamma=0.9$, fixed Iso-MAP remains effective, whereas at $\gamma=0.99$ the additional coordinatewise flexibility of ARD becomes more useful. We evaluate both BO-FIM criteria at $\gamma\in\{0.9,0.99\}$ in 50D and 100D, using budgets of 500 and 1000 evaluations, respectively. We compare against the fixed Vanilla-MAP and Iso-MAP models using 10 independent runs for each setting. Figure~\ref{fig:BO_FIM} reports the median best-observed value with the interquartile range.

Figure~\ref{fig:BO_FIM} shows the intended behavior. At $\gamma=0.9$, the adaptive variants remain competitive with the fixed models. At $\gamma=0.99$, where fixed isotropy becomes restrictive, both BO-FIM variants improve substantially over Iso-MAP. In 50D, both adaptive variants reach lower final median values than the two fixed models. In 100D, Eigen Product gives the lowest final median value, while Eigen Minimum recovers much of the gap between Iso-MAP and Vanilla-MAP.

\begin{figure}[t]
\centering
\includegraphics[width=\textwidth]{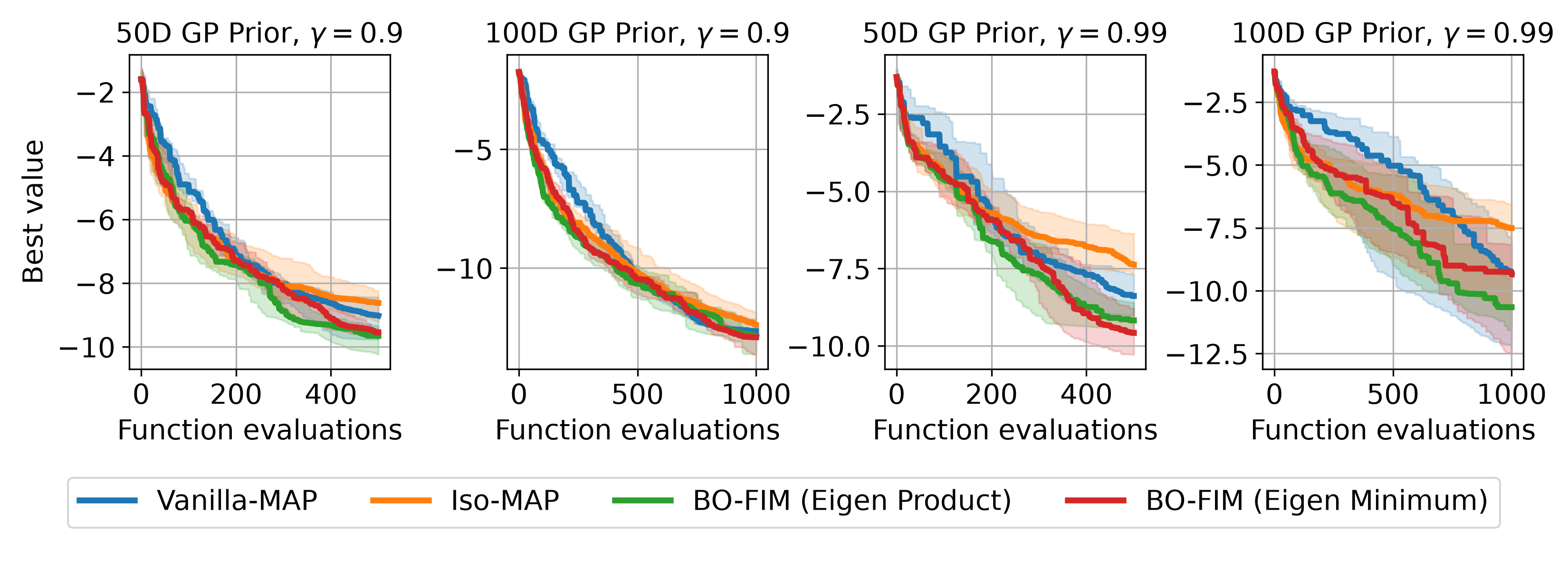}
\caption{BO-FIM on strongly anisotropic GP-prior objectives in 50D and 100D with $\gamma\in\{0.9,0.99\}$. ``Eigen Product'' uses the product of the FIM eigenvalues with threshold $0.1$, while ``Eigen Minimum'' uses the minimum eigenvalue with threshold $0.01$. Budgets are 500 evaluations in 50D and 1000 in 100D. Curves show the median over 10 independent runs and shaded regions show the interquartile range. The adaptive variants remain competitive at $\gamma=0.9$ and recover substantial performance relative to fixed Iso-MAP at $\gamma=0.99$.}
\label{fig:BO_FIM}
\end{figure}

These results provide a simple proof of concept for information-adaptive surrogate flexibility. Fixed isotropy can be useful when the data do not support coordinatewise structure, but that restriction can become costly when anisotropy becomes learnable. BO-FIM shows that local information from the current design can be used to move between these two regimes. Developing a better-scaled and more generally calibrated switching rule is an interesting and natural direction for future work.

\end{document}